\documentclass{article}

\usepackage{microtype}
\usepackage{graphicx}
\usepackage{subcaption}
\usepackage{booktabs} 

\usepackage{hyperref}

\usepackage[preprint]{icml2026}

\usepackage{hyperref}
\usepackage{url}
\usepackage{comment}
\usepackage{svg}
\usepackage{amsmath}
\usepackage{amssymb}
\usepackage{mathtools}
\usepackage{amsthm}
\usepackage{amsfonts}
\usepackage{hyperref}
\hypersetup{hidelinks}
\usepackage{url}
\usepackage{graphicx}

\usepackage{booktabs,multirow,threeparttable}
\usepackage{adjustbox}
\usepackage{bm}
\usepackage[capitalize,noabbrev]{cleveref}

\theoremstyle{plain}
\newtheorem{theorem}{Theorem}[section]

\newtheorem{lemma}[theorem]{Lemma}
\newtheorem{corollary}[theorem]{Corollary}
\theoremstyle{definition}
\newtheorem{definition}[theorem]{Definition}
\newtheorem{assumption}[theorem]{Assumption}
\theoremstyle{remark}
\newtheorem{remark}[theorem]{Remark}

\newcommand{\fh}[1]{{\color{purple}{[FH: #1]}}}

\usepackage[textsize=tiny]{todonotes}

\icmltitlerunning{PRISM: Parallel Reward Integration with Symmetry for MORL}
\begin{document}
\twocolumn[
  \icmltitle{PRISM: Parallel Reward Integration with Symmetry for MORL}




  \begin{icmlauthorlist}
    \icmlauthor{Finn van der Knaap}{ed}
    \icmlauthor{Kejiang Qian}{ed}
    \icmlauthor{Zheng Xu}{meta}
    \icmlauthor{Fengxiang He}{ed}
  \end{icmlauthorlist}

  \icmlaffiliation{ed}{University of Edinburgh}
  \icmlaffiliation{meta}{Meta Superintelligence Labs}

  \icmlcorrespondingauthor{Fengxiang He}{F.He@ed.ac.uk}

  \icmlkeywords{heterogeneous multi-objective reinforcement learning, reflection equivariance, reward shaping}

  \vskip 0.3in
]



\printAffiliationsAndNotice{}  



\begin{abstract}

This work studies heterogeneous Multi-Objective Reinforcement Learning (MORL), where objectives can differ sharply in temporal frequency. Such heterogeneity allows dense objectives to dominate learning, while sparse long-horizon rewards receive weak credit assignment, leading to poor sample efficiency. We propose a Parallel Reward Integration with Symmetry (PRISM) algorithm that enforces reflectional symmetry as an inductive bias in aligning reward channels. PRISM introduces ReSymNet, a theory-motivated model that reconciles temporal-frequency mismatches across objectives, using residual blocks to learn a scaled opportunity value that accelerates exploration while preserving the optimal policy. We also propose SymReg, a reflectional equivariance regulariser that enforces agent mirroring and constrains policy search to a reflection-equivariant subspace. This restriction provably reduces hypothesis complexity and improves generalisation. Across MuJoCo benchmarks, PRISM consistently outperforms both a sparse-reward baseline and an oracle trained with full dense rewards, improving Pareto coverage and distributional balance: it achieves hypervolume gains exceeding 100\% over the baseline and up to 32\% over the oracle. The code is at \href{https://github.com/EVIEHub/PRISM}{https://github.com/EVIEHub/PRISM}.

\end{abstract}
 
\section{Introduction}\label{sec:intro}
Reinforcement Learning (RL) has been approaching human-level capabilities in many decision-making tasks, such as 
playing Go \citep{silver2017mastering}, autonomous vehicles \citep{kiran2021deep}, robotics \citep{tang2025deep}, and finance \citep{hambly2023recent}. 
Multi-Objective Reinforcement Learning (MORL) extends this framework to handle multiple reward channels simultaneously, allowing agents to balance competing objectives efficiently \citep{liu2014multiobjective,hayes2022practical}. 
For example, a self-driving car must constantly balance multiple goals, such as minimising travel time while maximising passenger safety and energy efficiency. 
Prioritising speed would compromise the safety objectives, introducing the need for flexible and robust policies that can optimise across diverse and sometimes conflicting goals. 

\begin{figure}[t]
    \centering
    \includegraphics[width=0.8\linewidth]{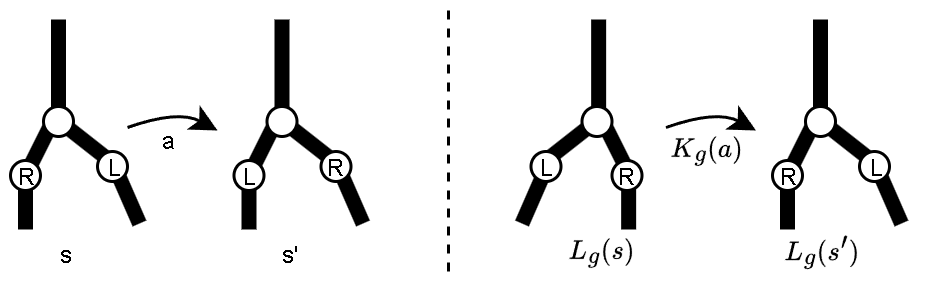}
    \caption{Reflectional symmetry in a two-legged agent. The left panel shows a transition from state $s$ to $s'$ under action $a$, whereas the right panel shows the reflected transition, where states and actions are transformed by $L_g$ and $K_g$, respectively.} 
    \label{fig:eq}
\end{figure}



 


This paper considers an important, yet premature, setting where reward channels exhibit considerable heterogeneity in facets such as sparsity. 
Dense objectives can overshadow their sparse and long-horizon counterparts, steering policies toward short-term gains, while neglecting the objectives that are harder to optimise but potentially more important.  
A straightforward approach is to employ reward shaping methods to align the reward channels. However, existing algorithms, such as intrinsic curiosity \citep{pathak2017curiosity, aubret2019survey} and attention-based exploration \citep{wei2025attention}, are developed for single-objective cases and have significant deficiencies: 
separately shaping individual objectives can distort the Pareto front and structures between objectives. This highlights a critical gap in the literature: MORL requires a reward shaping method that enables efficient integration of the parallel but heterogeneous reward signals, 
leveraging their intrinsic structure, in order to improve sample efficiency. 

To this end, we propose {Parallel Reward Integration with Symmetry for MORL} ({PRISM}), a method that structurally shapes the reward channels and leverages the reflectional symmetry in agents in heterogeneous MORL problems. 
We design a {Reward Symmetry Network} ({ReSymNet}) 
that predicts the reward given the state of the system and any available performance indicators (e.g., dense rewards in this work). 
The available sparse rewards are used as supervised targets. 
In ReSymNet, residual blocks are employed to approximate the `scaled opportunity value',  
which has been proven to help accelerate training, decrease the approximation error, while maintaining the optimal solution of the native reward signals \citep{laud2004theory}. 
After proper training, 
our ReSymNet can be a plug-and-play technique, compatible with any off-the-shelf MORL algorithm in an iterative refinement cycle, where the agent observes the shaped rewards to improve its policy and the reward model observes better trajectories from the updated policy to improve the approximated reward function. 
To exploit the structural information across reward signals, we design a Symmetry Regulariser (SymReg) to enforce reflectional equivariance of the objectives,
which provably reduces the hypothesis complexity. Intuitively, incorporating reflectional symmetry as an inductive bias allows an agent to generalise experience from one situation to its mirrored counterpart. 

The complementary components of PRISM synergise as follows. Heterogeneous reward structures cause asymmetric policy learning that violates the agent's physical symmetry: when dense objectives provide immediate gradients while sparse objectives only signal at the end of an episode, the policy may overfit to the denser objectives in specific states, failing to respect reflectional symmetry. ReSymNet eliminates temporal heterogeneity by aligning objectives to the same frequency, whereas SymReg enforces reflectional symmetry by preventing asymmetric learning dynamics.

We prove that PRISM constrains the policy search into a subspace of reflection-equivariant policies. This subspace is a projection of the original policy space, induced by the reflectional symmetry operator, provably of reduced hypothesis complexity, measured by covering number \citep{zhou2002covering} and Rademacher complexity \citep{bartlett2002rademacher}. 
This reduced complexity is further translated to improved generalisation guarantees. In practice, this means that by encouraging policies to respect natural symmetries, the agent searches over a smaller, more structured hypothesis space, reducing overfitting and improving sample efficiency. 

We conduct extensive experiments on the MuJoCo MORL environments \citep{todorov2012mujoco, felten_toolkit_2023}, using Concave-Augmented Pareto Q-learning (CAPQL) \citep{lu2023multi} as the backbone for PRISM. Sparse rewards are constructed by releasing cumulative rewards at the end of an episode. 
PRISM achieves hypervolume gains of over 100\% against the baseline operating on sparse signals, and even up to 32\% over the oracle (full dense rewards), also indicating a substantially improved Pareto front coverage. These gains are echoed in distributional metrics, confirming that PRISM learns a set of policies that are also better balanced and more robust. Comprehensive ablation studies further confirm that both 
ReSymNet and 
SymReg are critical. The code is at \href{https://github.com/EVIEHub/PRISM}{https://github.com/EVIEHub/PRISM}.

\section{Related Work}\label{sec:lit}


\textbf{Multi-Objective Reinforcement Learning.}
MORL algorithms typically fall into three categories: (1) single-policy methods that optimise user-specified scalarisations \citep{van2013scalarized,lu2023multi,hayes2022practical}; (2) multi-policy methods that approximate the Pareto front by solving multiple scalarisations or training policies in parallel \citep{roijers2015computing,van2014multi,reymond2019pareto,lautenbacher2025multi}; and (3) meta-policy and single universal policy methods that learn adaptable policies given some preferences \citep{DBLP:conf/iros/ChenGBJ19,yang2019generalized,basaklar2022pd,mu2025preference,liu2025pareto}. While these works have advanced Pareto-optimal learning, less attention has been given to heterogeneity in reward structures. 

\textbf{Reward Shaping.}
A large volume of literature tackles sparse rewards through reward shaping. Potential-based shaping \citep{ng1999policy} ensures policy invariance but requires hand-crafted potentials. However, this method's reliance on a manually designed potential function proved limiting. Intrinsic motivation methods reward novelty or exploration \citep{pathak2017curiosity,burda2018exploration}, while self-supervised methods predict extrinsic returns from trajectories \citep{memarian2021self,devidze2022exploration,holmes2025attention}. Recent advances utilise statistical decomposition to address sparsity \citep{DBLP:conf/nips/Gangwani0020, DBLP:conf/iclr/RenG0022}, or capture complex reward dependencies using transformers \citep{tang2024beyond, DBLP:journals/tmlr/TangC00LS25}. These approaches improve sample efficiency in single-objective RL, but do not extend naturally to MORL, where heterogeneous sparsity and scale can distort learning dynamics and Pareto-optimal trade-offs.

\textbf{Reflectional Equivariance.} 
%
To incorporate reflectional symmetry, 
a possible method is data augmentation, which adds mirrored transitions to the replay buffer but doesn't guarantee a symmetric policy and increases data processing costs \citep{lin2020invariant}. \citet{mondal2022eqr} propose latent space learning that encourages a symmetric representation through specialised loss functions. Another line of research focuses on equivariant neural networks \citep{DBLP:conf/nips/PolWHOW20, DBLP:journals/corr/abs-2007-03437, DBLP:conf/corl/WangWZP21}. For example, \citet{wang2022mathrm} design a stronger inductive bias via architecture-level symmetry, which hard-codes equivariance into the model for instantaneous generalisation.
However, \citet{park2024approximate} show that strictly equivariant architectures can be too rigid for tasks where symmetries are approximate rather than perfect. 
Building on this insight, our framework 
helps overcome the limitations of strictly equivariant architectures through tunable flexibility whilst being model-agnostic. 

\section{Preliminaries}\label{sec:prelims} 

\textbf{Multi-Objective Markov Decision Process.} 
Formally, we define an MORL problem via the Multi-Objective Markov Decision Process (MOMDP) model, as a tuple $\mathcal{M} = (\mathcal{S}, \mathcal{A}, \mathcal{P}, \bm{r}, \gamma)$: an agent at state $s$ from a finite or continuous state space $\mathcal{S}$, 
taking action $a$ from a finite or continuous action space $\mathcal{A}$, moves herself according to  
a transition probability function $\mathcal{P}: \mathcal{S} \times \mathcal{A} \times \mathcal{S'} \rightarrow [0,1]$, also denoted as $P(s' | s, a)$. 
The agent receives a reward via an $L$-dimensional vector-valued reward function $\boldsymbol{r}: \mathcal{S} \times \mathcal{A} \rightarrow \mathbb{R}^L$, 
where $L$ is the reward channel number, which decays by a discount factor $\gamma \in [0, 1)$. 
The goal in MORL is to find a policy $\pi: \mathcal{S} \rightarrow \mathcal{A}$ that optimises the expected cumulative vector return, defined as
$\bm{J}(\pi) = \mathbb{E}_\pi \left[ \sum_{t=0}^\infty \gamma^t \bm{r}_t \right]$. 
This paper addresses episodic tasks, where each interaction sequence has a finite horizon and concludes when the agent reaches a terminal state, at which point the environment is reset. Episodes $\tau_i$ are i.i.d.\ draws from the behaviour distribution $\mathcal D$, which describes the probability of observing different possible trajectories under the policy being followed.

\textbf{Reward Sparsity.}
Reward sparsity can be modelled as releasing the cumulative reward accumulated since the last non-zero reward with probability $p_{\text{rel}}$ at each timestep. When $p_{\text{rel}} = 0$, this reduces to the most extreme case: the agent receives rewards from dense channels $\mathcal{DC} = \{d_1, d_2, \ldots, d_D\}$ with observable rewards $r_t^{d_i}$ at every timestep, but the sparse channel is revealed only once at the end of the episode as $R_T^{sp} = \sum_{t=1}^T r_t^{sp}$.
The central challenge is to recover instantaneous sparse rewards $r_t^{sp}$ for each $(s_t,a_t)$ using only the cumulative observation $R_T^{sp}$ and correlations with dense channels. Formally, given a trajectory $\tau = \{(s_1,a_1),\ldots,(s_T,a_T)\}$ with cumulative sparse reward $R^{sp}(\tau)$, the task is to infer 
$\bm{r}^{sp} = [r_1^{sp}, \ldots, r_T^{sp}]^{\top}$, where $r_t^{sp}$ is the sparse reward at timestep $t$, such that $\sum_{t=1}^T r_t^{sp} \approx R^{sp}(\tau)$. 
For $p_{\text{rel}}>0$, an episode decomposes into sub-trajectories where the same formulation applies.

\textbf{Generalisability and Hypothesis Complexity.} 
A generalisation gap, at the episodic level, characterises the generalisability from a good empirical performance to its expected performance on new data \citep{wang2019generalization}. It depends on the hypothesis set's complexity, which is measured in this work by covering number \citep{zhou2002covering} and Rademacher complexity \citep{bartlett2002rademacher}. 

\begin{definition}[$l_{\infty,1}$ distance]
Let $\mathcal{X}$ be a feature space and $\mathcal{F}$ a space of functions from $\mathcal{X}$ to $\mathbb{R}^n$. The $l_{\infty,1}$-distance on the space $\mathcal{F}$ is defined as $l_{\infty,1}(f, g) = \max_{x \in \mathcal{X}} \left( \sum_{i=1}^{n} |f_i(x) - g_i(x)| \right)$.
\end{definition}

\begin{definition}[covering number]
The covering number, denoted $\mathcal{N}_{\infty,1}(\mathcal{F}, r)$, is the minimum number of balls of radius $r$ required to completely cover the function space $\mathcal{F}$ under the $l_{\infty,1}$-distance.
\end{definition}

\begin{definition}[Rademacher complexity]
Let $\mathcal{F}$ be a class of real-valued functions on a feature space $\mathcal{X}$, and let $\tau_1, \dots, \tau_N$ be i.i.d.\ samples from a distribution over $\mathcal{X}$. The empirical Rademacher complexity of $\mathcal{F}$ is $\hat{\mathfrak{R}}_N(\mathcal{F}) = \mathbb{E}_\sigma [ \sup_{f \in \mathcal{F}} \frac{1}{N} \sum_{i=1}^N \sigma_i f(\tau_i)]
$, where $\sigma_1, \dots, \sigma_N$ are independent Rademacher random variables taking values $\pm 1$ with equal probability. The Rademacher complexity of $\mathcal{F}$ is the expectation over the sample set. 
\end{definition}

\section{Parallel Reward Integration with Symmetry}\label{sec:framework}

This section introduces our algorithm PRISM.


\subsection{ReSymNet: Reward Symmetry Network}\label{sec:reward_model}
To address the challenge of heterogeneous reward objectives, PRISM first transforms sparse rewards into dense, per-step signals. We frame this as a supervised learning problem, inspired by but distinct from inverse reinforcement learning, as we do not assume access to expert demonstrations \citep{ng2000algorithms, arora2021survey}. The goal is to train a reward model, $\mathcal{R}_{\text{pred}}$, parametrised by $\psi$, that learns to map state-action pairs to individual extrinsic rewards. 


We hope to train the reward shaping model on a dataset collected by executing a purely random policy, ensuring broad state-space coverage. For each timestep $t$, we construct a feature vector $\bm{h}_t = [s_t, a_t, \bm{r}^{\text{dense}}_{t}]$, where $s_t$ is the state, 
$a_t$ is the action, 
and $\bm{r}_{\text{dense},t}$ are the dense rewards obtained from taking action $a_t$ at state $s_t$, which crucially leverages the information from already-dense objectives to help predict the sparse ones. Figure \ref{fig:model-reward} visualises the ResNet-like architecture. 

 \begin{remark}
Residual connections in $\mathcal{R}_{\text{pred}}$ are inspired by the theory of scaled opportunity value \citep{laud2004theory}, 
whose additive corrections preserve optimal policies, shorten the effective reward horizon, and improve local value approximation (see Appendix \ref{app:resymnet}).
\end{remark}
 
\begin{figure*}[t!]
    \centering
    \includegraphics[width=0.9\linewidth]{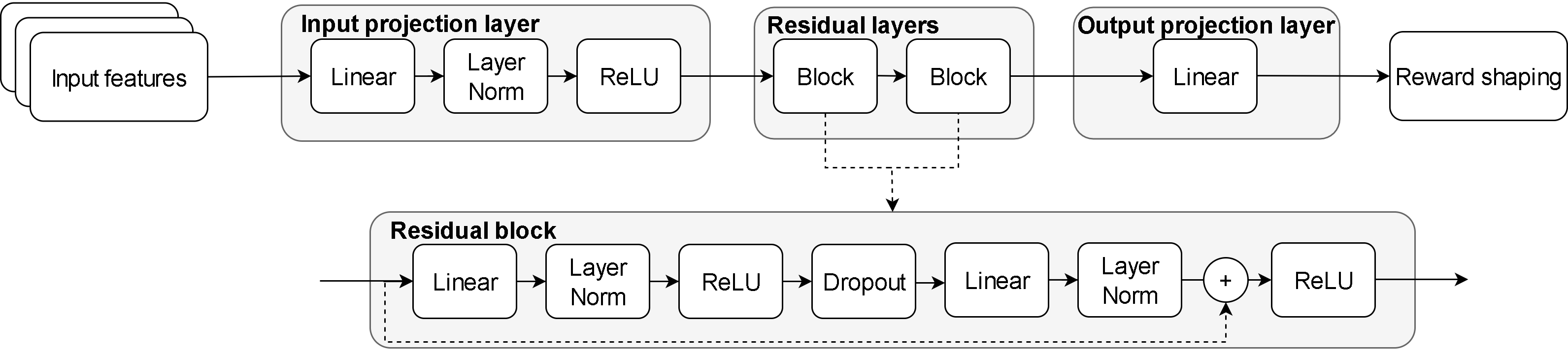}
    \caption{Overview of ReSymNet.} 
    
    \label{fig:model-reward}
\end{figure*}

The network is optimised by minimising the mean squared error between the sum of its per-step predictions over a trajectory and the true cumulative sparse reward observed for that trajectory:
\begin{equation}
    \mathcal{L}(\psi) = \sum_{\tau \in \mathcal{D}} \left(\sum_{t \in \tau} \mathcal{R}_{\text{pred}}(\bm{h}_t;\psi) - R^{sp}(\tau) \right)^2.
    \label{eq:nn_loss}
\end{equation}
To ensure the learned reward function is robust and adapts to the agent's improving policy, we incorporate two techniques: (1) we train an ensemble of reward models to reduce variance and produce a more stable shaping signal, and (2) we employ iterative refinement: the reward model is periodically updated using new, on-policy data collected by the agent. This allows the reward 
model to correct for the initial distribution shift and remain accurate as the agent's behaviour evolves from random exploration to expert execution, as outlined in Algorithm \ref{algo:reward_shaping} in Appendix \ref{app:resymnet}.





\subsection{SymReg: Enforcing Reflectional Equivariance}\label{sec:eq}



However, aligning reward frequencies alone is insufficient, as heterogeneous rewards cause the policy to learn asymmetrically across objectives, violating the agent's physical symmetry. To address this, we leverage reflectional symmetry as an inductive bias to prevent asymmetric policy learning. 
For example, for legged agents, flexing a leg is essentially the mirror image of extending it. Standard policies must learn both motions separately, wasting data. By encoding symmetry as an inductive bias, experience from one motion can be reused for its mirror, improving sample efficiency and robustness.

We formalise this physical intuition using group theory, specifically the reflection group $G = \mathbb{Z}_2$. This group consists of two transformations: the identity and a negation/reflection operator, $g$. Let $\mathcal{S} \subseteq \mathbb{R}^{d_s}$ and $\mathcal{A} \subseteq \mathbb{R}^{d_a}$ denote the state and action spaces, respectively, where $d_s$ is the dimension of the state space and $d_a$ of the action space. We define index sets $I_{\text{asym}}^s \subset \{1,\ldots,d_s\}$ and $I_{\text{sym}}^s \subset \{1,\ldots,d_s\}$ such that $I_{\text{asym}}^s \cap I_{\text{sym}}^s = \emptyset$ and $I_{\text{asym}}^s \cup I_{\text{sym}}^s = \{1,\ldots,d_s\}$. This partitions the state vector as $s = (s_{\text{asym}}, s_{\text{sym}})$ where $s_{\text{asym}} = s_{I_{\text{asym}}^s}$ and $s_{\text{sym}} = s_{I_{\text{sym}}^s}$.
We first partition the state vector $s$ into an asymmetric part, $s_{\text{asym}}$ (e.g., the torso's position), and a symmetric part, $s_{\text{sym}}$ (e.g., the leg's relative joint angles and velocities in Figure \ref{fig:eq}). The state transformation operator, $L_g: \mathcal{S} \to \mathcal{S}$, reflects the symmetric part of the state as follows: $L_g(s) = (s_{\text{asym}}, -s_{\text{sym}})$. 
Similarly, we define index sets $I_{\text{asym}}^a$ and $I_{\text{sym}}^a$ for the action space, and the action space is split up into an asymmetric part, $a_{\text{asym}}$, and a symmetric part, $a_{\text{sym}}$. The action transformation operator, $K_g: \mathcal{A} \to \mathcal{A}$, reflects the symmetric part of the action (e.g., the leg torques): $K_g(a) = (a_{\text{asym}},-a_{\text{sym}})$. 

The goal is to learn a policy, $\pi$, that is equivariant in terms of the aforementioned transformation. A policy $\pi$ is reflectional-equivariant if it satisfies the following condition for all states $s \in \mathcal{S}$:
 $   \pi(L_g(s)) = K_g(\pi(s))$.
This property means that the action for a reflected state is the same as the reflection of the action for the original state. To enforce this, we introduce a Symmetry Regulariser (SymReg) that explicitly penalises deviations from the desired symmetry property. During training, for each observation $s$, we compute both the standard policy output $\pi(a|s;\phi)$, parameterised by $\phi$, and the output for the reflected state $\pi(a|L_g(s);\phi)$. The equivariance loss is then defined as:
\begin{equation*}
\mathcal{L}_{\text{eq}} = \mathbb{E}_{s \sim \mathcal{D}, a \sim \pi_{\phi}} \left[ \| \pi(a|L_g(s);\phi) - K_g(\pi(a|s;\phi)) \|^2_1 \right].
\end{equation*}
SymReg measures the deviation between the policy's actual response to a reflected state and the expected reflected response.
The training objective combines the standard policy gradient loss, $J_{\pi}(\phi)$, with SymReg:
$\mathcal{L}_{\text{total}} = J_{\pi}(\phi) + \lambda \mathcal{L}_{\text{eq}}$,
where $\lambda$ is a hyperparameter controlling SymReg.

\section{Theoretical Analysis}\label{sec:theory}

This section presents theoretical guarantees of PRISM's generalisability. Let $\Pi$ be the full hypothesis space of policies represented by ReSymNet, $R(\pi;\tau)$ is the cumulative return for a single trajectory $\tau$ obtained following policy $\pi$. 

\begin{remark}
    As the backbone of the whole method, the hypothesis complexity and generalisability of ReSymNet contribute significantly to the generalisability of the whole algorithm. Due to space limit, we present Theorem \ref{thm:resymnet_cover} in the appendices for the covering number of ReSymNet's hypothesis space.
\end{remark}

The theory relies on these assumptions:

\begin{assumption}[bounded returns]\label{ass:bounded_returns}
For all policies $\pi$ and trajectories $\tau$,
$  0 \le R(\pi;\tau) \le B$.
\end{assumption}

\begin{assumption}[Lipschitz-continuous return]\label{ass:lipschitz_returns}
There exists $L_R>0$ such that for all $\pi,\tilde\pi\in\Pi$ and any trajectory $\tau$,
  $|R(\pi;\tau)-R(\tilde\pi;\tau)| \le L_R d(\pi,\tilde\pi)$,
where $d(\pi,\tilde\pi) := \sup_{s\in\mathcal S}\|\pi(s)-\tilde\pi(s)\|_1$.
\end{assumption}

\begin{assumption}[compact spaces]\label{ass:spaces}
The state space $\mathcal S$ and action space $\mathcal A$ are compact metric spaces.
\end{assumption}

\begin{assumption}[bounded policy]\label{ass:policy_class}
Policies $\pi\in\Pi$ have bounded inputs and weights. 
\end{assumption}

\begin{assumption}[episode sampling]\label{ass:episode_sampling}
The behaviour distribution $\mathcal D$ has state marginal lower-bounded by $p_{\min}>0$ 
on the state support of interest (finite-support or density lower-bound assumption).
\end{assumption} 

The Assumptions are reasonably mild. \cite{bartlett2017spectrally} prove that feedforward ReLU are Lipschitz functions; 
since our policies are implemented as ReLU networks, this ensures bounded sensitivity of the policy outputs to perturbations. Assuming further that the return function is Lipschitz in the policy outputs, it follows that returns are Lipschitz in the policies themselves, as stated in Assumption \ref{ass:lipschitz_returns}. Assumption \ref{ass:episode_sampling} ensures that all relevant states are sufficiently sampled under the behaviour policy, 
which is, in practice, reasonable because policy exploration mechanisms prevent the policy from collapsing onto a subset of states. 

\subsection{Generalisability of Reflection-Equivariant Subspace} 

Let $G=\mathbb{Z}_2$ act on states and actions via $L_g, K_g$. An orbit-averaging operator
$  \mathcal Q(\pi)(s) = \tfrac{1}{2}\big(\pi(s)+K_g(\pi(L_g(s)))\big)$
maps any policy to a reflection-equivariant subspace \citep{qin2022benefits}. 
The regulariser
$\mathcal{L}_{\mathrm{eq}} = \mathbb{E}_s\|\pi(L_g(s))-K_g(\pi(s))\|_1^2$
encourages convergence to the fixed-point subspace, defined as follows.
\begin{definition}[reflection-equivariant subspace]
    We define reflection-equivariant subspace as $\Pi_{\mathrm{eq}}:=\{\pi:\pi(L_g(s))=K_g(\pi(s))\}$.
\end{definition}
We prove that $\mathcal Q$ is reflectional equivariant, a projection, and that its image coincides with the set of equivariant policies in Lemmas \ref{lemma:app-Q-eq}, \ref{lemma:app-q-idem}, and \ref{lemma:app-q-image} in Appendix \ref{app:Q}, respectively. Thus, $\mathcal Q$ is surjective
onto $\Pi_{\text{eq}}$.
To prove that the subspace $\Pi_{\text{eq}}$ is less complex, we show that the projection $\mathcal{Q}$ is non-expansive, which implies its image has a covering number no larger than the original space.
  
\begin{theorem}\label{theorem:covering}
The space $\Pi_{\text{eq}}$ has a covering number less than or equal to that of $\Pi$. Let $\mathcal{N}_{\infty,1}(\mathcal{F}, r)$ be the covering number of a function space $\mathcal{F}$ under the $l_{\infty,1}$-distance. Then, $\mathcal{N}_{\infty,1}(\Pi_{\text{eq}}, r) \le \mathcal{N}_{\infty,1}(\Pi, r)$.
\end{theorem}
 
The $l_{\infty,1}$-distance between two policies $\pi_{\phi}$ and $\pi_{\theta}$ is $d(\pi_{\phi}, \pi_{\theta}) = \sup_{s} \|\pi_{\phi}(s) - \pi_{\theta}(s)\|_1$.
The distance between their projections, $d(\mathcal{Q}(\pi_{\phi}), \mathcal{Q}(\pi_{\theta}))$, is no larger using the fact that $K_g$ is a norm-preserving isometry, $\|K_g(a)\|_1 = \|a\|_1$, and that $L_g$ is a bijection, which implies that the supremum over $s$ equals the supremum over $L_g(s)$. Hence $\mathcal{Q}$ is non-expansive, and a non-expansive surjective map cannot increase the covering number. Following Lemma \ref{lemma:app-q-image}, $\mathcal{N}(\Pi_{\text{eq}}, r) \le \mathcal{N}(\Pi, r)$. A detailed proof can be found in Appendix \ref{app:covering}. 


The symmetrisation technique is fundamental in empirical process theory that reduces the problem of bounding uniform deviations to analysing Rademacher complexity \citep{bartlett2002rademacher}. 
\begin{corollary}\label{lemma:rademacher}
For any class $\mathcal{F}$ of functions bounded in $[0,B]$, the expected supremum of empirical deviations satisfies:
\begin{equation*}
    \mathbb{E}\left[\sup_{f\in\mathcal{F}} \left|\frac{1}{N}\sum_{i=1}^N (f(\tau_i)-\mathbb{E}[f])\right|\right] \leq 2\mathbb{E}[\mathfrak{R}_N(\mathcal{F})],
\end{equation*}
where $\mathfrak{R}_N(\mathcal{F}) = \mathbb{E}_\sigma\left[\sup_{f\in\mathcal{F}}\frac{1}{N}\sum_{i=1}^N\sigma_i f(\tau_i)\right]$ is the Rademacher complexity and $\sigma_i$ are independent Rademacher random variables taking values $\pm 1$. 
\end{corollary}
This bound transforms the original centred empirical process into a symmetrised version that is often easier to analyse.
We now prove a high-probability uniform generalisation bound over the reflection-equivariant subspace. A detailed proof can be found in Appendix \ref{app:exact}. We recognise that PRISM does not necessarily converge to it, which will be discussed in the following subsection.

\begin{theorem}\label{thm:exact_full_proof}
With $\mathcal R_{\Pi_{\mathrm{eq}}}=\{ \tau\mapsto R(\pi;\tau) : \pi\in\Pi_{\mathrm{eq}}\}$,
fix any accuracy parameter $r\in(0,B)$ and confidence $\delta\in(0,1)$. Then with probability at least $1-\delta$,
\begin{equation*}
\begin{split}
&\sup_{\pi\in\Pi_{\mathrm{eq}}} \hspace{-0.4em} |J(\pi)-\hat J_N(\pi)| \\
\le& C \Biggl(
\int_{r}^{B} \sqrt{\frac{\log \mathcal N_{\infty,1}(\mathcal R_{\Pi_{\mathrm{eq}}},\varepsilon)}{N}} \, d\varepsilon
\Biggr)\\
&+\frac{8r}{\sqrt{N}} 
+ B \sqrt{\frac{\log(2/\delta)}{2N}},
\end{split}
\end{equation*}

where $C$ is an absolute numeric constant, $J(\pi)$ is the
population expected return and $\hat J_N(\pi)=\tfrac{1}{N}\sum_{i=1}^N R(\pi;\tau_i)$
is the empirical return on $N$ i.i.d.\ episodes $\tau_1,\dots,\tau_N$.
\end{theorem}
 

\begin{corollary}
\label{cor:smaller_entropy}
Under the same assumptions as Theorem~\ref{thm:exact_full_proof}, for any
$r\in(0,B)$ and $\delta\in(0,1)$, the upper bound in
Theorem~\ref{thm:exact_full_proof} for $\Pi_{\mathrm{eq}}$ is at most the
same bound obtained by replacing $\Pi_{\mathrm{eq}}$ with $\Pi$. 
By Lemma~\ref{lem:app-retcov_full}, the return-class covering numbers can be
bounded by those of the policy class with radius scaled by $1/L_R$.
Mathematically, following Theorem \ref{theorem:covering}, for every $\varepsilon>0$,
\begin{equation}
\log\mathcal N_{\infty,1}\big(\Pi_{\text{eq}} ,\varepsilon / L_R\big)
\le
\log\mathcal N_{\infty,1}\big(\Pi ,\varepsilon / L_R\big),
\end{equation}
hence the upper bound in Theorem~\ref{thm:exact_full_proof} is no larger when evaluated on $\Pi_{\mathrm{eq}}$.
\end{corollary}

The equivariance regulariser projects policies onto a smaller fixed-point subspace $\Pi_{\text{eq}}$, which provably has covering numbers no larger than $\Pi$. The return class inherits this reduction via the Lipschitz map, so the Dudley entropy integral for $\Pi_{\text{eq}}$ is bounded by that of $\Pi$. As such, the upper bound on the generalisation gap is no larger for $\Pi_{\text{eq}}$ compared to $\Pi$.


\subsection{Generalisability of PRISM} 
   
We now study the generalisability of PRISM, which does not necessarily converge to the reflection-equivariant subspace exactly. 
Rather, PRISM might converge to an approximately reflection-equivariant class. Using the orbit averaging $\mathcal{Q}$, we quantify this effect below.
\begin{definition}[approximately reflection-equivariant class]
    Approximately reflection-equivariant class is defined as \(\Pi_{approx}(\varepsilon_{eq}):=\{\pi\in\Pi: \mathcal{L}_{\text{eq}}\le\varepsilon_{eq}\}\).
\end{definition}

\begin{theorem}
\label{theorem:cov2}
Let
\(\xi:=\frac{1}{2}\sqrt{\varepsilon_{eq}/p_{\min}}\). Then for every policy \(\pi \in \Pi\),
\begin{equation}
|J(\pi)-J(Q(\pi))| \le L_R\cdot d(\pi,Q(\pi)) \le L_R \xi.
\end{equation}
Then every \(\pi\in\Pi_{approx}(\varepsilon_{eq})\) lies in the sup-ball
of radius \(\xi\) around \(\Pi_{eq}\). Consequently, for any target
covering radius \(r>\xi\), we have:
\begin{equation}
\mathcal N_{\infty,1}\big(\Pi_{approx}(\varepsilon_{eq}),r\big)\le
\mathcal N_{\infty,1}\big(\Pi_{eq},r-\xi\big).
\end{equation}
\end{theorem}

By Lipschitzness of returns, the expected return of a policy and its
projection differ by at most $L_R d(\pi,Q(\pi))$. The mismatch
$\Delta_\pi$ controls this distance, and Lemma~\ref{prop:app-approx_sup_bound}
bounds its supremum by $\xi$, giving the first inequality.
Geometrically, $\Pi_{approx}(\varepsilon_{eq})$ is contained in a
$\xi$-tube around $\Pi_{eq}$. Hence any $(r-\xi)$-cover of
$\Pi_{eq}$ yields an $r$-cover of $\Pi_{approx}(\varepsilon_{eq})$,
proving the covering-number relation (see Appendix \ref{app:approx} for a detailed proof).

\begin{theorem}\label{theorem:approx-bound}
With $\mathcal R_{\Pi_{\mathrm{eq}}}=\{ \tau\mapsto R(\pi;\tau) : \pi\in\Pi_{\mathrm{eq}}\}$,
fix any accuracy parameter $r\in(0,B)$ and confidence $\delta\in(0,1)$. Then with probability at least $1-\delta$,
\begin{equation*}
\begin{split}
  &\sup_{\mathclap{\hspace{25pt} \pi\in\Pi_{approx}(\varepsilon_{eq})}} \hspace{0.5em} |J(\pi)-\hat J_N(\pi)|\\
  \le&
  C\left(\int_{r}^{B} \sqrt{\frac{\log\mathcal N_{\infty,1}(\mathcal R_{\Pi_{\mathrm{eq}}},\varepsilon)}{N}}  d\varepsilon\right) \\
  &+  \frac{8r}{\sqrt{N}} + B\sqrt{\frac{\log(2/\delta)}{2N}} + 2L_R\xi.
\end{split}
\end{equation*}
\end{theorem}

For $\pi\in\Pi_{approx}(\varepsilon_{eq})$, decompose the generalisation error
relative to its projection $Q(\pi)\in\Pi_{eq}$. The differences in population
returns $|J(\pi)-J(Q(\pi))|$ and in empirical returns
$|\hat J_N(\pi)-\hat J_N(Q(\pi))|$ are bounded by $L_R\xi$
(Theorem~\ref{theorem:cov2}). The middle term
$|J(Q(\pi))-\hat J_N(Q(\pi))|$ is the generalisation error of an
equivariant policy. Taking supremum, an equivariant bound is obtained
(Theorem~\ref{thm:exact_full_proof}) plus $2L_R\xi$. Detailed proofs are in Appendix \ref{app:approx}. 

\begin{corollary}
\label{cor:approx}
Under the same assumptions as Theorem~\ref{theorem:approx-bound}, for any
$r\in(0,B)$ and $\delta\in(0,1)$, the upper bound in
Theorem~\ref{theorem:approx-bound} for $\Pi_{approx}(\varepsilon_{eq})$ is at most the
same bound obtained by replacing $\Pi_{approx}(\varepsilon_{eq})$ with $\Pi$. 
By Lemma~\ref{lem:app-retcov_full}, the return-class covering numbers can be
bounded by those of the policy class with radius scaled by $1/L_R$.
For any target
covering radius \(r>\xi\), we have
\begin{align}
\nonumber &\log \mathcal N_{\infty,1}\big(\Pi_{approx}(\varepsilon_{eq}),r / L_R\big) \\ \nonumber \le&  
\log \mathcal N_{\infty,1}\big(\Pi_{eq},(r-\xi) /  L_R\big) \\ \le&
\log \mathcal N_{\infty,1}\big(\Pi,(r-\xi) /  L_R\big).
\end{align}
Hence the upper bound in Theorem~\ref{theorem:approx-bound} is no larger when evaluated on $\Pi_{\mathrm{eq}}$.
\end{corollary}

The covering relation incurs a slack of size $\xi$, leading to bounds of the form
$N(\Pi_{approx}(\varepsilon_{eq}), r) \le N(\Pi_{\text{eq}}, r-\xi) \le N(\Pi, r-\xi)$ . 
By contrast, in  Corollary~\ref{cor:smaller_entropy}, this
slack disappears. Thus, the exact case guarantees a strict reduction in complexity, whereas the
approximate case trades a $\xi$-shift in the radius for retaining proximity
to the equivariant subspace.

\section{Experiments}\label{sec:results}


We conduct extensive experiments to verify PRISM. The code is at \href{https://github.com/EVIEHub/PRISM}{https://github.com/EVIEHub/PRISM}. 

\subsection{Experimental Settings}

\textbf{Environments.}
Four MuJoCo \citep{todorov2012mujoco} environments are used: mo-hopper-v5, mo-walker2d-v5, mo-halfcheetah-v5, and mo-swimmer-v5. Table \ref{tab:env} in Appendix \ref{app:A} displays the environments and their dimensions, highlighting the diversity in space complexity. As a result, a method must be able to find general solutions applicable to various MORL challenges, instead of being just tailored to one specific type of problem. Furthermore, the division of asymmetric and symmetric state and action spaces to model equivariance is detailed in Appendix \ref{app:A}.


\textbf{Baselines.} 
PRISM is adaptable to any off-the-shelf MORL algorithm. In this work, CAPQL \citep{lu2023multi} is used as a backbone model, which is a method that trains a single universal network to cover the entire preference space and approximate the Pareto front. 
We produce 
(1) \textbf{oracle:} instead of artificially setting a reward channel to be sparse, this baseline model can be seen as the gold standard, and 
    (2) \textbf{baseline:} instead of utilising the proposed reward shaping model, this method uses CAPQL \citep{lu2023multi} and only observes the sparse rewards.

\textbf{Evaluation.} We use hypervolume (HV), Expected Utility Metric (EUM), and one distributional metric, Variance Objective (VO) \citep{NEURIPS2023_32285dd1}, for evaluation. The used hyperparameters, together with a detailed explanation of evaluation metrics, can be found in Appendix \ref{app:B}.

\subsection{Empirical Results}
 


\textbf{Reward Sparsity Sensitivity.}
Figure \ref{fig:spar} illustrates the sensitivity of MORL agents to varying levels of reward sparsity. Across all environments, we observe a sharp decline in HV when one objective is made extremely sparse, with reductions ranging from 20 to 40\% relative to the dense setting. 
These results confirm that sparse objectives worsen policy quality, as agents tend to neglect long-term sparse signals in favour of denser objectives. For the rest of the paper, we continue with the most difficult setting where extreme sparsity is imposed on the first reward objective. 

\begin{figure*}[t!]
    \centering
    \begin{subfigure}[b]{0.23\textwidth}
        \centering
        \includegraphics[width=0.9\textwidth]{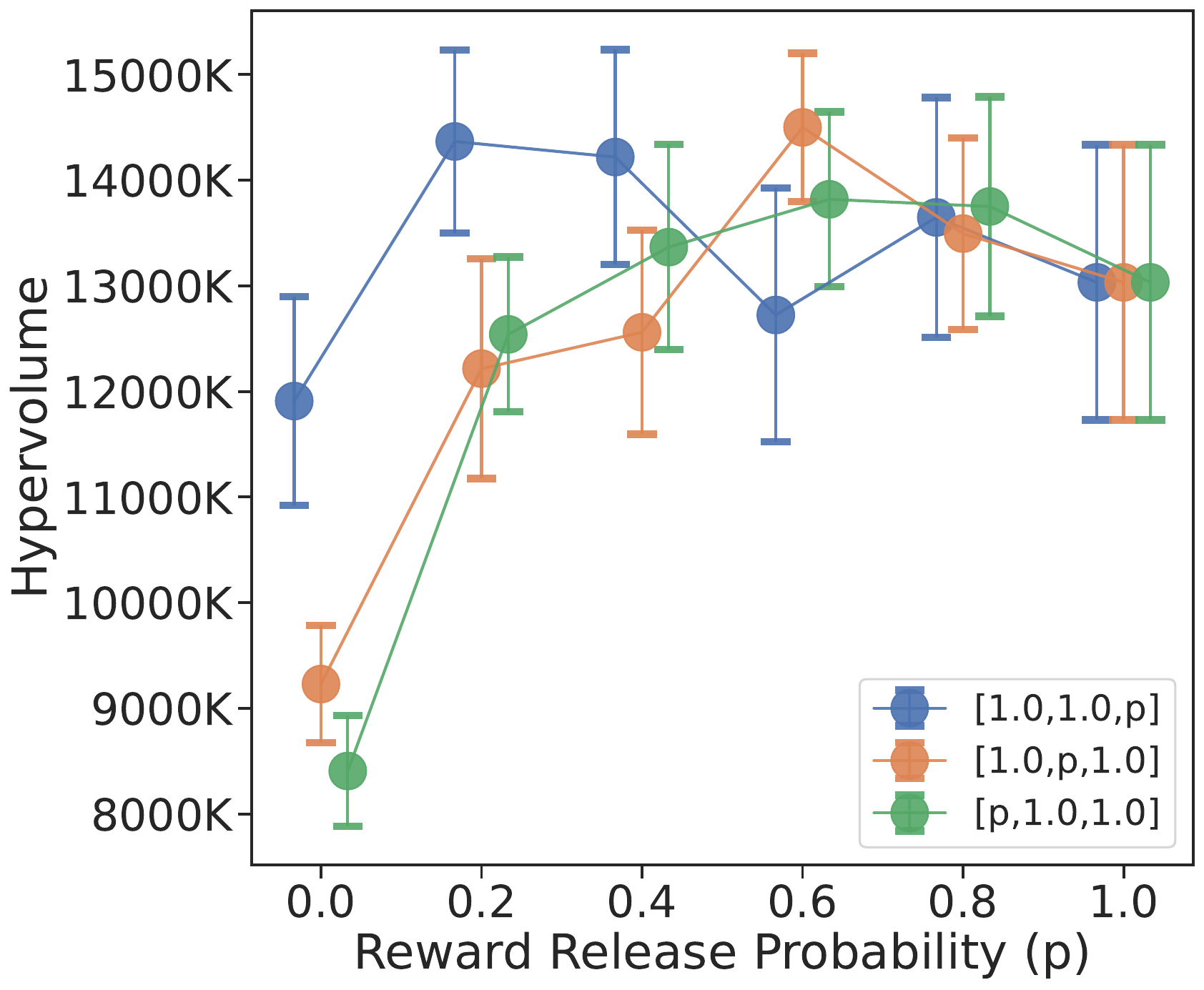}
        \caption{Mo-hopper-v5}
        \label{fig:spar-hop}
    \end{subfigure}
    \hfill 
    \begin{subfigure}[b]{0.23\textwidth}
        \centering
        \includegraphics[width=0.9\textwidth]{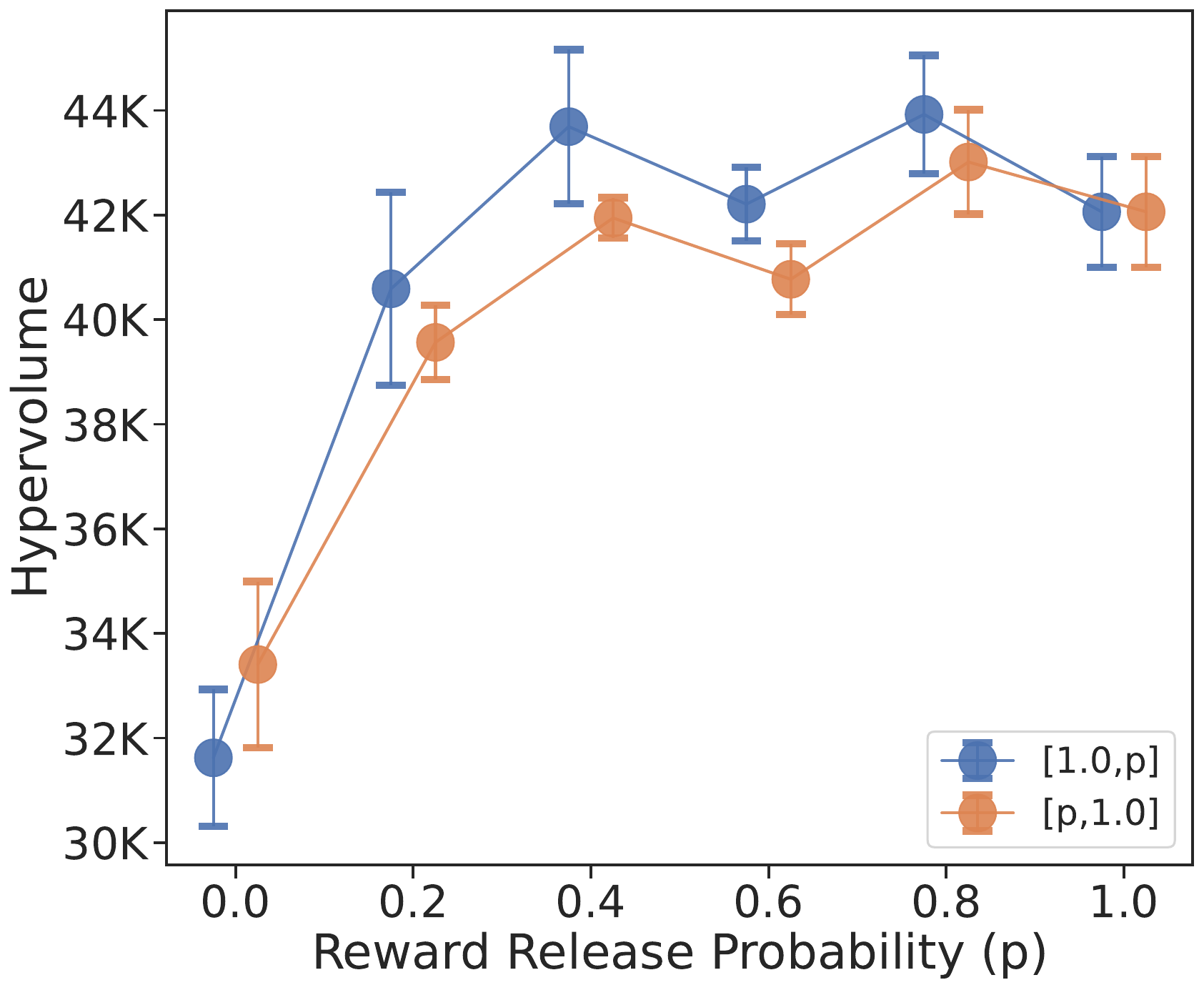}
        \caption{Mo-walker2d-v5}
        \label{fig:spar-walker}
    \end{subfigure}
    \hfill 
    \begin{subfigure}[b]{0.23\textwidth}
        \centering
        \includegraphics[width=0.9\textwidth]{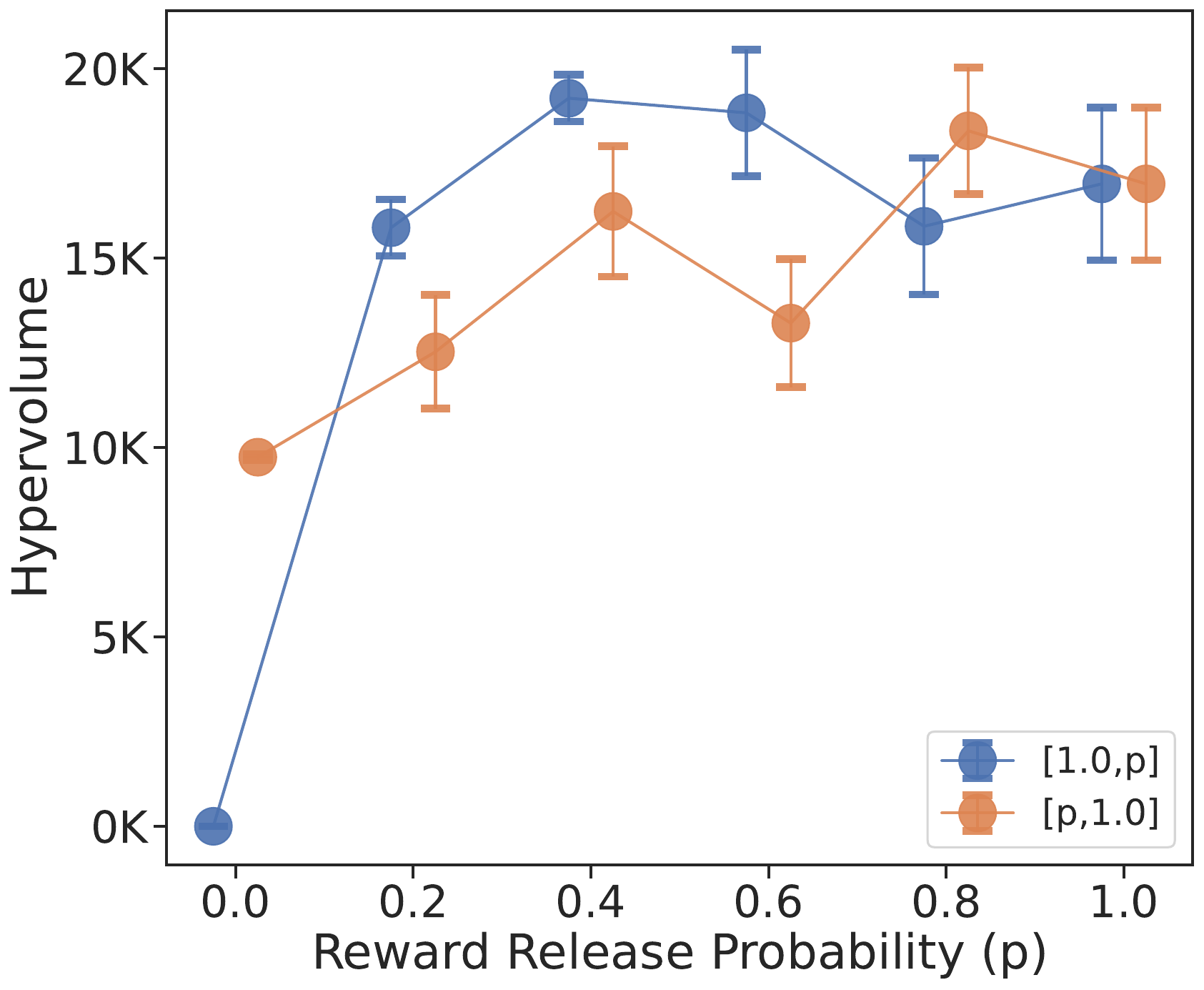}
        \caption{Mo-halfcheetah-v5}
        \label{fig:spar-chee}
    \end{subfigure}
    \hfill 
    \begin{subfigure}[b]{0.23\textwidth}
        \centering
        \includegraphics[width=0.9\textwidth]{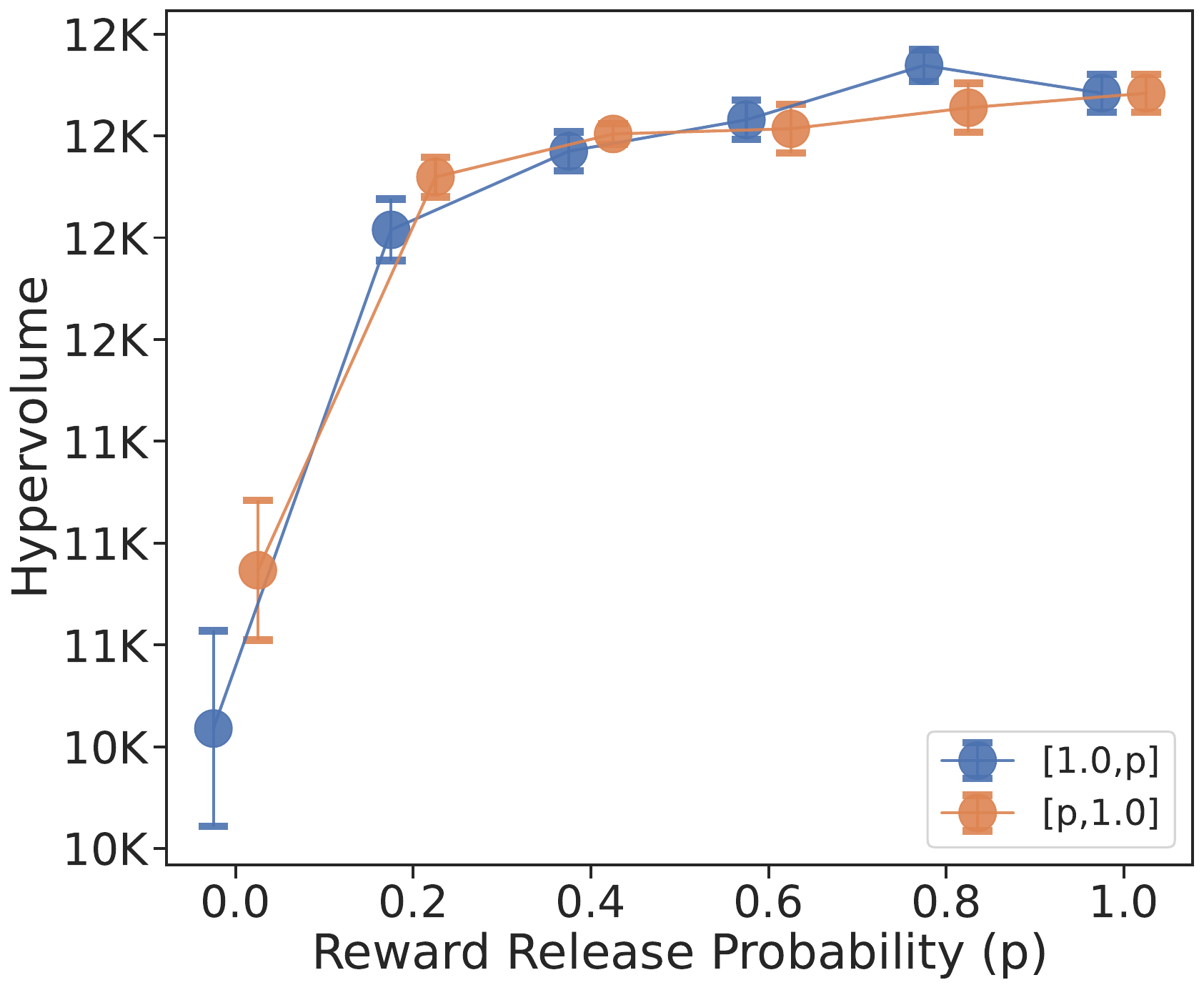}
        \caption{Mo-swimmer-v5}
        \label{fig:spar-swim}
    \end{subfigure}
    \caption{The obtained hypervolume for various levels of sparsity amongst various dimensions.}
    \label{fig:spar}
\end{figure*}

\textbf{Return Distribution of Policy.}
Figure~\ref{fig:four_figures} illustrates the impact of mixed sparsity on MORL across the considered environments. Each subplot compares the approximated Pareto fronts obtained when objective one is dense (blue dots) versus when it is made sparse (orange dots), while keeping all other objectives dense. Extreme sparsity is imposed, where the sparse reward is released at the end of an episode. The results demonstrate a consistent pattern across all environments: when objective one becomes sparse, agents systematically fail to discover high-performing solutions along this dimension, instead concentrating their learning efforts on the remaining dense objectives. 

\begin{figure*}[h!]
    \centering
    \begin{subfigure}[b]{0.23\textwidth}
        \centering
        \includegraphics[width=0.9\textwidth]{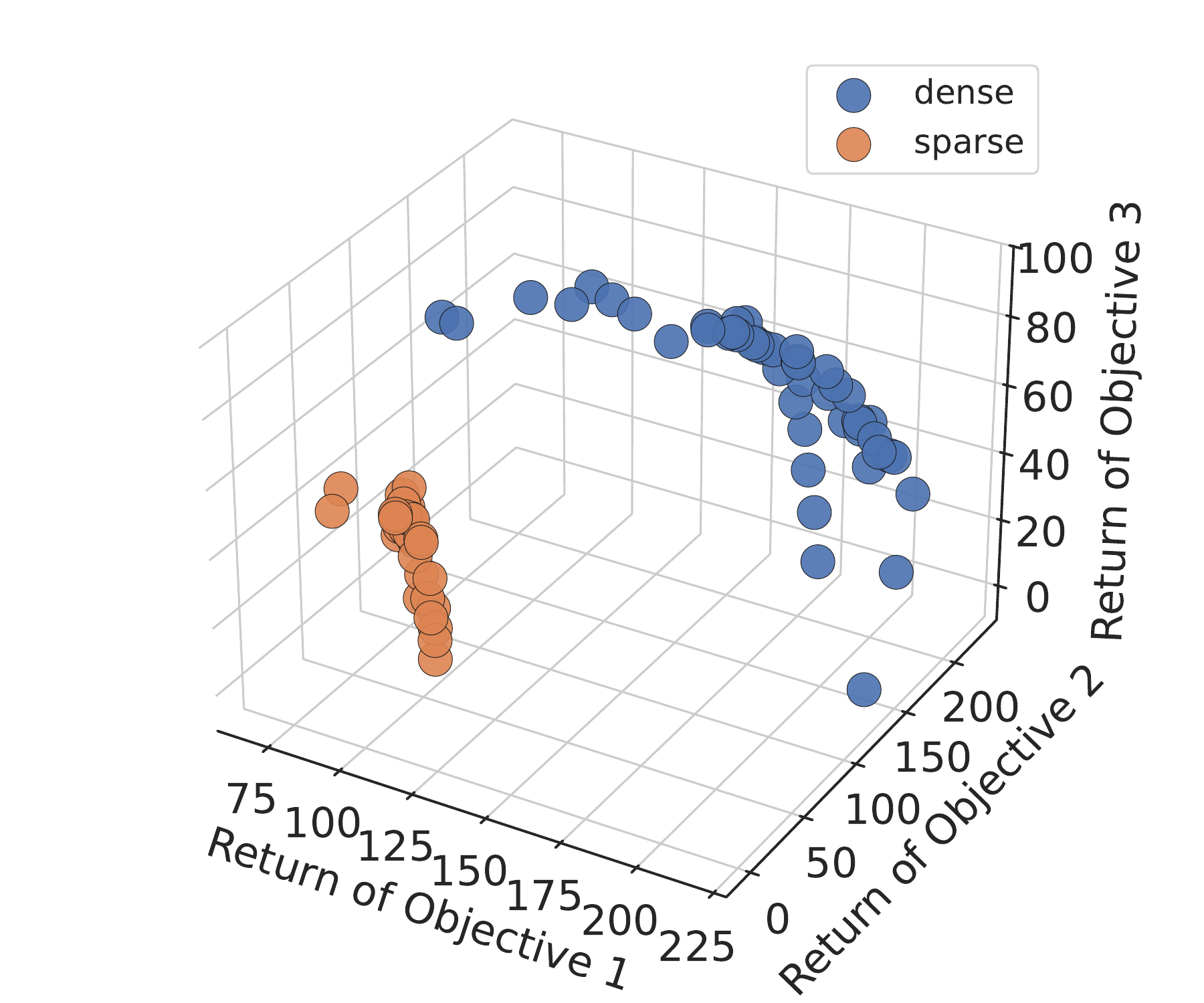}
        \caption{Mo-hopper-v5}
        \label{fig:a}
    \end{subfigure}
    \hfill 
    \begin{subfigure}[b]{0.23\textwidth}
        \centering
        \includegraphics[width=0.9\textwidth]{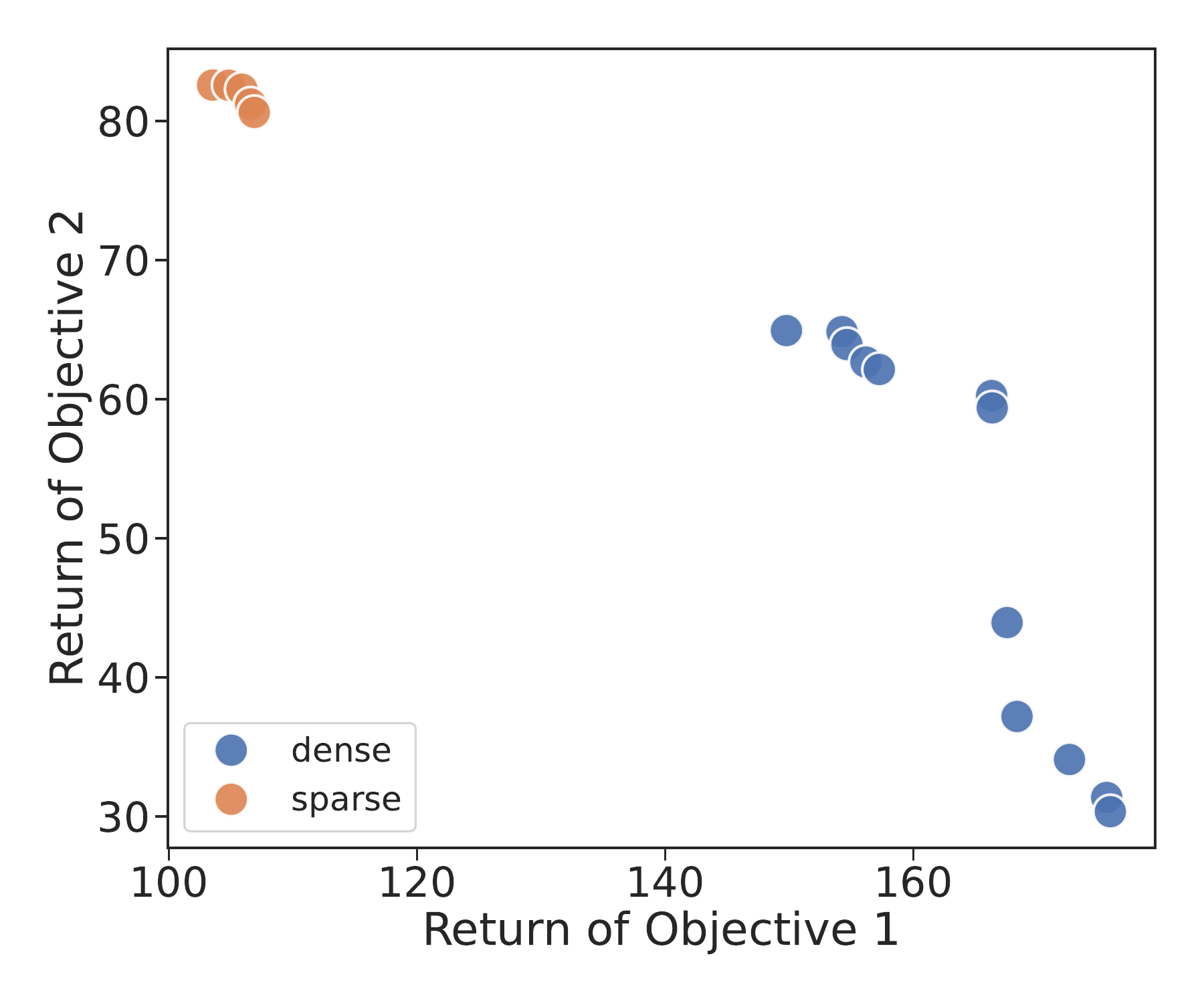}
        \caption{Mo-walker2d-v5}
        \label{fig:b}
    \end{subfigure}
    \hfill 
    \begin{subfigure}[b]{0.23\textwidth}
        \centering
        \includegraphics[width=0.9\textwidth]{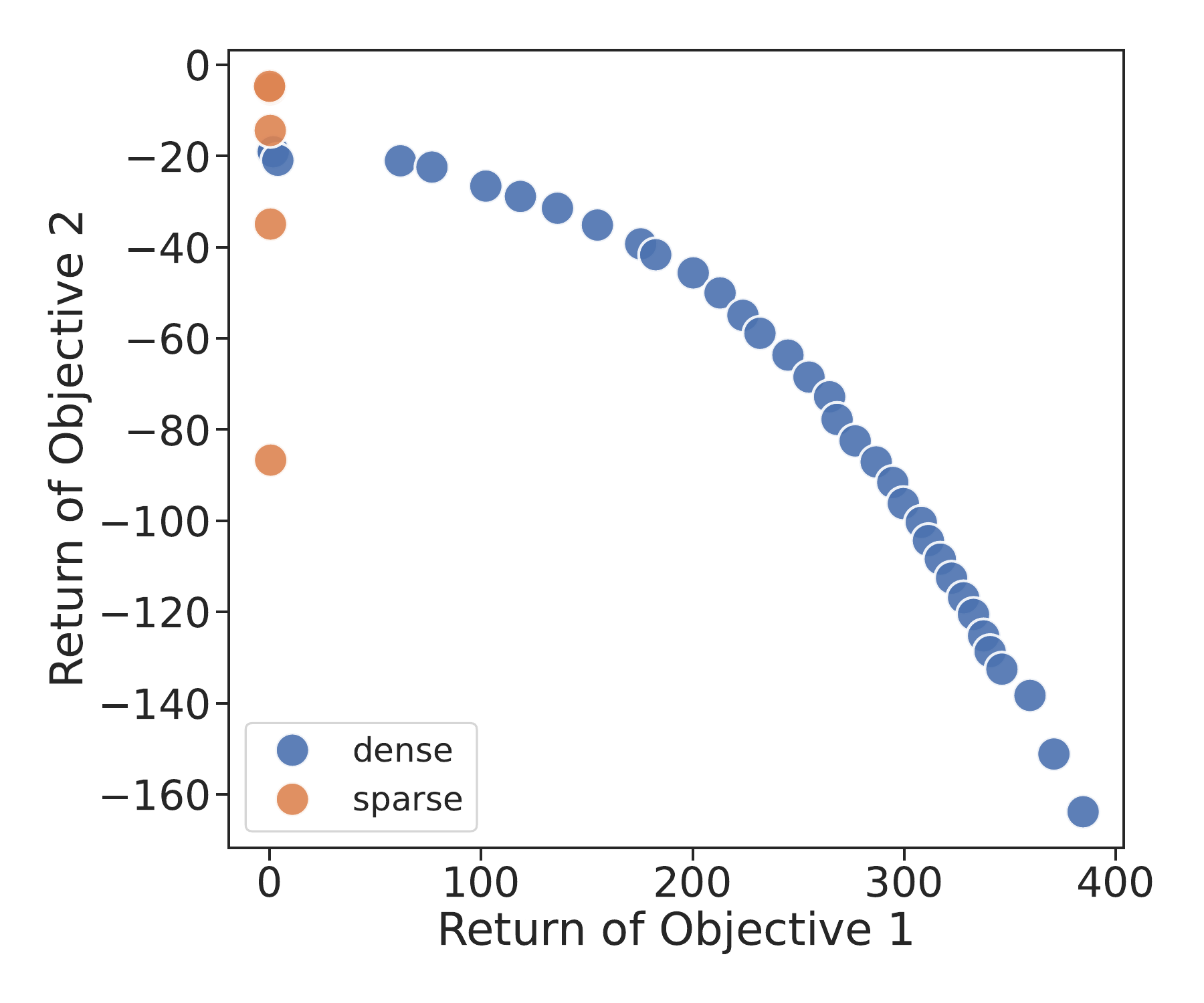}
        \caption{Mo-halfcheetah-v5}
        \label{fig:c}
    \end{subfigure}
    \hfill 
    \begin{subfigure}[b]{0.23\textwidth}
        \centering
        \includegraphics[width=0.9\textwidth]{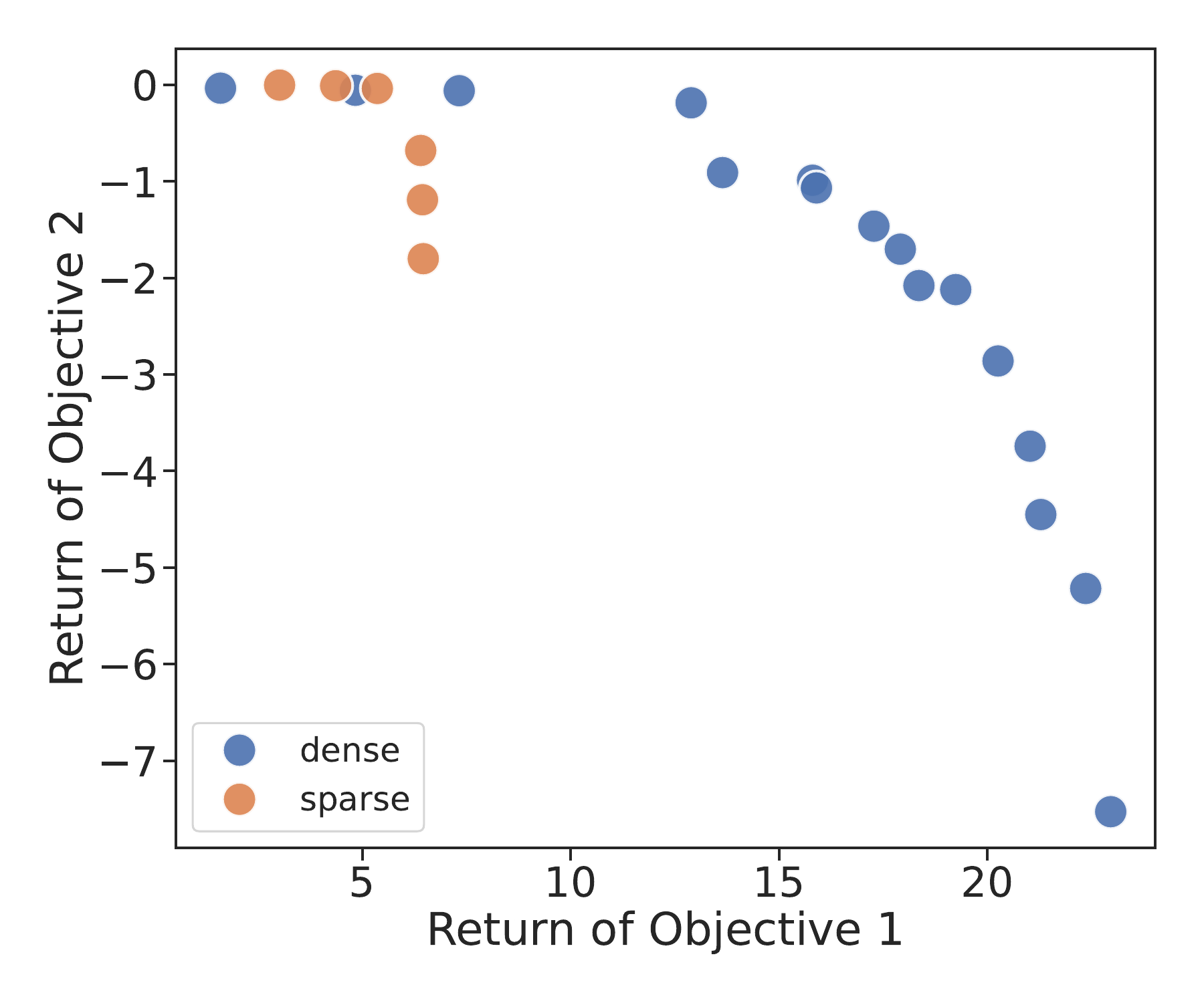}
        \caption{Mo-swimmer-v5}
        \label{fig:d}
    \end{subfigure}
    \caption{The approximated Pareto front for dense rewards (blue dots) and sparse rewards (orange dots) for the first reward objective.}
    \label{fig:four_figures}
\end{figure*}


\textbf{Comparison Experiments.}
Table \ref{tab:all} reports the obtained results for HV, EUM, and VO. The results are averaged over 10 trials, with the standard deviations shown in grey.
 

PRISM consistently outperforms both the oracle and baseline across environments. For mo-hopper-v5, PRISM improves hypervolume by 21.5\% over the oracle ($1.58 \times 10^7$ compared to $1.30 \times 10^7$) and 88\% over the baseline. Similar gains are observed for mo-walker2d-v5, where PRISM achieves a 13\% HV improvement over oracle and 43\% over the baseline. Notably, in mo-halfcheetah-v5, PRISM yields a 32\% improvement in HV compared to the oracle ($2.25 \times 10^4$ against $1.70 \times 10^4$) and more than doubles the sparse result. These improvements imply that PRISM not only restores solutions lost under sparsity but also expands the range of trade-offs accessible to the agent. Improvements in EUM follow the same trend, with increases of up to 50\% compared to the baseline. The concurrent increase in EUM demonstrates that these solutions provide higher expected utility, confirming that PRISM learns policies that are both diverse and practically useful. 

On distributional metrics, PRISM delivers more consistent performance than both the oracle and baseline. VO in mo-hopper-v5 increases from 43.36 (baseline) and 59.07 (oracle) to 66.66 under PRISM, and mo-walker2d-v5 shows a 51\% gain over the baseline. These gains are crucial because they indicate that PRISM does not simply maximise HV by focusing on extreme solutions, but also produces Pareto fronts that are better balanced, robust, and fair across objectives. Figure \ref{fig:four_figures2} in Appendix \ref{app:pareto-fronts}, which shows the approximated Pareto fronts, aligns with these results.

\begin{table}[t!]
    \centering
    \scriptsize
    \resizebox{\columnwidth}{!}{
    \begin{threeparttable}
    
    \caption{Experimental results. We report the average hypervolume (HV), Expected Utility Metric (EUM), and Variance Objective (VO) over 10 trials, with the standard error shown in grey. The largest (best) values are in bold font.}
    \label{tab:all}

    \begin{tabular}{l l l l l}
    \toprule
   Environment & Metric & Oracle & Baseline & PRISM \\
    \midrule \midrule
      \multirow{3}{*}{Mo-hopper-v5} 
      & HV ($\times 10^7$) & 1.30\,$\pm$\,\textcolor{gray}{0.13} & 0.84\,$\pm$\,\textcolor{gray}{0.05} & \textbf{1.58}\,$\pm$\,\textcolor{gray}{0.05} \\
      & EUM & 129.04\,$\pm$\,\textcolor{gray}{7.96} & 97.64\,$\pm$\,\textcolor{gray}{4.18} & \textbf{147.43}\,$\pm$\,\textcolor{gray}{2.61} \\
      & VO & 59.07\,$\pm$\,\textcolor{gray}{3.45} & 43.36\,$\pm$\,\textcolor{gray}{1.61} & \textbf{66.66}\,$\pm$\,\textcolor{gray}{1.40} \\
      \midrule
      \multirow{3}{*}{Mo-walker2d-v5} 
      & HV ($\times 10^4$) & 4.21\,$\pm$\,\textcolor{gray}{0.11} & 3.34\,$\pm$\,\textcolor{gray}{0.16} & \textbf{4.77}\,$\pm$\,\textcolor{gray}{0.07} \\
      & EUM & 107.58\,$\pm$\,\textcolor{gray}{2.86} & 82.13\,$\pm$\,\textcolor{gray}{4.34} & \textbf{120.43}\,$\pm$\,\textcolor{gray}{1.64} \\
      & VO & 53.22\,$\pm$\,\textcolor{gray}{1.39} & 39.18\,$\pm$\,\textcolor{gray}{2.49} & \textbf{59.35}\,$\pm$\,\textcolor{gray}{0.80} \\
      \midrule
      \multirow{3}{*}{Mo-halfcheetah-v5} 
      & HV ($\times 10^4$) & 1.70\,$\pm$\,\textcolor{gray}{0.20} & 0.97\,$\pm$\,\textcolor{gray}{0.00} & \textbf{2.25}\,$\pm$\,\textcolor{gray}{0.18} \\
      & EUM & 81.29\,$\pm$\,\textcolor{gray}{21.85} & -1.46\,$\pm$\,\textcolor{gray}{0.27} & \textbf{89.94}\,$\pm$\,\textcolor{gray}{15.33} \\
      & VO & 36.84\,$\pm$\,\textcolor{gray}{10.06} & -1.01\,$\pm$\,\textcolor{gray}{0.20} & \textbf{40.72}\,$\pm$\,\textcolor{gray}{7.02} \\
      \midrule
      \multirow{3}{*}{Mo-swimmer-v5} 
      & HV ($\times 10^4$) & \textbf{1.21}\,$\pm$\,\textcolor{gray}{0.00} & 1.09\,$\pm$\,\textcolor{gray}{0.02} & \textbf{1.21}\,$\pm$\,\textcolor{gray}{0.00} \\
      & EUM & 9.41\,$\pm$\,\textcolor{gray}{0.12} & 4.10\,$\pm$\,\textcolor{gray}{0.80} & \textbf{9.44}\,$\pm$\,\textcolor{gray}{0.14} \\
      & VO & 4.22\,$\pm$\,\textcolor{gray}{0.08} & 1.58\,$\pm$\,\textcolor{gray}{0.40} & \textbf{4.24}\,$\pm$\,\textcolor{gray}{0.07} \\
    \bottomrule
    \end{tabular}
    \begin{tablenotes}[flushleft]
\item[\hspace{-\labelsep}] 
\end{tablenotes}
    \end{threeparttable}
    }
\end{table}

\begin{figure}[h!]
    \centering
    \begin{subfigure}[b]{0.23\textwidth}
        \centering
        \includegraphics[width=\textwidth]{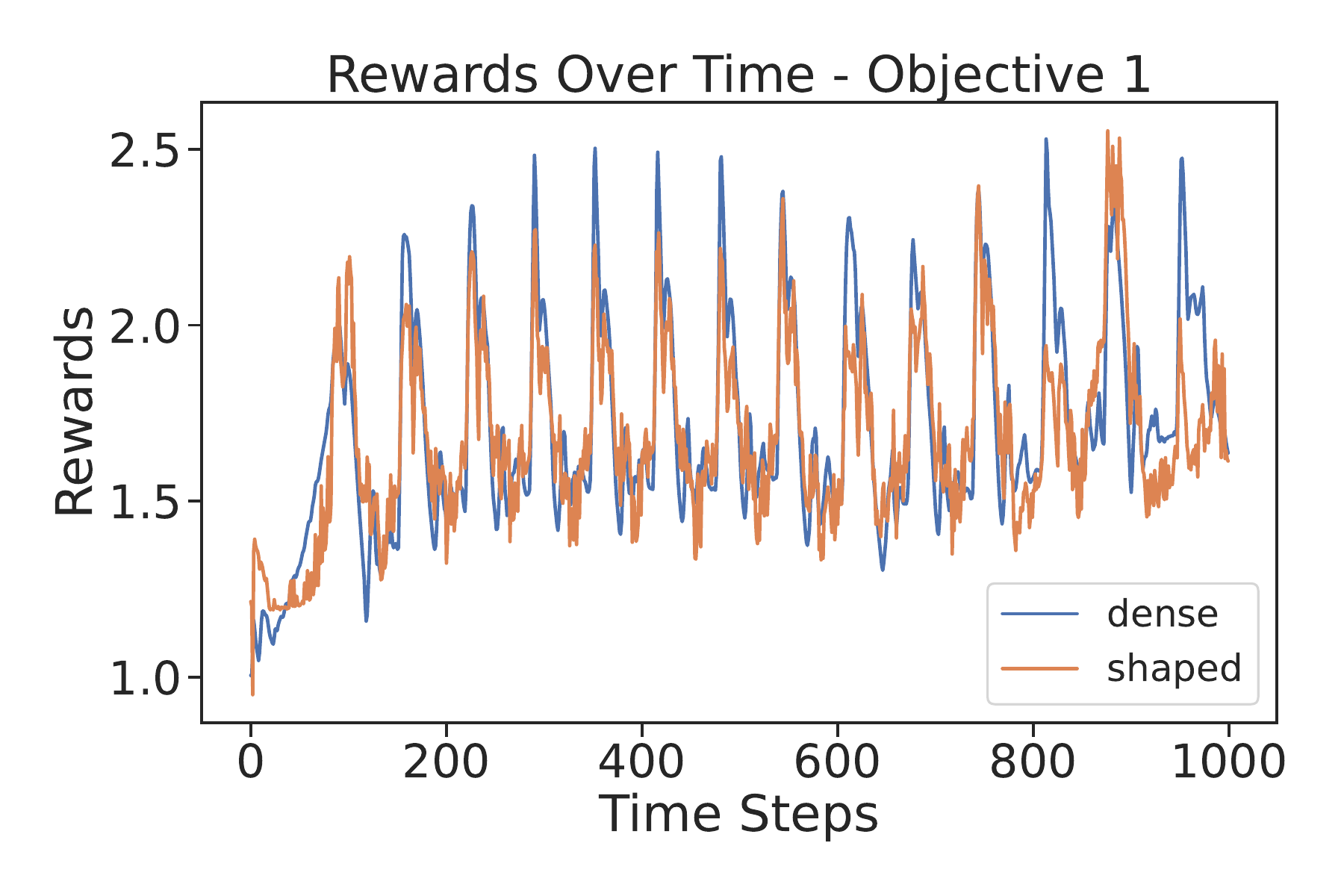}
        \caption{Full-episode stability}
        \label{fig:a3}
    \end{subfigure}
    \begin{subfigure}[b]{0.23\textwidth}
        \centering
        \includegraphics[width=\textwidth]{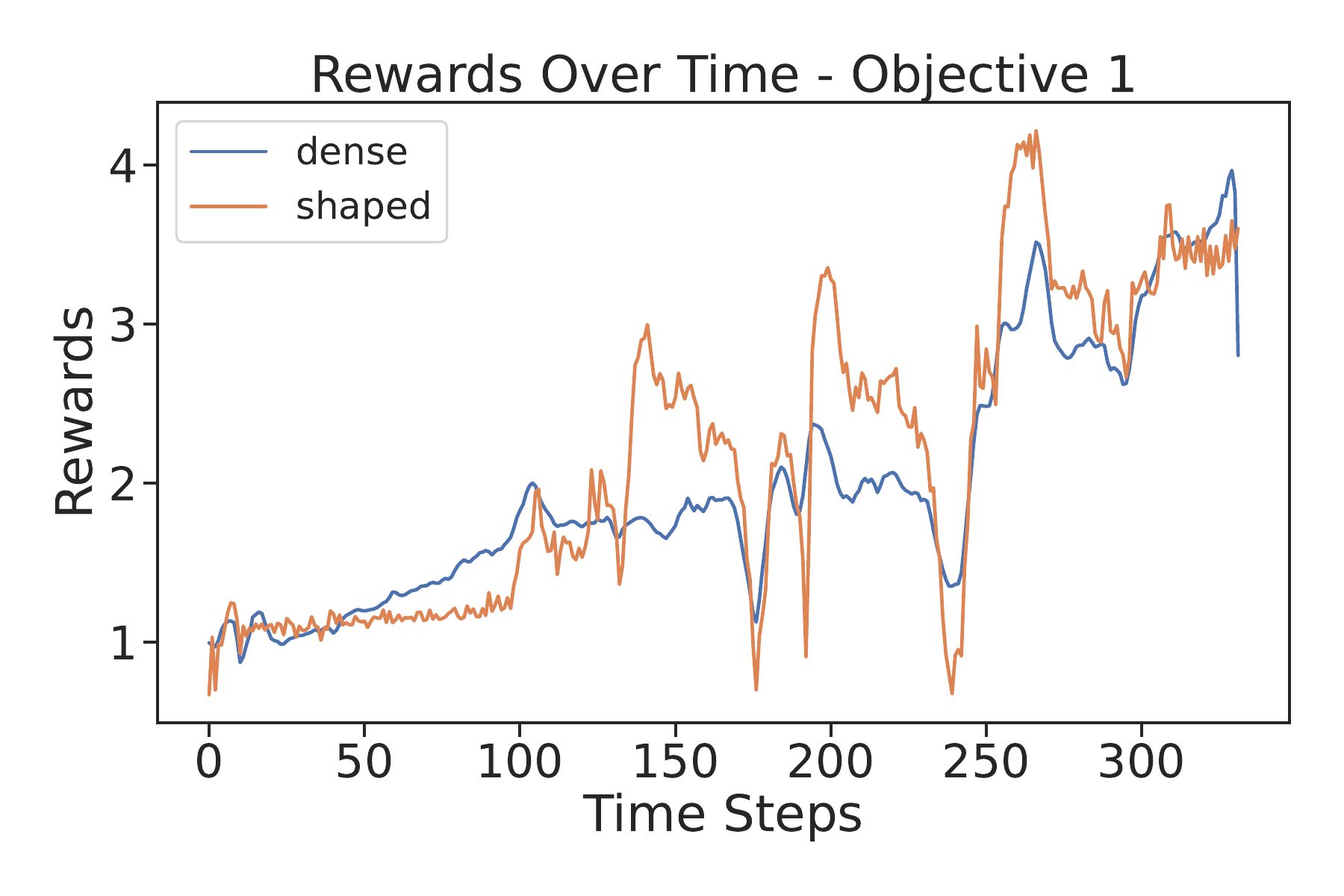}
        \caption{Signal optimisation}
        \label{fig:b3}
    \end{subfigure}
    \caption{The dense (blue line) and shaped rewards (orange line) over time for mo-walker2d-v5 and the first reward objective.}
    \label{fig:two_figures}
\end{figure}

We provide two distinct examples to analyse the behaviour of the learned reward signals compared to the oracle for mo-walker2d-v5. Figure \ref{fig:a3} illustrates a full 1000-step episode. The shaped reward is highly correlated with the dense reward throughout the entire trajectory. The alignment of peaks and troughs confirms that ReSymNet captures the dynamics of the environment, ensuring accurate credit assignment without temporal drift.

Figure \ref{fig:b3} highlights a key theoretical advantage of ReSymNet. In high-performance regions (e.g., steps 250–270), the shaped reward amplifies the signal, exceeding the magnitude of the oracle. By creating steeper gradients for desirable behaviours, the shaped reward can provide more effective guidance than the raw environmental signal, explaining why PRISM is capable of outperforming the oracle.

\textbf{Ablation Study.} 
We analyse the following ablation models (w/o is the abbreviation for without), which remove several aspects of the reward shaping model or the equivariance loss: 
    (1) \textbf{PRISM:} This is the proposed method, involving all components, 
    (2) \textbf{w/o residual:} This ablation model removes the two residual blocks from the reward shaping model, 
    (3) \textbf{w/o dense rewards:} We remove the dense rewards as input features to the reward model, 
    (4) \textbf{w/o ensemble:} We remove the ensemble of reward shaping models, and only employ one, 
    (5) \textbf{w/o refinement:} Rather than updating the reward shaping model with expert trajectories, this approach merely trains the reward shaping model using the random trajectories collected at first, and
    (6) \textbf{w/o loss:} We remove the equivariance loss term and merely use the reward shaping model. We also include two ablation studies that remove ReSymNet from PRISM and replace the reward shaping model as follows: (7) \textbf{uniform}: Distributes the episodic sparse reward $R^{\mathrm{sp}}(\tau)$ equally across all $T$ timesteps, and 
    (8) \textbf{random}: Samples random weights $\alpha_t \sim \mathcal{U}(-1, 1)$ for each timestep, normalises to sum to one, and scales by the total reward.

The ablation results in Tables \ref{tab:ablation} and \ref{tab:ablation2} in Appendix \ref{app:ablation} highlight the contribution of individual components. Removing residual connections reduces HV and EUM across all environments (e.g., mo-hopper-v5 EUM falls from 147.43 to 128.40), showing their importance for scaled opportunity value. Excluding dense reward features or ensembles also lowers performance, but only moderately, suggesting that state–action features already contain substantial signal. Interestingly, removing iterative refinement barely reduces performance; in some cases, such as mo-halfcheetah-v5, HV, and EUM remain comparable or even slightly higher than the full model. This implies that shaping rewards from a broad set of random trajectories is already highly effective. Removing the symmetry loss reduces performance across environments, indicating that the loss term successfully reduces the search space. Similar patterns are observed for VO. Considering ReSymNet, uniform achieves moderate performance by providing per-step gradients and leveraging SymReg, while random performs poorly due to noisy, misleading rewards. PRISM consistently outperforms both by learning reward decomposition with ReSymNet and enforcing structural consistency via SymReg, enabling accurate credit assignment in complex multi-objective tasks. 

\section{Conclusion}\label{sec:conclusion}


This work proposes Parallel Reward Integration with reflectional Symmetry for Multi-objective reinforcement learning (PRISM), a framework designed to tackle sample inefficiency in heterogeneous multi-objective reinforcement learning, particularly in environments with sparse rewards. Our approach is centred around two key contributions: (1) ReSymNet, a theory-inspired reward model that leverages residual blocks to align reward channels by learning a refined `scaled opportunity value', and (2) SymReg, a novel regulariser that enforces reflectional symmetry as an inductive bias in the policy’s action space. We prove that PRISM restricts policy search to a reflection-equivariant subspace, a projection of the original policy space with provably reduced hypothesis complexity; in this way, the generalisability is rigorously improved. Extensive experiments on MuJoCo benchmarks show that PRISM consistently outperforms even a strong oracle with full reward access in terms of a wide range of metrics, including HV, EUM, and VO. 


\section*{Acknowledgements}

K. Qian was supported in part by the UKRI Grant EP/Y03516X/1 for the UKRI Centre for Doctoral Training in Machine Learning Systems (\href{https://mlsystems.uk/}{https://mlsystems.uk/}).






\bibliography{icml2026_conference}

@inproceedings{felten_toolkit_2023,
author       = {Florian Felten and
                  Lucas N. Alegre and
                  Ann Now{\'{e}} and
                  Ana L. C. Bazzan and
                  El{-}Ghazali Talbi and
                  Gr{\'{e}}goire Danoy and
                  Bruno C. da Silva},
  title        = {A Toolkit for Reliable Benchmarking and Research in Multi-Objective
                  Reinforcement Learning},
  booktitle    = {37th International Conference
                  on Neural Information Processing Systems (NIPS 2023)},
publisher = {{Curran Associates}},
  year         = {2023},
}

@inproceedings{zintgraf2015quality,
  title={Quality assessment of {MORL} algorithms: A utility-based approach},
  author={Zintgraf, Luisa M and Kanters, Timon V and Roijers, Diederik M and Oliehoek, Frans and Beau, Philipp},
  booktitle={ 24th Annual Machine Learning Conference of Belgium and the Netherlands},
  year={2015}
}

@inproceedings{DBLP:conf/iros/ChenGBJ19,
  author       = {Xi Chen and
                  Ali Ghadirzadeh and
                  M{\aa}rten Bj{\"{o}}rkman and
                  Patric Jensfelt},
  title        = {Meta-Learning for Multi-objective Reinforcement Learning},
  booktitle    = {2019 {IEEE/RSJ} International Conference on Intelligent Robots and
                  Systems ({IROS} 2019)},
  pages        = {977--983},
  publisher    = {{IEEE}},
  year         = {2019},

}

@inproceedings{DBLP:conf/nips/PolWHOW20,
  author       = {Elise van der Pol and
                  Daniel E. Worrall and
                  Herke van Hoof and
                  Frans A. Oliehoek and
                  Max Welling},

  title        = {{MDP} Homomorphic Networks: Group Symmetries in Reinforcement Learning},
  booktitle    = {33st Annual Conference
                  on Neural Information Processing Systems
                  (NIPS 2020)},
  year         = {2020},
}

@article{DBLP:journals/corr/abs-2007-03437,
  title={Group equivariant deep reinforcement learning},
  author={Mondal, Arnab Kumar and Nair, Pratheeksha and Siddiqi, Kaleem},
  journal={arXiv preprint arXiv:2007.03437},
  year={2020}
}

@inproceedings{DBLP:conf/corl/WangWZP21,
  author       = {Dian Wang and
                  Robin Walters and
                  Xupeng Zhu and
                  Robert Platt Jr.},

  title        = {Equivariant {Q} Learning in Spatial Action Spaces},
  booktitle    = {5th Conference on Robot Learning}, 
  series       = {{PMLR}},
  volume       = {164},
  pages        = {1713--1723},
  publisher    = {{PMLR}},
  year         = {2021},
}

@article{tang2024beyond,
  title={Beyond Simple Sum of Delayed Rewards: Non-Markovian Reward Modeling for Reinforcement Learning},
  author={Tang, Yuting and Cai, Xin-Qiang and Pang, Jing-Cheng and Wu, Qiyu and Ding, Yao-Xiang and Sugiyama, Masashi},
  journal={arXiv preprint arXiv:2410.20176},
  year={2024}
}

@article{DBLP:journals/tmlr/TangC00LS25,
  author       = {Yuting Tang and
                  Xin{-}Qiang Cai and
                  Yao{-}Xiang Ding and
                  Qiyu Wu and
                  Guoqing Liu and
                  Masashi Sugiyama},
  title        = {Reinforcement Learning from Bagged Reward},
  journal      = {Transactions on Machine Learning Research},
  year         = {2025},
 
}

@inproceedings{DBLP:conf/iclr/RenG0022,
  author       = {Zhizhou Ren and
                  Ruihan Guo and
                  Yuan Zhou and
                  Jian Peng},
  title        = {Learning Long-Term Reward Redistribution via Randomized Return Decomposition},
  booktitle    = {Tenth International Conference on Learning Representations ({ICLR} 2022)},
  publisher    = {OpenReview.net},
  year         = {2022},
}

@article{he2020resnet,
  author       = {Fengxiang He and
                  Tongliang Liu and
                  Dacheng Tao},
  title        = {Why ResNet Works? Residuals Generalize},
  journal      = {{IEEE} Transactions on Neural Networks and Learning Systems},
  volume       = {31},
  number       = {12},
  pages        = {5349--5362},
  year         = {2020},

}

@article{bartlett2017spectrally,
  title={Spectrally-normalized margin bounds for neural networks},
  author={Bartlett, Peter L and Foster, Dylan J and Telgarsky, Matus J},
  journal={Advances in neural information processing systems},
  volume={30},
  year={2017}
}

@inproceedings{DBLP:conf/atal/AlegreBRN023,
  author       = {Lucas Nunes Alegre and
                  Ana L. C. Bazzan and
                  Diederik M. Roijers and
                  Ann Now{\'{e}} and
                  Bruno C. da Silva},
  title        = {Sample-Efficient Multi-Objective Learning via Generalized Policy Improvement
                  Prioritization},
  booktitle    = {2023 International Conference on Autonomous Agents
                  and Multiagent Systems ({AAMAS} 2023)},
  pages        = {2003--2012},
  publisher    = {{ACM}},
  year         = {2023},
}

@inproceedings{DBLP:conf/nips/Gangwani0020,
  author       = {Tanmay Gangwani and
                  Yuan Zhou and
                  Jian Peng},
  title        = {Learning Guidance Rewards with Trajectory-space Smoothing},
  booktitle    = {33rd Annual Conference
                  on Neural Information Processing Systems 2020 (NIPS 2020},
  year         = {2020},
}

@inproceedings{fonseca2006improved,
  author       = {Carlos M. Fonseca and
                  Lu{\'{\i}}s Paquete and
                  Manuel L{\'{o}}pez{-}Ib{\'{a}}{\~{n}}ez},
  title        = {An Improved Dimension-Sweep Algorithm for the Hypervolume Indicator},
  booktitle    = {{IEEE} International Conference on Evolutionary Computation (CEC 2006)},
  pages        = {1157--1163},
  publisher    = {{IEEE}},
  year         = {2006},
  
}

@inproceedings{todorov2012mujoco,
  title={{MuJoCo}: {A} physics engine for model-based control},
  author={Todorov, Emanuel and Erez, Tom and Tassa, Yuval},
  booktitle={2012 IEEE/RSJ International Conference on Intelligent Robots and Systems},
  pages={5026--5033},
  year={2012},
  organization={IEEE}
}

@inproceedings{pathak2017curiosity,
  author       = {Deepak Pathak and
                  Pulkit Agrawal and
                  Alexei A. Efros and
                  Trevor Darrell},
  title        = {Curiosity-driven Exploration by Self-supervised Prediction},
  booktitle    = {34th International Conference on Machine Learning
                  ({ICML} 2017)},
  series       = {PMLR},
  volume       = {70},
  pages        = {2778--2787},
  publisher    = {{PMLR}},
  year         = {2017},
}

@article{aubret2019survey,
  title={A survey on intrinsic motivation in reinforcement learning},
  author={Aubret, Arthur and Matignon, Laetitia and Hassas, Salima},
  journal={arXiv preprint arXiv:1908.06976},
  year={2019}
}

@article{wei2025attention,
  title={Attention With System Entropy for Optimizing Credit Assignment in Cooperative Multi-Agent Reinforcement Learning},
  author={Wei, Wei and Li, Haibin and Zhou, Shiyuan and Li, Baifeng and Liu, Xue},
  journal={IEEE Transactions on Automation Science and Engineering},
  year={2025},
volume       = {22},
  pages        = {14775--14787},
  publisher={IEEE}
}

@inproceedings{wang2019generalization,
  title={On the generalization gap in reparameterizable reinforcement learning},
  author={Wang, Huan and Zheng, Stephan and Xiong, Caiming and Socher, Richard},
  booktitle={36th International Conference on Machine Learning (ICML 2019)},
  pages={6648--6658},
  year={2019},
volume       = {97},
  publisher={PMLR}
}

@article{silver2017mastering,
  title={Mastering the game of go without human knowledge},
  author={Silver, David and Schrittwieser, Julian and Simonyan, Karen and Antonoglou, Ioannis and Huang, Aja and Guez, Arthur and Hubert, Thomas and Baker, Lucas and Lai, Matthew and Bolton, Adrian and others},
  journal={Nature},
  volume={550},
  number={7676},
  pages={354--359},
  year={2017},
  publisher={Nature Publishing Group UK London}
}

@article{kiran2021deep,
  title={Deep reinforcement learning for autonomous driving: A survey},
  author={Kiran, B Ravi and Sobh, Ibrahim and Talpaert, Victor and Mannion, Patrick and Al Sallab, Ahmad A and Yogamani, Senthil and P{\'e}rez, Patrick},
  journal={IEEE transactions on intelligent transportation systems},
  volume={23},
  number={6},
  pages={4909--4926},
  year={2021},
  publisher={IEEE}
}

@inproceedings{tang2025deep,
  title={Deep reinforcement learning for robotics: A survey of real-world successes},
  author={Tang, Chen and Abbatematteo, Ben and Hu, Jiaheng and Chandra, Rohan and Mart{\'\i}n-Mart{\'\i}n, Roberto and Stone, Peter},
  booktitle={Proceedings of the AAAI Conference on Artificial Intelligence},
  volume={39},
  number={27},
  pages={28694--28698},
  year={2025}
}

@article{hambly2023recent,
  title={Recent advances in reinforcement learning in finance},
  author={Hambly, Ben and Xu, Renyuan and Yang, Huining},
  journal={Mathematical Finance},
  volume={33},
  number={3},
  pages={437--503},
  year={2023},
  publisher={Wiley Online Library}
}

@article{liu2014multiobjective,
  title={Multiobjective reinforcement learning: A comprehensive overview},
  author={Liu, Chunming and Xu, Xin and Hu, Dewen},
  journal={IEEE Transactions on Systems, Man, and Cybernetics: Systems},
  volume={45},
  number={3},
  pages={385--398},
  year={2014},
  publisher={IEEE}
}

@article{zhou2002covering,
  title={The covering number in learning theory},
  author={Zhou, Ding-Xuan},
  journal={Journal of Complexity},
  volume={18},
  number={3},
  pages={739--767},
  year={2002},
  publisher={Elsevier}
}

@article{hayes2022practical,
  title={A practical guide to multi-objective reinforcement learning and planning},
  author={Hayes, Conor F and R{\u{a}}dulescu, Roxana and Bargiacchi, Eugenio and K{\"a}llstr{\"o}m, Johan and Macfarlane, Matthew and Reymond, Mathieu and Verstraeten, Timothy and Zintgraf, Luisa M and Dazeley, Richard and Heintz, Fredrik and others},
  journal={Autonomous Agents and Multi-Agent Systems},
  volume={36},
  number={1},
  pages={26},
  year={2022},
  publisher={Springer}
}

@article{roijers2015computing,
  title={Computing convex coverage sets for faster multi-objective coordination},
  author={Roijers, Diederik Marijn and Whiteson, Shimon and Oliehoek, Frans A},
  journal={Journal of Artificial Intelligence Research},
  volume={52},
  pages={399--443},
  year={2015}
}

@article{lautenbacher2025multi,
  title={Multi-Objective Reinforcement Learning for Power Grid Topology Control},
  author={Lautenbacher, Thomas and Rajaei, Ali and Barbieri, Davide and Viebahn, Jan and Cremer, Jochen L},
  journal={arXiv preprint arXiv:2502.00040},
  year={2025}
}

@article{van2014multi,
  title={Multi-objective reinforcement learning using sets of {P}areto dominating policies},
  author={Van Moffaert, Kristof and Now{\'e}, Ann},
  journal={The Journal of Machine Learning Research},
  volume={15},
  number={1},
  pages={3483--3512},
  year={2014},
  publisher={JMLR. org}
}

@inproceedings{reymond2019pareto,
  title={Pareto-{DQN}: Approximating the {P}areto front in complex multi-objective decision problems},
  author={Reymond, Mathieu and Now{\'e}, Ann},
  booktitle={Adaptive and Learning Agents Workshop (ALA 2019)},
  year={2019}
}

@inproceedings{yang2019generalized,
  author       = {Runzhe Yang and
                  Xingyuan Sun and
                  Karthik Narasimhan},
  title        = {A Generalized Algorithm for Multi-Objective Reinforcement Learning
                  and Policy Adaptation},
  booktitle    = {33rd International Conference
                  on Neural Information Processing Systems (NIPS 2019)},
publisher = {{Curran Associates}},
  pages        = {14610--14621},
  year         = {2019},
}

@inproceedings{burda2018exploration,
  author       = {Yuri Burda and
                  Harrison Edwards and
                  Amos J. Storkey and
                  Oleg Klimov},
  title        = {Exploration by random network distillation},
  booktitle    = {7th International Conference on Learning Representations ({ICLR} 2019)},
  publisher    = {OpenReview.net},
  year         = {2019},

}

@inproceedings{ng2000algorithms,
  author       = {Andrew Y. Ng and
                  Stuart Russell},
  title        = {Algorithms for Inverse Reinforcement Learning},
  booktitle    = { Seventeenth International Conference on Machine
                  Learning {(ICML} 2000)},
  pages        = {663--670},
  publisher    = {Morgan Kaufmann},
  year         = {2000},
}

@article{arora2021survey,
  author       = {Saurabh Arora and
                  Prashant Doshi},
  title        = {A survey of inverse reinforcement learning: Challenges, methods and
                  progress},
  journal      = {Artificial Intelligence},
  volume       = {297},
  pages        = {103500},
  year         = {2021},
}

@inproceedings{lu2023multi,
  author       = {Haoye Lu and
                  Daniel Herman and
                  Yaoliang Yu},
  title        = {Multi-Objective Reinforcement Learning: Convexity, Stationarity and
                  {P}areto Optimality},
  booktitle    = { Eleventh International Conference on Learning Representations ({ICLR} 2023)},
  publisher    = {OpenReview.net},
  year         = {2023},
}

@inproceedings{basaklar2022pd,
author       = {Toygun Basaklar and
                  Suat Gumussoy and
                  {\"{U}}mit Y. Ogras},
  title        = {{PD-MORL:} {P}reference-Driven Multi-Objective Reinforcement Learning
                  Algorithm},
  booktitle    = {Eleventh International Conference on Learning Representations ({ICLR} 2023)},
  publisher    = {OpenReview.net},
  year         = {2023},
}

@inproceedings{liu2025pareto,
  author       = {Erlong Liu and
                  Yu{-}Chang Wu and
                  Xiaobin Huang and
                  Chengrui Gao and
                  Ren{-}Jian Wang and
                  Ke Xue and
                  Chao Qian},
  title        = {Pareto Set Learning for Multi-Objective Reinforcement Learning},
  booktitle    = {AAAI Conference on Artificial Intelligence},
  pages        = {18789--18797},
  publisher    = {{AAAI} Press},
  year         = {2025},
}

@article{mu2025preference,
  title={Preference-based Multi-Objective Reinforcement Learning},
  author={Mu, Ni and Luan, Yao and Jia, Qing-Shan},
  journal={IEEE Transactions on Automation Science and Engineering},
  year={2025},
  publisher={IEEE}
}

@inproceedings{memarian2021self,
  title={Self-supervised online reward shaping in sparse-reward environments},
  author={Memarian, Farzan and Goo, Wonjoon and Lioutikov, Rudolf and Niekum, Scott and Topcu, Ufuk},
  booktitle={2021 IEEE/RSJ International Conference on Intelligent Robots and Systems (IROS 2021)},
  pages={2369--2375},
  year={2021},
  organization={IEEE}
}

@inproceedings{devidze2022exploration,
  author       = {Rati Devidze and
                  Parameswaran Kamalaruban and
                  Adish Singla},

  title        = {Exploration-Guided Reward Shaping for Reinforcement Learning under
                  Sparse Rewards},
  booktitle    = {36th International Conference on Neural Information Processing Systems ({NIPS 2022})},
publisher = {Curran Associates},
  year         = {2022},
}

@inproceedings{he2015delving,
  author       = {Kaiming He and
                  Xiangyu Zhang and
                  Shaoqing Ren and
                  Jian Sun},
  title        = {Delving Deep into Rectifiers: Surpassing Human-Level Performance on
                  ImageNet Classification},
  booktitle    = {2015 {IEEE} International Conference on Computer Vision ({ICCV} 2015)},
  pages        = {1026--1034},
  publisher    = {{IEEE} Computer Society},
  year         = {2015},
}

@inproceedings{wang2022mathrm,
  author       = {Dian Wang and
                  Robin Walters and
                  Robert Platt},
  title        = {{SO(2)}-Equivariant
                  Reinforcement Learning},
  booktitle    = { Tenth International Conference on Learning Representations ({ICLR}
                  2022)},
  publisher    = {OpenReview.net},
  year         = {2022},
 
}

@inproceedings{park2024approximate,
  author       = {Jung Yeon Park and
                  Sujay Bhatt and
                  Sihan Zeng and
                  Lawson L. S. Wong and
                  Alec Koppel and
                  Sumitra Ganesh and
                  Robin Walters},
  title        = {Approximate Equivariance in Reinforcement Learning},
  booktitle    = {International Conference on Artificial Intelligence and Statistics ({AISTATS} 2025)},
  series       = {PMLR},
  volume       = {258},
  pages        = {4177--4185},
  publisher    = {{PMLR}},
  year         = {2025},
}

@article{lin2020invariant,
  title={Invariant transform experience replay: Data augmentation for deep reinforcement learning},
  author={Lin, Yijiong and Huang, Jiancong and Zimmer, Matthieu and Guan, Yisheng and Rojas, Juan and Weng, Paul},
  journal={IEEE Robotics and Automation Letters},
  volume={5},
  number={4},
  pages={6615--6622},
  year={2020},
  publisher={IEEE}
}

@inproceedings{mondal2022eqr,
  author       = {Arnab Kumar Mondal and
                  Vineet Jain and
                  Kaleem Siddiqi and
                  Siamak Ravanbakhsh},

  title        = {EqR: Equivariant Representations for Data-Efficient Reinforcement
                  Learning},
  booktitle    = {International Conference on Machine Learning ({ICML} 2022)},
  series       = {PMLR},
  volume       = {162},
  pages        = {15908--15926},
  publisher    = {{PMLR}},
  year         = {2022},
}

@article{holmes2025attention,
  title={Attention-Based Reward Shaping for Sparse and Delayed Rewards},
  author={Holmes, Ian and Chi, Min},
  journal={arXiv preprint arXiv:2505.10802},
  year={2025}
}

@inproceedings{ng1999policy,
  author       = {Andrew Y. Ng and
                  Daishi Harada and
                  Stuart Russell},

  title        = {Policy Invariance Under Reward Transformations: Theory and Application
                  to Reward Shaping},
  booktitle    = {Sixteenth International Conference on Machine Learning
                  {(ICML} 1999)},
  pages        = {278--287},
  publisher    = {Morgan Kaufmann},
  year         = {1999},
}

@inproceedings{NEURIPS2023_32285dd1,
 author = {Cai, Xin-Qiang and Zhang, Pushi and Zhao, Li and Bian, Jiang and Sugiyama, Masashi and Llorens, Ashley},
 booktitle = {37th International Conference on Neural Information Processing Systems (NIPS 2023)},
 pages = {15593--15613},
 publisher = {Curran Associates},
 title = {Distributional {P}areto-Optimal Multi-Objective Reinforcement Learning},
 volume = {36},
 year = {2023}
}

@book{laud2004theory,
  title={Theory and application of reward shaping in reinforcement learning},
  author={Laud, Adam Daniel},
  year={2004},
  publisher={University of Illinois at Urbana-Champaign}
}

@article{bartlett2002rademacher,
  title={Rademacher and {G}aussian complexities: Risk bounds and structural results},
  author={Bartlett, Peter L and Mendelson, Shahar},
  journal={Journal of Machine Learning Research},
  volume={3},
  pages={463--482},
  year={2002}
}

@article{dudley1967sizes,
  title={The sizes of compact subsets of {H}ilbert space and continuity of {G}aussian processes},
  author={Dudley, Richard M},
  journal={Journal of Functional Analysis},
  volume={1},
  number={3},
  pages={290--330},
  year={1967},
  publisher={Elsevier}
}

@article{mcdiarmid1989method,
  title={On the method of bounded differences},
  author={McDiarmid, Colin and others},
  journal={Surveys in Combinatorics},
  volume={141},
  number={1},
  pages={148--188},
  year={1989},
  publisher={Norwich}
}

@inproceedings{van2013scalarized,
  author       = {Kristof Van Moffaert and
                  Madalina M. Drugan and
                  Ann Now{\'{e}}},
  title        = {Scalarized multi-objective reinforcement learning: Novel design techniques},
  booktitle    = {2013 {IEEE} Symposium on Adaptive Dynamic Programming
                  and Reinforcement Learning ({ADPRL} 2013)},
  pages        = {191--199},
  publisher    = {{IEEE}},
  year         = {2013},
}

@article{qin2022benefits,
  title={Benefits of permutation-equivariance in auction mechanisms},
  author={Qin, Tian and He, Fengxiang and Shi, Dingfeng and Huang, Wenbing and Tao, Dacheng},
  journal={36th International Conference on Neural Information Processing Systems (NIPS 2022)},
  volume={35},
  pages={18131--18142},
publisher = {Curran Associates},
  year={2022}
}
\bibliographystyle{icml2026}

\newpage
\onecolumn
\appendix

\section{Notation}\label{app:notation}
\begin{table}[h!]
\centering
\footnotesize

\begin{threeparttable}
\caption{Notation.}
\label{tab:notation}
\begin{tabular}{l l }
\toprule
Symbol & Description \\
\midrule \midrule
$\mathcal{S}$ & State space \\
$\mathcal{A}$ & Action space \\
$P(s'|s,a)$ & Transition probability \\
$\bm{r}(s,a) \in \mathbb{R}^L$ & Vector-valued reward with $L$ objectives \\
$\gamma \in [0,1)$ & Discount factor \\
$\pi: \mathcal{S} \to \mathcal{A}$ & Policy mapping \\
$\bm{J}(\pi) = \mathbb{E}_\pi\!\big[\sum_{t=0}^\infty \gamma^t \bm{r}_t\big]$ & Expected cumulative vector return \\
\midrule
$\mathcal{D}$ & Behaviour distribution to sample episodes from \\
$DC = \{d_1,\dots,d_D\}$ & Dense reward channels \\
$r^{d_i}_t$ & Reward from dense channel $d_i$ at timestep $t$ \\
$r^{sp}_t$ & Sparse reward at timestep $t$ \\
$\tau = \{(s_1,a_1),\dots,(s_T,a_T)\}$ & Trajectory \\
$R^{sp}(\tau)$ & Cumulative sparse reward in episode $\tau $\\
$p_{\text{rel}}$ & Probability of releasing sparse reward \\
$h_t = [s_t,a_t,\bm{r}^{\text{dense}}_t]$ & Input feature vector for ReSymNet \\
$\mathcal{R}_{\text{pred}}$ & ReSymNet \\
$r^{sh}_t$ & Shaped reward at timestep $t$ \\
\midrule
$L_g, K_g$ & Reflection operators on states and actions \\
$\Delta_\pi(s) = \pi(L_g(s)) - K_g(\pi(s))$ & Equivariance mismatch \\
$\mathcal{L}_{eq}$ & Equivariance regularisation loss \\
$\Pi$ & Hypothesis space of policies \\
$\Pi_{eq} = \{\pi : \pi(L_g(s)) = K_g(\pi(s))\}$ & Reflection-equivariant subspace \\
$\Pi_{\text{approx}}(\varepsilon_{eq})$ & Approximate equivariant policies with tolerance $\varepsilon_{eq}$ \\
\bottomrule
\end{tabular}
\end{threeparttable}
\end{table}


\section{Additional Details and Theory of ReSymNet}\label{app:resymnet}

We give additional details of ReSymNet as well as the theoretical motivation behind its architecture in this appendix.

\subsection{Theoretical Motivation via Scaled Opportunity Value}

The use of residual connections in $\mathcal{R}_{\text{pred}}$ is motivated by the theory of scaled opportunity value \citep{laud2004theory}. 

\begin{definition}[Opportunity value]
Let $M$ be an MDP with native reward function $R$. 
The opportunity value of a transition $(s,a,s')$ is defined as the difference in the optimal value of successor and current states:
$\mathrm{OPV}(s,a,s') = \gamma V^M(s') - V^M(s),
$
where $V^M$ is the optimal state-value function under MDP $M$.
\end{definition}

\begin{definition}[Scaled opportunity value]
For a scale parameter $k > 0$, the scaled opportunity value shaping function augments the native reward with a scaled opportunity correction:
$\mathrm{OPV}_k(s,a,s') = F_k(s,a,s') = k (\gamma V^M(s') - V^M(s)) + (k-1) R(s,a).$
\end{definition}

\begin{lemma}
Let $M$ be an MDP with reward function $R$ and optimal policy $\pi^\star$. 
With $k$ sufficiently large, the MDP with shaped reward $F_k$ satisfies: (1) policy invariance, $\pi^\star$ remains optimal under $F_k$; (2) horizon reduction, the effective reward horizon is reduced to $1$; and (3) improved local approximation, the additive term increases the separability of local utilities, reducing approximation error in value estimation.
\end{lemma}

Residual blocks mirror the additive structure of scaled opportunity value: each block refines its input prediction via:
$\mathcal{R}_{\text{pred}}^{(i)}(\bm{h}_t; \psi) = \mathcal{R}_{\text{pred}}^{(i-1)}(\bm{h}_t; \psi) + \Delta_i(\bm{h}_t;\psi)$,
where $\Delta_i$ is a learned correction. 
A single block can be viewed as approximating a scaled opportunity-value transformation of its input, 
while stacking multiple blocks
implements iterative refinement: each stage reduces the residual error left by the previous one. 
This residual formulation both stabilises training and aligns with the principle of scaled opportunity value, 
gradually shaping per-step predictions into horizon-1 signals that remain consistent with the sparse episodic return $R^{sp}(\tau)$.

\subsection{Generalisability of ReSymNet}

We extend the theoretical justification of ReSymNet from optimisation to generalisation. Following the stem--vine decomposition of \citet{he2020resnet}, we prove that residual connections do not increase hypothesis complexity, and derive a high-probability bound.

\textbf{Notation and Assumptions.}
ReSymNet maps feature vectors $\mathbf h_t\in\mathbb R^{d_0}$ to sparse reward predictions $ r^{\mathrm{sp}}_t \in \mathbb R$ through a residual network. We decompose the network into:
\begin{itemize}
    \item A stem: the main feedforward pathway consisting of $K$ layers, each with a weight matrix $\mathbf A_i \in \mathbb{R}^{d_{i-1} \times d_i}$ and nonlinearity $\sigma_i: \mathbb{R}^{d_i} \to \mathbb{R}^{d_i}$ for $i = 1, \ldots, K$.
    \item A collection of vines: residual connections (skip connections) indexed by triples $(s,t,i)$ where $s$ is the source vertex (where the connection starts), $t$ is the target vertex (where it reconnects), and $i$ distinguishes multiple vines between the same pair of vertices. We denote the set of all vine indices as $\mathcal{I}_V$.
\end{itemize}

We denote vertices in the network as $N(t)$, where $t$ indexes the position in the computational graph. Each vine $\mathcal{V}(s,t,i)$ is itself a small feedforward network with weight matrices 
$\mathbf{A}^{s,t,i}_1, \ldots, \mathbf{A}^{s,t,i}_{K_{s,t,i}}$ and nonlinearities 
$\sigma^{s,t,i}_1, \ldots, \sigma^{s,t,i}_{K_{s,t,i}}$, where $K_{s,t,i}$ is the number of layers in that vine. 
The output at vertex $N(t)$ is:
\[
F_t(\mathbf X) = F^S_t(\mathbf X) + \sum_{(s,t,i) \in \mathcal{I}_V} F^V_{s,t,i}(\mathbf X),
\]
where $F^S_t(\mathbf X)$ is the stem's output at vertex $t$ and the sum runs over all vines that reconnect at vertex $t$.

\begin{assumption}[Bounded parameters]
\label{ass:bounded_params}
Each stem weight matrix satisfies $\|\mathbf A_i\|_\sigma \le s_i$ for $i=1,\ldots,K$, where $\|\cdot\|_\sigma$ denotes the spectral norm.
Each vine weight matrix satisfies $\|\mathbf A^{s,t,i}_j\|_\sigma \le s^{s,t,i}_j$.
All nonlinearities are $\rho_i$-Lipschitz continuous: for any $\mathbf x_1, \mathbf x_2$ in the domain,
\[
\|\sigma_i(\mathbf x_1) - \sigma_i(\mathbf x_2)\|_2 \le \rho_i\|\mathbf x_1 - \mathbf x_2\|_2.
\]
Input features satisfy $\|\mathbf h_t\|_2 \le B_h$, network per-step outputs satisfy $|R_{\text{pred}}(\mathbf{h}_t; \psi)| \leq B_{\text{pred}}$ for all $t$, and sparse rewards satisfy $|R^{\mathrm{sp}}(\tau)| \le B_r$ for all trajectories $\tau$. Trajectories have length bounded by $T_{\max}$.
\end{assumption}

\begin{lemma}
\label{lem:single_layer}
Let $\mathbf X\in\mathbb R^{n\times d}$ be a data matrix with $n$ samples and $d$ features, satisfying $\|\mathbf X\|_2\le B$.
Consider the hypothesis space formed by all linear transformations with bounded spectral norm:
\[
\mathcal H_A=\{\mathbf X\mathbf A:\mathbf{A} \in \mathbb{R}^{d \times m},\ \|\mathbf A\|_\sigma\le s\}.
\]
Then the $\varepsilon$-covering number satisfies:
\[
\log\mathcal{N}_{\infty,2}(\mathcal H_A,\varepsilon)
\le
\left\lceil \frac{s^2 B^2 m^2}{\varepsilon^2}\right\rceil
\log(2dm),
\]
where $m$ is the output dimension.
\end{lemma}

This lemma \citep{bartlett2017spectrally} 
shows that the complexity of a single linear layer scales with the square of its spectral norm and input norm.

\begin{lemma}
\label{lem:feedforward}
For an $K$-layer feedforward network with hypothesis space $\mathcal H_{\mathrm{ff}}$, the covering number satisfies:
\[
\mathcal{N}_{\infty,2}(\mathcal H_{\mathrm{ff}},\varepsilon)
\le
\prod_{i=1}^K
\sup_{\mathbf A_1,\ldots,\mathbf A_{i-1}}
\mathcal N_i,
\]
where $\mathcal N_i$ is the covering number of layer $i$ (viewed as a function of its input) when the preceding layers $\mathbf{A}_1, \ldots, \mathbf{A}_{i-1}$ are held fixed. The supremum is taken over all choices of the preceding weight matrices within their respective spectral norm bounds.
\end{lemma}

This result shows that the covering number of a deep network is the product of the covering numbers of its individual layers. For residual networks, where outputs are sums of stem and vine contributions, we require:

\begin{lemma}
\label{lem:sum_cover}
Let $\mathcal F$ and $\mathcal G$ be two function classes. If $\mathcal W_F$ is an $\varepsilon_F$-cover of $\mathcal F$ (meaning every $f \in \mathcal{F}$ is within distance $\varepsilon_F$ of some element in $\mathcal{W}_F$), and $\mathcal W_G$ is an $\varepsilon_G$-cover of $\mathcal G$, then the set
\[
\mathcal W_F + \mathcal W_G = \{f+g: f\in\mathcal W_F,\ g\in\mathcal W_G\}
\]
is an $(\varepsilon_F+\varepsilon_G)$-cover of the sum class $\mathcal F+\mathcal G = \{f+g : f \in \mathcal{F}, g \in \mathcal{G}\}$, and
\[
\mathcal{N}_{\infty,2}(\mathcal F+\mathcal G,\varepsilon_F+\varepsilon_G)
\le
\mathcal{N}_{\infty,2}(\mathcal F,\varepsilon_F)\,
\mathcal{N}_{\infty,2}(\mathcal G,\varepsilon_G).
\]
\end{lemma}

\begin{proof}
For any $f+g \in \mathcal F + \mathcal G$, there exist $w_f \in \mathcal W_F$ and $w_g \in \mathcal W_G$ 
such that $\|f - w_f\|_2 \le \varepsilon_F$ and $\|g - w_g\|_2 \le \varepsilon_G$. By the triangle inequality:
\[
\|(f+g) - (w_f + w_g)\|_2 
\le \|f - w_f\|_2 + \|g - w_g\|_2 
\le \varepsilon_F + \varepsilon_G.
\]
The covering number bound follows since there are at most $|\mathcal W_F| \cdot |\mathcal W_G|$ distinct pairs $(w_f, w_g)$.
\end{proof}

\begin{theorem}
\label{thm:resymnet_cover}

Under Assumption~\ref{ass:bounded_params}, let $\{\varepsilon_j\}_{j=1}^K$ be tolerances for each stem layer and $\{\varepsilon_{s,t,i}\}_{(s,t,i)\in\mathcal I_V}$ be tolerances for each vine, satisfying
\[
\sum_{j=1}^K \varepsilon_j + \sum_{(s,t,i)\in\mathcal I_V}\varepsilon_{s,t,i}
\le \varepsilon.
\]
Then the covering number of ReSymNet's hypothesis space $\mathcal H_{\mathrm{res}}$ satisfies:
\[
\mathcal{N}_{\infty,2}(\mathcal H_{\mathrm{res}},\varepsilon)
\le
\prod_{j=1}^K\mathcal{N}_{\infty,2}(\mathcal H_j,\varepsilon_j)\,
\prod_{(s,t,i)\in\mathcal I_V}
\mathcal{N}_{\infty,2}(\mathcal H^V_{s,t,i},\varepsilon_{s,t,i}),
\]
where $\mathcal{H}_j$ is the hypothesis space of stem layer $j$ and $\mathcal{H}^V_{s,t,i}$ is the hypothesis space of vine $\mathcal{V}(s,t,i)$.

Applying Lemma~\ref{lem:single_layer} to each weight matrix, this yields:
\[
\log\mathcal{N}_{\infty,2}(\mathcal H_{\mathrm{res}},\varepsilon)
\le
\frac{\mathcal R}{\varepsilon^2},
\]
where the complexity measure $\mathcal{R}$ is:
\[
\mathcal R = \sum_{i=1}^K \frac{s_i^2\|F_{i-1}(\mathbf X)\|_2^2}{\varepsilon_i^2}\log(2d_i^2) 
+ \sum_{(s,t,i)\in\mathcal I_V} \frac{(s^{s,t,i})^2\|F_s(\mathbf X)\|_2^2}{\varepsilon_{s,t,i}^2}\log(2d_{s,t,i}^2).
\]
Here, $F_{i-1}(\mathbf{X})$ denotes the output of the network after layer $i-1$ (the input to layer $i$), and $d_i$ is the dimension at layer $i$.

\end{theorem}

\begin{proof}

We proceed by analysing how residual connections compose with the stem. Consider vertex $N(t)$ where one or more vines reconnect. The output is:
\[
F_t(\mathbf{X}) = F^S_t(\mathbf{X}) + \sum_{(s,t,i) \in \mathcal{I}_V} F^V_{s,t,i}(\mathbf{X}).
\]

Let $\mathcal W_t$ be an $\varepsilon_t$-cover of $\mathcal H_t$ (all possible stem outputs at vertex $t$). For each vine $\mathcal{V}(s,t,i)$ that reconnects at $t$, let $\mathcal W^V_{s,t,i}$ be an $\varepsilon_{s,t,i}$-cover of $\mathcal H^V_{s,t,i}$ (all possible outputs of that vine).

By repeated application of Lemma~\ref{lem:sum_cover}, the set:
\[
\mathcal{W}'_t = \left\{W_S + \sum_{(s,t,i) \in \mathcal{I}_V} W^V_{s,t,i} : W_S \in \mathcal{W}_t, W^V_{s,t,i} \in \mathcal{W}^V_{s,t,i}\right\}
\]
is an $\left(\varepsilon_t + \sum_{(s,t,i) \in \mathcal{I}_V} \varepsilon_{s,t,i}\right)$-cover of $\mathcal{H}'_t$ (the combined outputs at vertex $t$), with covering number:
\[
\mathcal{N}_{\infty,2}(\mathcal H'_t, \varepsilon'_t)
\le \mathcal{N}_{\infty,2}(\mathcal H_t, \varepsilon_t) \cdot \prod_{(s,t,i) \in \mathcal{I}_V} \mathcal{N}_{\infty,2}(\mathcal H^V_{s,t,i}, \varepsilon_{s,t,i}),
\]
where $\varepsilon'_t = \varepsilon_t + \sum_{(s,t,i) \in \mathcal{I}_V} \varepsilon_{s,t,i}$.

Each vine $\mathcal{V}(s,t,i)$ is itself a chain-like feedforward network, so Lemma~\ref{lem:feedforward} applies to bound $\mathcal{N}_{\infty,2}(\mathcal H^V_{s,t,i}, \varepsilon_{s,t,i})$.
For identity vines (containing no trainable parameters), we have $\mathcal N^V_{s,t,i} = 1$ since there is only one function in the class.

Propagating this argument through all $K$ stem layers yields:
\[
\mathcal{N}_{\infty,2}(\mathcal H_{\mathrm{res}},\varepsilon)
\le
\prod_{j=1}^K\mathcal{N}_{\infty,2}(\mathcal H_j,\varepsilon_j)\,
\prod_{(s,t,i)\in\mathcal I_V}
\mathcal{N}_{\infty,2}(\mathcal H^V_{s,t,i},\varepsilon_{s,t,i}).
\]

The bound on $\mathcal R$ follows by applying Lemma~\ref{lem:single_layer} to each weight matrix. For the stem, layer $i$ contributes:
\[
\log \mathcal{N}_i \le \frac{s_i^2 \|F_{i-1}(\mathbf{X})\|_2^2 d_i^2}{\varepsilon_i^2} \log(2d_{i-1}d_i) \approx \frac{s_i^2 \|F_{i-1}(\mathbf{X})\|_2^2}{\varepsilon_i^2}\log(2d_i^2),
\]
where we simplify by assuming similar dimensions. Summing over all stem layers and all vine layers gives $\mathcal{R}$.
\end{proof}

\begin{corollary}
\label{cor:residual_equiv}
Let $\mathcal H_{\mathrm{ff}}$ be the hypothesis space of feedforward networks with the same total number
of weight matrices $K_{\text{total}} = K + \sum_{(s,t,i)\in\mathcal{I}_V} K_{s,t,i}$ as ReSymNet. Then for any $\varepsilon>0$,
\[
\mathcal{N}_{\infty,2}(\mathcal H_{\mathrm{res}},\varepsilon)
\le
\mathcal{N}_{\infty,2}(\mathcal H_{\mathrm{ff}},\varepsilon).
\]
\end{corollary}

\begin{proof}
Both covering numbers have the product form $\prod_{k=1}^{K_{\text{total}}} \mathcal N_k$, where each factor $\mathcal N_k$ 
corresponds to a single weight matrix. By Lemma~\ref{lem:single_layer}, each $\mathcal N_k$ depends only on 
the spectral norm $s_k$ of that weight matrix and the norm of its input $\|F_{k-1}(\mathbf{X})\|_2$, 
regardless of whether the matrix appears in the stem or a vine. Therefore, when the total number of weight matrices and their norms are held fixed, the covering numbers are bounded identically.
\end{proof}

\subsection{Algorithm Chart}

\begin{algorithm}[h!]
\caption{ReSymNet with any MORL algorithm}
\label{algo:reward_shaping}
\begin{algorithmic}[1] 
   \STATE {\bfseries Input:} Release probability $p_{\text{rel}}$, number of initial episodes $N$, number of expert episodes $E$, dense channels $\mathcal{DC}$, MORL algorithm, timesteps per cycle $M$, ensembles $K$, refinements $IR$, val split, patience
   \STATE {\bfseries Output:} Trained reward ensemble $\mathcal{E} = \{\mathcal{R}_{\text{pred},\psi_1}, \dots, \mathcal{R}_{\text{pred},\psi_K}\}$, trained MORL policy

   \STATE
   \STATE \textit{\# Collecting random experiences}
   \FOR{$i=1$ {\bfseries to} $N$}
       \STATE Execute random policy to collect $\tau = \{(s_0, a_0), \dots, (s_T, a_T)\}$
       \STATE Set $l=0$
       \FOR{$t \in T$}
           \STATE With prob. $p_{\text{rel}}$, release $R_t^{sp} = \sum_{s=l}^{t} r_s^{sp}$; Set $l=t$ if released
       \ENDFOR
       \STATE Segment $\tau$ into sub-trajectories $\{\tau_j\}$ based on released rewards
       \FORALL{sub-trajectory $\tau_j$}
           \FORALL{$(s_t, a_t) \in \tau_j$}
               \STATE Compute features: $\bm{h}_t = [s_t, a_t, \bm{r}^{\text{dense}}_{t}]$
           \ENDFOR
           \STATE Add datapoint $\left(\{\bm{h}_t\}_{t \in \tau_j}, R^{sp}(\tau_j)\right)$ to dataset $\mathcal{D}$
       \ENDFOR
   \ENDFOR

   \STATE
   \STATE \textit{\# Ensemble training}
   \FOR{$k=1$ {\bfseries to} $K$}
       \STATE Split $\mathcal{D}$ into $\mathcal{D}_{\text{train}}$ and $\mathcal{D}_{\text{val}}$
       \STATE Train $\mathcal{R}_{\text{pred},\psi_k}$ via Eq. \ref{eq:nn_loss} with early stopping:
       \STATE \quad $\mathcal{L}(\psi_k) = \sum_{\tau \in \mathcal{D}_{\text{train}}} \left( \sum_{t \in \tau} \mathcal{R}_{\text{pred}}(\bm{h}_t; {\psi_k}) - R^{sp}(\tau) \right)^2$
   \ENDFOR

   \STATE
   \STATE \textit{\# RL training with iterative refinement}
   \STATE $timestep = 1$
   \FOR{$cycle=1$ {\bfseries to} $IR$}
       \FOR{$t = timestep$ {\bfseries to} $M + timestep$}
           \STATE Observe $s_t, a_t$ and compute features $\bm{h}_t$
           \STATE $r_t^{(k)} \gets \mathcal{R}_{\text{pred}}(\bm{h}_t; {\psi_k})$ for $k = 1, \dots, K$
           \STATE $r_t^{\text{sh}} \gets \frac{1}{K} \sum_{k=1}^K r_t^{(k)}$
           \STATE Update RL algorithm using $r_t^{\text{sh}}$ and dense rewards
       \ENDFOR
       
       \STATE \textit{\# Iterative refinement}
       \STATE Collect $E$ expert trajectories $\mathcal{D}_{\text{new}}$ using new policy
       \FORALL{$\mathcal{R}_{\text{pred},\psi_k} \in \mathcal{E}$}
           \STATE Update $\mathcal{R}_{\text{pred},\psi_k}$ using new data $\mathcal{D}_{\text{new}}$
       \ENDFOR
       \STATE $timestep = t$
   \ENDFOR
\end{algorithmic}
\end{algorithm}

\section{Proofs}\label{app:proofs}

This appendix collects all proofs omitted from the main text.

\subsection{Lemmas}\label{app:lemmas}

This section introduces the general lemmas used to obtain an upper bound on the generalisation gap. 

\textbf{Dudley Entropy Integral.} 
The Rademacher complexity can be bounded through the metric entropy of the function class using Dudley's entropy integral \citep{dudley1967sizes, bartlett2002rademacher}. 

\begin{lemma}[Dudley Entropy Integral]\label{lemma:dudley}
For any coarse-scale parameter $r\in(0,B)$, the empirical Rademacher complexity satisfies:
\begin{equation*}
    \hat {\mathfrak{R}}_N(\mathcal{F}) \le C \left( \int_{r}^{B} \sqrt{\frac{\log \mathcal{N}_{\infty,1}(\mathcal{F}, r)}{N}}  d\varepsilon \right) + \frac{4 r}{\sqrt{N}},
\end{equation*}
where $C>0$ is an absolute constant, and $\mathcal{N}_{\infty,1}(\mathcal{F}, r)$ is the covering number of $\mathcal{F}$ in $\ell_\infty$ at scale $r$ with respect to $N$ samples
\end{lemma}
This inequality connects the probabilistic complexity (Rademacher complexity) to the geometric complexity of the function class and covering numbers.

\textbf{McDiarmid's Concentration Inequality.}
To convert expectation bounds into high-probability statements, we employ McDiarmid's bounded difference inequality \citep{mcdiarmid1989method}. 

\begin{lemma}[McDiarmid's Concentration Inequality]\label{lemma:mc}
If each trajectory's replacement can change any empirical average by at most $B/N$, then for any $t>0$:
\begin{equation*}
    \Pr\left( \left|\sup_{f\in\mathcal{F}} \frac{1}{N}\sum_{i=1}^N (f(\tau_i)-\mathbb{E}[f]) - \mathbb{E}\left[\sup_{f\in\mathcal{F}} \frac{1}{N}\sum_{i=1}^N (f(\tau_i)-\mathbb{E}[f])\right]\right| \geq t \right) \leq 2\exp\left(-\frac{2Nt^2}{B^2}\right).
\end{equation*}

\end{lemma}
This concentration result allows us to bound the deviation between the random supremum and its expectation, completing the pipeline from covering numbers to high-probability uniform generalisation gaps.

\subsection{Generalisation of Scalarised Returns}\label{app:scalar}

This section shows that generalisation for an arbitrary scalar return implies guarantees for the scalarised components of the Pareto front.

\begin{corollary}
\label{cor:scalar_gen}
Let $\Pi$ be a policy class equipped with a metric $d(\cdot,\cdot)$, and let
$\bm{R}(\pi;\tau)\in\mathbb{R}^L$ denote the vector-valued return of policy $\pi$ on trajectory $\tau$.
Following Assumption \ref{ass:bounded_returns}:
\[
\sup_{\tau}\|\bm{R}(\pi;\tau)-\bm{R}(\tilde\pi;\tau)\|_\infty \le L_R\,d(\pi,\tilde\pi)
\qquad\text{for all }\pi,\tilde\pi\in\Pi.
\]
For any weight vector $\omega\in\mathbb R^L$ define the scalarised return
$R^\omega(\pi;\tau)=\omega^\top \bm{R}(\pi;\tau)$ and let $\mathcal R_{\Pi}^{\omega}$ be the class of
scalarised returns induced by $\Pi$.
Then for any $\varepsilon>0$,
\[
\mathcal{N}_{\infty,1}(\mathcal{R}_{\Pi}^{\omega}, \varepsilon)
\le
\mathcal{N}_{\infty,1}\!\big(\Pi,\; \varepsilon / L_R^{\omega}\big),
\qquad\text{where }L_R^{\omega} := \|\omega\|_1\,L_R.
\]
In particular, when $\|\omega\|_1 = 1$ we have $L_R^{\omega}=L_R$ and the scalarised return class
has covering numbers no larger than those of the policy class.
Consequently, any complexity reduction obtained by projecting $\Pi$ to an equivariant subspace
(e.g.\ $\Pi_{eq}$) is inherited by the scalarised objective class $\mathcal R_{\Pi}^{\omega}$.
\end{corollary}

\begin{proof}
Fix $\omega\in\mathbb R^L$ and let $\pi,\tilde\pi\in\Pi$. For any trajectory $\tau$,
\[
\big|R^{\omega}(\pi;\tau)-R^{\omega}(\tilde\pi;\tau)\big|
= \big|\omega^\top\big(\bm{R}(\pi;\tau)-\bm{R}(\tilde\pi;\tau)\big)\big|
\le \sum_{j=1}^L |\omega_j|\,\big|R_j(\pi;\tau)-R_j(\tilde\pi;\tau)\big|.
\]
Using \(\max_j |R_j(\pi;\tau)-R_j(\tilde\pi;\tau)|=\|\bm{R}(\pi;\tau)-\bm{R}(\tilde\pi;\tau)\|_\infty,\) we obtain
\[
\big|R^{\omega}(\pi;\tau)-R^{\omega}(\tilde\pi;\tau)\big|
\le \|\omega\|_1 \,\|\bm{R}(\pi;\tau)-\bm{R}(\tilde\pi;\tau)\|_\infty.
\]
Taking the supremum over trajectories and applying the vector Lipschitz assumption yields
\[
\sup_{\tau}\big|R^{\omega}(\pi;\tau)-R^{\omega}(\tilde\pi;\tau)\big|
\le \|\omega\|_1\,L_R\,d(\pi,\tilde\pi) = L_R^{\omega}\,d(\pi,\tilde\pi).
\]
Thus the scalarised return map $\pi\mapsto R^{\omega}(\pi;\cdot)$ is Lipschitz with constant
$L_R^{\omega}=\|\omega\|_1 L_R$.

Following Lemma \ref{lem:app-retcov_full}, for any $\varepsilon>0$,
\[
\mathcal{N}_{\infty,1}(\mathcal R_{\Pi}^{\omega},\varepsilon)
\le \mathcal{N}_{\infty,1}\!\big(\Pi,\; \varepsilon / L_R^{\omega}\big).
\]
This proves the displayed inequality. The special case \(\|\omega\|_1=1\) follows immediately.
Finally, since the inequality holds for any policy class \(\Pi\), replacing \(\Pi\) by the
equivariant subspace \(\Pi_{eq}\) shows that any complexity reduction (\(\mathcal N(\Pi_{eq},\cdot)\) is directly inherited by the scalarised return class.
\end{proof}

\subsection{Projection to Reflection-Equivariant Subspace}\label{app:Q}

Let the full hypothesis space of policies be $\Pi = \{\pi_{\phi} : \phi \in \Phi\}$, where $\phi$ represents the neural network parameters and $\Phi$ represents the parameter space. The reflection group $G = \mathbb{Z}_2 = \{e, g\}$ acts on the state and action spaces via operators $L_g$ and $K_g$, respectively.

We can map any policy to its equivariant counterpart using an orbit averaging operator $\mathcal{Q}: \Pi \to \Pi$, defined as:
\begin{align}
    \mathcal{Q}(\pi_{\phi})(s) 
    &= \frac{1}{|G|}\sum_{h \in G} \rho(h)\pi_{\phi}(h^{-1}\cdot s) \nonumber \\
    &= \frac{1}{|G|}\sum_{h \in G} K_h\big(\pi_{\phi}(L_h(s))\big) \nonumber \\
    &= \tfrac{1}{2}\left(\pi_{\phi}(s) + K_g(\pi_{\phi}(L_g(s)))\right).
\end{align}
Here, $\rho(h)$ is the abstract representation in the action space, and $h^{-1}\cdot s$ is the abstract action in the state space. In the second line we replace $\rho(h)$ with the action transformation $K_h$, and $h^{-1}\cdot s$ with the state transformation $L_h(s)$. For the reflection group $G=\mathbb{Z}_2 = \{e,g\}$, since $g=g^{-1}$ we may drop the inverse without ambiguity.
This operator averages a policy's output with its reflected-transformed equivalent. The regulariser, $\mathcal{L}_{\text{eq}} = \mathbb{E}_{s} [ \| \pi_{\phi}(L_g(s)) - K_g(\pi_{\phi}(s)) \|^2_1 ]$, encourages policies to become fixed points of this operator, thereby learning policies within the subspace of equivariant functions, denoted $\Pi_{\text{eq}}$.

The operator $\mathcal{Q}$ and the subspace $\Pi_{\text{eq}}$ have several crucial properties, which we state in the following lemmas.

\begin{lemma}\label{lemma:app-Q-eq}
For any $\pi \in \Pi$, the function $\mathcal{Q}(\pi)$ is reflectional equivariant: 
\begin{equation*}
\mathcal{Q}(\pi)(L_g(s)) = K_g(\mathcal{Q}(\pi)(s)), \quad \forall s \in \mathcal{S}.
\end{equation*}
\end{lemma}

\begin{proof}
By direct calculation:
\begin{align*}
\mathcal{Q}(\pi)(L_g(s)) 
&= \tfrac{1}{2} \big( \pi(L_g(s)) + K_g(\pi(L_g(L_g(s)))) \big) \nonumber \\
&= \tfrac{1}{2} \big( \pi(L_g(s)) + K_g(\pi(s)) \big), \\
K_g\mathcal{Q}(\pi)(s) 
&= \tfrac{1}{2} \big( K_g(\pi(s)) + K_g K_g(\pi(L_g(s))) \big) \nonumber \\
&= \tfrac{1}{2} \big( K_g(\pi(s)) + \pi(L_g(s)) \big),
\end{align*}
since $K_g$ and $L_g$ are involutions. Thus, the two expressions coincide. Therefore $\mathcal{Q}(\pi)$ is equivariant.
\end{proof}

\begin{lemma}\label{lemma:app-q-idem}
The operator $\mathcal{Q}$ is a projection, meaning it is idempotent: $\mathcal{Q}(\mathcal{Q}(\pi)) = \mathcal{Q}(\pi)$ for any $\pi \in \Pi$.
\end{lemma}

\begin{proof}
We apply the operator to its own output:
\begin{align*}
    \mathcal{Q}(\mathcal{Q}(\pi))(s) &= \frac{1}{2} \left( \mathcal{Q}(\pi)(s) + K_g(\mathcal{Q}(\pi)(L_g(s))) \right).
\end{align*}
First, evaluating the second term, $\mathcal{Q}(\pi)(L_g(s))$:
\begin{align*}
    \mathcal{Q}(\pi)(L_g(s)) &= \frac{1}{2} \left( \pi(L_g(s)) + K_g(\pi(L_g(L_g(s)))) \right) \nonumber \\
    &= \frac{1}{2} \left( \pi(L_g(s)) + K_g(\pi(s)) \right).
\end{align*}
Substituting this back:
\begin{align*}
    \mathcal{Q}(\mathcal{Q}(\pi))(s) 
    &= \frac{1}{2} \left( \mathcal{Q}(\pi)(s) + K_g\!\left[ \tfrac{1}{2} (\pi(L_g(s)) + K_g(\pi(s))) \right] \right) \nonumber \\
    &= \frac{1}{2} \mathcal{Q}(\pi)(s) + \frac{1}{4} \left( K_g(\pi(L_g(s))) + K_g(K_g(\pi(s))) \right) \nonumber \\
    &= \frac{1}{2} \mathcal{Q}(\pi)(s) + \frac{1}{4} \left( K_g(\pi(L_g(s))) + \pi(s) \right) \nonumber \\
    &= \frac{1}{2} \mathcal{Q}(\pi)(s) + \frac{1}{2} \left( \tfrac{1}{2} (\pi(s) + K_g(\pi(L_g(s)))) \right) \nonumber \\
    &= \frac{1}{2} \mathcal{Q}(\pi)(s) + \frac{1}{2} \mathcal{Q}(\pi)(s) \nonumber \\
    &= \mathcal{Q}(\pi)(s).
\end{align*}
Thus $\mathcal{Q}$ is idempotent. 
\end{proof}

\begin{lemma}\label{lemma:app-q-image}
The image of the operator $\mathcal{Q}$ coincides with the set of equivariant policies: $\mathrm{Im}(\mathcal{Q}) = \{\mathcal{Q}(\pi): \pi \in \Pi\} = \Pi_{\text{eq}}$.
\end{lemma}
\begin{proof}
We establish set equality by showing inclusion in both directions.

First inclusion ($\mathrm{Im}(\mathcal{Q}) \subseteq \Pi_{\text{eq}}$): 
By Lemma \ref{lemma:app-Q-eq}, for any $\pi \in \Pi$, the output $\mathcal{Q}(\pi)$ is equivariant. Therefore, every element in the image of $\mathcal{Q}$ belongs to $\Pi_{\text{eq}}$.

Second inclusion ($\Pi_{\text{eq}} \subseteq \mathrm{Im}(\mathcal{Q})$):
Let $\pi_{\text{eq}}$ be any equivariant policy, so $\pi_{\text{eq}} \in \Pi_{\text{eq}}$. We need to show that $\pi_{\text{eq}}$ can be expressed as $\mathcal{Q}(\pi)$ for some $\pi \in \Pi$. 

Since $\pi_{\text{eq}}$ is equivariant, it satisfies $\pi_{\text{eq}}(L_g(s)) = K_g(\pi_{\text{eq}}(s))$ for all $s$. Therefore:
\begin{align*}
\mathcal{Q}(\pi_{\text{eq}})(s) &= \frac{1}{2} \left( \pi_{\text{eq}}(s) + K_g(\pi_{\text{eq}}(L_g(s))) \right) \nonumber \\
&= \frac{1}{2} \left( \pi_{\text{eq}}(s) + K_g(K_g(\pi_{\text{eq}}(s))) \right) \quad \text{(by equivariance)} \nonumber \\
&= \frac{1}{2} \left( \pi_{\text{eq}}(s) + \pi_{\text{eq}}(s) \right) \quad \text{(since $K_g$ is an involution)} \\
&= \pi_{\text{eq}}(s).
\end{align*}
Therefore, $\pi_{\text{eq}} = \mathcal{Q}(\pi_{\text{eq}}) \in \mathrm{Im}(\mathcal{Q})$. This shows that equivariant policies are fixed points of $\mathcal{Q}$, which is consistent with Lemma \ref{lemma:app-q-idem}. Since every equivariant policy is its own image under $\mathcal{Q}$, we have $\Pi_{\text{eq}} \subseteq \mathrm{Im}(\mathcal{Q})$. Combining both inclusions yields $\mathrm{Im}(\mathcal{Q}) = \Pi_{\text{eq}}$. Therefore $\mathcal{Q}$ is surjective onto $\Pi_{\text{eq}}$.
\end{proof}

\subsection{Reduced Hypothesis Complexity of Reflection-Equivariant Subspace}\label{app:covering}

To prove that the subspace $\Pi_{\text{eq}}$ is less complex, we show that the projection $\mathcal{Q}$ is non-expansive, which implies its image has a covering number no larger than the original space.

\begin{theorem}\label{theorem:app-covering}
The space $\Pi_{\text{eq}}$ has a covering number less than or equal to that of $\Pi$. Let $\mathcal{N}_{\infty,1}(\mathcal{F}, r)$ be the covering number of a function space $\mathcal{F}$ under the $l_{\infty,1}$-distance. Then, $\mathcal{N}_{\infty,1}(\Pi_{\text{eq}}, r) \le \mathcal{N}_{\infty,1}(\Pi, r)$.
\end{theorem}

\begin{proof}
We show that $\mathcal{Q}$ is non-expansive. The $l_{\infty,1}$-distance between two policies $\pi_{\phi}$ and $\pi_{\theta}$ is
\[
d(\pi_{\phi}, \pi_{\theta}) = \sup_{s} \|\pi_{\phi}(s) - \pi_{\theta}(s)\|_1.
\]
The distance between their projections is:
\begin{align*}
    d(\mathcal{Q}(\pi_{\phi}), \mathcal{Q}(\pi_{\theta})) 
    &= \sup_{s} \left\| \tfrac{1}{2} \big(\pi_{\phi}(s) + K_g(\pi_{\phi}(L_g(s)))\big) - \tfrac{1}{2}\big(\pi_{\theta}(s) + K_g(\pi_{\theta}(L_g(s)))\big) \right\|_1 \nonumber \\
    &= \tfrac{1}{2} \sup_{s} \left\| (\pi_{\phi}(s) - \pi_{\theta}(s)) + K_g(\pi_{\phi}(L_g(s)) - \pi_{\theta}(L_g(s))) \right\|_1. \nonumber \\
    &\le \tfrac{1}{2} \sup_{s} \Big( \|\pi_{\phi}(s) - \pi_{\theta}(s)\|_1 + \|K_g(\pi_{\phi}(L_g(s)) - \pi_{\theta}(L_g(s)))\|_1 \Big). \nonumber \\
    &\le \tfrac{1}{2} \Big( \sup_{s} \|\pi_{\phi}(s) - \pi_{\theta}(s)\|_1 + \sup_{s} \|\pi_{\phi}(L_g(s)) - \pi_{\theta}(L_g(s))\|_1 \Big). \nonumber \\
&= \tfrac{1}{2}\big( d(\pi_{\phi}, \pi_{\theta}) + d(\pi_{\phi}, \pi_{\theta}) \big) = d(\pi_{\phi}, \pi_{\theta}),
\end{align*}
where we use the triangle inequality, the fact that $K_g$ is a norm-preserving isometry, $\|K_g(a)\|_1 = \|a\|_1$, and that $L_g$ is a bijection, which implies that the supremum over $s$ equals the supremum over $L_g(s)$. Hence $\mathcal{Q}$ is non-expansive, and a non-expansive surjective map cannot increase the covering number. Following Lemma \ref{lemma:app-q-image}, $\mathcal{N}(\Pi_{\text{eq}}, r) \le \mathcal{N}(\Pi, r)$.
\end{proof}

The following lemma links coverings of the policy class (with metric $d$) to coverings of the induced return class (supremum over trajectories). This is the deterministic Lipschitz step that makes the entropy of returns comparable to the entropy of the policy class.

\begin{lemma}\label{lem:app-retcov_full}
 For any policy set $\mathcal P\subseteq\Pi$ and any $\varepsilon>0$,
\begin{equation*}
  \mathcal N_{\infty,1}\big(\{ \tau\mapsto R(\pi;\tau) : \pi\in\mathcal P\},\varepsilon\big)
  \le \mathcal N_{\infty,1}\big(\mathcal P,\varepsilon / L_R\big),
\end{equation*}
where the left covering number is with respect to the sup-norm over trajectories and the right is with respect to $d(\cdot,\cdot)$.
\end{lemma}

\begin{proof}
Let $\{\pi_1,\dots,\pi_M\}$ be an $\varepsilon/L_R$-cover of $\mathcal P$ under $d(\cdot,\cdot)$.
For any $\pi\in\mathcal P$ choose $j$ with $d(\pi,\pi_j)\le \varepsilon/L_R$. Then
for every trajectory $\tau$,
\begin{equation*}
  |R(\pi;\tau)-R(\pi_j;\tau)| \le L_R d(\pi,\pi_j) \le \varepsilon,
\end{equation*}
so the set $\{\tau\mapsto R(\pi_j;\tau)\}_{j=1}^M$ is an $\varepsilon$-cover of the return-class. Thus, the covering inequality holds.
\end{proof}

\subsection{Generalisation of Reflection-Equivariant Subspace}\label{app:exact}

We now prove a high-probability uniform bound over the equivariant class.

\begin{theorem}\label{thm:app-exact_full_proof}
With $\mathcal R_{\Pi_{\mathrm{eq}}}=\{ \tau\mapsto R(\pi;\tau) : \pi\in\Pi_{\mathrm{eq}}\}$,
fix any accuracy parameter $r\in(0,B)$ and confidence $\delta\in(0,1)$. Then with probability at least $1-\delta$,
\[
  \sup_{\pi\in\Pi_{\mathrm{eq}}} |J(\pi)-\hat J_N(\pi)|
  \le
  C\left(\int_{r}^{B} \sqrt{\frac{\log\mathcal N_{\infty,1}(\mathcal R_{\Pi_{\mathrm{eq}}},\varepsilon)}{N}}  d\varepsilon\right) + \frac{8r}{\sqrt{N}} + B\sqrt{\frac{\log(2/\delta)}{2N}},
\]
where $C$ is an absolute numeric constant, $J(\pi)$ is the
population expected return and $\hat J_N(\pi)=\tfrac{1}{N}\sum_{i=1}^N R(\pi;\tau_i)$
is the empirical return on $N$ i.i.d.\ episodes $\tau_1,\dots,\tau_N$.
\end{theorem}

\begin{proof}
Let $\mathcal F=\mathcal R_{\Pi_{\mathrm{eq}}}$. Following Corollary \ref{lemma:rademacher}, we have:
\begin{equation*}
  \mathbb{E}\Big[ \sup_{f\in\mathcal R_{\Pi_{\mathrm{eq}}}} \Big| \tfrac{1}{N}\sum_{i=1}^N (f(\tau_i)-\mathbb{E}[f]) \Big| \Big]
  \le 2 \mathbb{E}\big[ \mathfrak R_N(\mathcal R_{\Pi_{\mathrm{eq}}} ) \big].
\end{equation*}

Applying Lemma \ref{lemma:dudley}, for any $r>0$:
\begin{equation}\label{eq:upper-bound-exp}
  \mathbb{E}\Big[ \sup_{f\in\mathcal R_{\Pi_{\mathrm{eq}}}} \Big| \tfrac{1}{N}\sum_{i=1}^N (f(\tau_i)-\mathbb{E}[f]) \Big| \Big]
  \le C\left( \int_{r}^{B} \sqrt{\frac{\log\mathcal N_{\infty,1}(\mathcal R_{\Pi_{\mathrm{eq}}},\varepsilon)}{N}} d\varepsilon\right) + \frac{8r}{\sqrt N}.
\end{equation}

Now apply Lemma \ref{lemma:mc} to convert the expectation bound into a high-probability statement, with probability at least $1-\delta$:
\begin{equation}\label{eq:upperbound sup}
  \sup_{f\in\mathcal R_{\Pi_{\mathrm{eq}}}}\Big| \tfrac{1}{N}\sum_{i=1}^N (f(\tau_i)-\mathbb{E}[f]) \Big|
  \le \mathbb{E}\Big[ \sup_{f\in\mathcal R_{\Pi_{\mathrm{eq}}}}\Big| \tfrac{1}{N}\sum_{i=1}^N (f(\tau_i)-\mathbb{E}[f]) \Big| \Big] + B\sqrt{\frac{\log(2/\delta)}{2N}}.
\end{equation}

Combining Equations \ref{eq:upper-bound-exp} and \ref{eq:upperbound sup} yields the claimed inequality.
\end{proof}

\subsection{Generalisatisability of PRISM}\label{app:approx}

\begin{lemma}
\label{prop:app-approx_sup_bound}
If a policy \(\pi\) satisfies \(\mathcal{L}_{eq}\le \varepsilon_{eq}\), then
\begin{equation*}
\sup_{s}\|\Delta_\pi(s)\|_1 \le \sqrt{\frac{\varepsilon_{eq}}{p_{\min}}}.
\end{equation*}
Consequently, the sup--\(\ell_1\) distance between \(\pi\) and its orbit
projection \(Q(\pi)\) satisfies
\begin{equation*}
d(\pi,Q(\pi))= \sup_s\|\pi(s)-Q(\pi)(s)\|_1
\le \sqrt{\frac{\varepsilon_{eq}}{p_{\min}}}.
\end{equation*}
\end{lemma}

\begin{proof}
Assume the state space has density $\frac{d\mu}{ds}(s) \geq p_{\min}$ on the common support. Let $s^*$ be such that $\|\Delta_\pi(s^*)\|_1 = \sup_s \|\Delta_\pi(s)\|_1$. The expectation is:
\begin{equation*}
\varepsilon_{eq} = \mathbb{E}_\mu\big[\|\Delta_\pi(s)\|_1^2\big] = \int \|\Delta_\pi(s)\|_1^2  d\mu(s).
\end{equation*}

For any neighbourhood $B_\delta(s^*)$ of $s^*$:
\begin{equation*}
\varepsilon_{eq} \geq \int_{B_\delta(s^*)} \|\Delta_\pi(s)\|_1^2  d\mu(s).
\end{equation*}

By continuity of $\|\Delta_\pi(\cdot)\|_1$ and the density lower bound:
\begin{equation*}
\int_{B_\delta(s^*)} \|\Delta_\pi(s)\|_1^2  d\mu(s) \geq \left(\|\Delta_\pi(s^*)\|_1 - \epsilon\right)^2 \int_{B_\delta(s^*)} d\mu(s) \geq \left(\|\Delta_\pi(s^*)\|_1 - \epsilon\right)^2 p_{\min} \cdot \text{vol}(B_\delta(s^*)),
\end{equation*}
for sufficiently small $\delta$ and any $\epsilon > 0$. Taking $\delta \to 0$ and $\epsilon \to 0$:
\begin{equation*}
\varepsilon_{eq} \geq p_{\min} \left(\sup_s \|\Delta_\pi(s)\|_1\right)^2.
\end{equation*}

Rearranging gives $\sup_s \|\Delta_\pi(s)\|_1 \leq \sqrt{\frac{\varepsilon_{eq}}{p_{\min}}}$.
\end{proof}

We can now translate this approximation to a bound on returns and to a covering-number statement.

\begin{theorem}
\label{theorem:app-cov2}
Let
\(\xi:=\frac{1}{2}\sqrt{\varepsilon_{eq}/p_{\min}}\). Then for every policy \(\pi\),
\begin{equation*}
|J(\pi)-J(Q(\pi))| \le L_R\cdot d(\pi,Q(\pi)) \le L_R \xi.
\end{equation*}
Define the approximately reflection-equivariant class
\(\Pi_{approx}(\varepsilon_{eq}):=\{\pi\in\Pi: L_{eq}(\pi)\le\varepsilon_{eq}\}\).
Then every \(\pi\in\Pi_{approx}(\varepsilon_{eq})\) lies in the sup-ball
of radius \(\xi\) around \(\Pi_{eq}\). Consequently, for any target
covering radius \(r>\xi\):
\begin{equation*}
N_{\infty,1}\big(\Pi_{approx}(\varepsilon_{eq}),r\big)\le
N_{\infty,1}\big(\Pi_{eq},r-\xi\big).
\end{equation*}
\end{theorem}

\begin{proof}
The first claim is that $|J(\pi)-J(Q(\pi))| \le L_R\cdot d(\pi,Q(\pi)) \le L_R  \xi$.

First, we establish the $L_R$-Lipschitz property of the expected return $J(\pi) = \mathbb{E}_{\tau}[R(\pi;\tau)]$. Using the property from that the return function $R$ is $L_R$-Lipschitz, we have:
\begin{align*}
    |J(\pi) - J(Q(\pi))| &= |\mathbb{E}_{\tau}[R(\pi;\tau) - R(Q(\pi);\tau)]| \nonumber \\
    &\le \mathbb{E}_{\tau}\big[|R(\pi;\tau) - R(Q(\pi);\tau)|\big] \nonumber \\
    &\le \mathbb{E}_{\tau}\big[L_R \cdot d(\pi, Q(\pi))\big] = L_R \cdot d(\pi, Q(\pi)).
\end{align*}

Next, we bound the distance $d(\pi, Q(\pi))$. Using the definition of the projection $Q(\pi)$, we find the distance from $\pi$ to its projection:
\begin{align*}
    d(\pi, Q(\pi)) &= \sup_s \|\pi(s) - Q(\pi)(s)\|_1 \nonumber \\
    &= \sup_s \left\| \pi(s) - \tfrac{1}{2}\left(\pi(s) + K_g(\pi(L_g(s)))\right) \right\|_1 \nonumber \\
    &= \tfrac{1}{2} \sup_s \left\| \pi(s) - K_g(\pi(L_g(s))) \right\|_1.
\end{align*}
The term inside the norm is equal to the equivariance mismatch $\Delta_\pi(s'):= \pi(L_g(s')) - K_g(\pi(s'))$ evaluated at $s' = L_g(s)$, since $L_g$ is an involution.
\[
    \Delta_\pi(L_g(s)) = \pi(L_g(L_g(s))) - K_g(\pi(L_g(s))) = \pi(s) - K_g(\pi(L_g(s))).
\]
Since $L_g$ is a bijection, $\sup_s \|\Delta_\pi(L_g(s))\|_1 = \sup_{s'}\|\Delta_\pi(s')\|_1$. By Lemma~\ref{prop:app-approx_sup_bound}, this supremum is bounded by $\xi$. Therefore:
\begin{equation*}
    d(\pi, Q(\pi)) = \frac{1}{2} \sup_{s'} \left\| \Delta_\pi(s') \right\|_1 \le \xi.
\end{equation*}

The second claim is that for any radius $r>\xi$, we have $N_{\infty,1}\big(\Pi_{approx}(\varepsilon_{eq}),r\big)\le N_{\infty,1}\big(\Pi_{eq},r-\xi\big)$. We know that for any $\pi \in \Pi_{approx}(\varepsilon_{eq})$, its projection $Q(\pi) \in \Pi_{eq}$ satisfies $d(\pi, Q(\pi)) \le \xi$. This implies that the set $\Pi_{approx}(\varepsilon_{eq})$ is contained in a $\xi$-neighbourhood of $\Pi_{eq}$. Let $\{\pi_j\}_{j=1}^M$ be a minimal $(r-\xi)$-cover for $\Pi_{eq}$, where $M = N_{\infty,1}(\Pi_{eq}, r-\xi)$. Now, consider any policy $\pi \in \Pi_{approx}(\varepsilon_{eq})$. There must exist a centre $\pi_j$ from our cover such that $d(Q(\pi), \pi_j) \le r-\xi$.
By the triangle inequality, we can bound the distance from $\pi$ to this centre $\pi_j$:
\begin{align*}
    d(\pi, \pi_j) &\le d(\pi, Q(\pi)) + d(Q(\pi), \pi_j) \nonumber \\
    &\le \xi + (r-\xi) = r.
\end{align*}
This shows that the set $\{\pi_j\}_{j=1}^M$ is an $r$-cover for $\Pi_{approx}(\varepsilon_{eq})$. Since we have found a valid cover of size $M$, the size of the minimal cover must be no larger:
\begin{equation*}
    N_{\infty,1}\big(\Pi_{approx}(\varepsilon_{eq}),r\big) \le N_{\infty,1}\big(\Pi_{eq},r-\xi\big).
\end{equation*}
\end{proof}

\begin{theorem}\label{theorem:app-approx-bound}
With $\mathcal R_{\Pi_{\mathrm{eq}}}=\{ \tau\mapsto R(\pi;\tau) : \pi\in\Pi_{\mathrm{eq}}\}$,
fix any accuracy parameter $r\in(0,B)$ and confidence $\delta\in(0,1)$. Then with probability at least $1-\delta$,
\[
  \sup_{\pi\in\Pi_{approx}(\varepsilon_{eq})} |J(\pi)-\hat J_N(\pi)|
  \le
  C\left(\int_{r}^{B} \sqrt{\frac{\log\mathcal N_{\infty,1}(\mathcal R_{\Pi_{\mathrm{eq}}},\varepsilon)}{N}}  d\varepsilon\right) + \frac{8r}{\sqrt
  N} + B\sqrt{\frac{\log(2/\delta)}{2N}} + 2L_R\xi.
\]
\end{theorem}

\begin{proof}
For any policy $\pi \in \Pi_{approx}(\varepsilon_{eq})$, we can decompose the generalisation error using the triangle inequality by introducing its exact-equivariant projection $Q(\pi) \in \Pi_{eq}$:
\begin{equation*}
|J(\pi) - \hat{J}_N(\pi)| \le |J(\pi) - J(Q(\pi))| + |J(Q(\pi)) - \hat{J}_N(Q(\pi))| + |\hat{J}_N(Q(\pi)) - \hat{J}_N(\pi)|.
\end{equation*}
We bound each of the three terms on the right-hand side.

From Theorem~\ref{theorem:app-cov2}, we have:
\begin{equation*}
|J(\pi) - J(Q(\pi))| \le L_R \cdot d(\pi, Q(\pi)) \le L_R\xi.
\end{equation*}

Since the return function $R(\cdot;\tau)$ is $L_R$-Lipschitz:
\begin{align*}
|\hat{J}_N(Q(\pi)) - \hat{J}_N(\pi)| &= \left| \frac{1}{N} \sum_{i=1}^N \left( R(Q(\pi);\tau_i) - R(\pi; \tau_i) \right) \right| \nonumber \\
&\le \frac{1}{N} \sum_{i=1}^N \left| R(Q(\pi);\tau_i) - R(\pi;\tau_i) \right| \nonumber \\
&\le \frac{1}{N} \sum_{i=1}^N L_R \cdot d(\pi, Q(\pi)) \le L_R\xi.
\end{align*}

The middle term, $|J(Q(\pi)) - \hat{J}_N(Q(\pi))|$, is the generalisation error for an exactly equivariant policy. Combining the bounds, we get:
\begin{equation*}
\sup_{\pi \in \Pi_{approx}(\varepsilon_{eq})} |J(\pi) - \hat{J}_N(\pi)| \le \sup_{\pi' \in \Pi_{eq}} |J(\pi') - \hat{J}_N(\pi')| + L_R\gamma.
\end{equation*}
Applying the high-probability bound from Theorem~\ref{thm:app-exact_full_proof} to the supremum over $\Pi_{eq}$ yields the final result.
\end{proof}

\section{Additional Details of Environments}\label{app:A}

This appendix presents the tables on the environments and how the state space is divided into a symmetric and an asymmetric part. First Table \ref{tab:env} highlights the differences between environments in dimension sizes. Tables \ref{tab:hopper}, \ref{tab:walker2d}, \ref{tab:halfcheetah}, and \ref{tab:swimmer} show the division for mo-hopper-v5, mo-walker2d-v5, mo-halfcheetah-v5, and mo-swimmer-v5, respectively. The action space is always divided into an empty set for the asymmetric part, and the complete set for the symmetric part.

\begin{table}[h!]
\centering
\footnotesize

\begin{threeparttable}
\caption{ Considered MuJoCo environments.}
\label{tab:env}
\begin{tabular}{l c c c}
\toprule
& State Space & Action Space & Reward Space \\
\midrule \midrule
Mo-hopper-v5 & $\mathcal{S} \in \mathbb{R}^{11}$ & $\mathcal{A} \in \mathbb{R}^{3}$ & $\mathcal{R} \in \mathbb{R}^{3}$\\
Mo-walker2d-v5 & $\mathcal{S} \in \mathbb{R}^{17}$ & $\mathcal{A} \in \mathbb{R}^{6}$ & $\mathcal{R} \in \mathbb{R}^{2}$\\
Mo-halfcheetah-v5 & $\mathcal{S} \in \mathbb{R}^{17}$ & $\mathcal{A} \in \mathbb{R}^{6}$ & $\mathcal{R} \in \mathbb{R}^{2}$\\
Mo-swimmer-v5 & $\mathcal{S} \in \mathbb{R}^{8}$ & $\mathcal{A} \in \mathbb{R}^{2}$ & $\mathcal{R} \in \mathbb{R}^{2}$\\
\bottomrule
\end{tabular}
\end{threeparttable}
\end{table}

\begin{table}[h!]
\centering
\footnotesize

\begin{threeparttable}
\caption{ Reflectional symmetry partition for mo-hopper-v5 observation space.}
\label{tab:hopper}
\begin{tabular}{l c c c}
\toprule
Index & Observation Component & Type & Symmetry \\
\midrule \midrule

0 & z-coordinate of the torso & position & Asymmetric \\
1 & angle of the torso & angle & Asymmetric \\

2 & angle of the thigh joint & angle & Symmetric \\
3 & angle of the leg joint & angle & Symmetric \\
4 & angle of the foot joint & angle & Symmetric \\

5 & velocity of the x-coordinate of the torso & velocity & Asymmetric \\
6 & velocity of the z-coordinate of the torso & velocity & Asymmetric \\
7 & angular velocity of the angle of the torso & angular velocity & Asymmetric \\

8 & angular velocity of the thigh hinge & angular velocity & Symmetric \\
9 & angular velocity of the leg hinge & angular velocity & Symmetric \\
10 & angular velocity of the foot hinge & angular velocity & Symmetric \\
\bottomrule
\end{tabular}
\end{threeparttable}
\end{table}

\begin{table}[h!]
\centering
\footnotesize

\begin{threeparttable}
\caption{ Reflectional symmetry partition for mo-walker2d-v5 observation space.}
\label{tab:walker2d}
\begin{tabular}{l c c c}
\toprule
Index & Observation Component & Type & Symmetry \\
\midrule \midrule

0 & z-coordinate of the torso & position & Asymmetric \\
1 & angle of the torso & angle & Asymmetric \\

2 & angle of the thigh joint & angle & Symmetric \\
3 & angle of the leg joint & angle & Symmetric \\
4 & angle of the foot joint & angle & Symmetric \\
5 & angle of the left thigh joint & angle & Symmetric \\
6 & angle of the left leg joint & angle & Symmetric \\
7 & angle of the left foot joint & angle & Symmetric \\

8 & velocity of the x-coordinate of the torso & velocity & Asymmetric \\
9 & velocity of the z-coordinate of the torso & velocity & Asymmetric \\
10 & angular velocity of the angle of the torso & angular velocity & Asymmetric \\

11 & angular velocity of the thigh hinge & angular velocity & Symmetric \\
12 & angular velocity of the leg hinge & angular velocity & Symmetric \\
13 & angular velocity of the foot hinge & angular velocity & Symmetric \\
14 & angular velocity of the left thigh hinge & angular velocity & Symmetric \\
15 & angular velocity of the left leg hinge & angular velocity & Symmetric \\
16 & angular velocity of the left foot hinge & angular velocity & Symmetric \\
\bottomrule
\end{tabular}
\end{threeparttable}
\end{table}

\begin{table}[h!]
\centering
\footnotesize

\begin{threeparttable}
\caption{ Reflectional symmetry partition for mo-halfcheetah-v5 observation space.}
\label{tab:halfcheetah}
\begin{tabular}{l c c c}
\toprule
Index & Observation Component & Type & Symmetry \\
\midrule \midrule

0 & z-coordinate of the front tip & position & Asymmetric \\
1 & angle of the front tip & angle & Asymmetric \\

2 & angle of the back thigh & angle & Symmetric \\
3 & angle of the back shin & angle & Symmetric \\
4 & angle of the back foot & angle & Symmetric \\
5 & angle of the front thigh & angle & Symmetric \\
6 & angle of the front shin & angle & Symmetric \\
7 & angle of the front foot & angle & Symmetric \\

8 & velocity of the x-coordinate of front tip & velocity & Asymmetric \\
9 & velocity of the z-coordinate of front tip & velocity & Asymmetric \\
10 & angular velocity of the front tip & angular velocity & Asymmetric \\

11 & angular velocity of the back thigh & angular velocity & Symmetric \\
12 & angular velocity of the back shin & angular velocity & Symmetric \\
13 & angular velocity of the back foot & angular velocity & Symmetric \\
14 & angular velocity of the front thigh & angular velocity & Symmetric \\
15 & angular velocity of the front shin & angular velocity & Symmetric \\
16 & angular velocity of the front foot & angular velocity & Symmetric \\
\bottomrule
\end{tabular}
\end{threeparttable}
\end{table}

\newpage

\begin{table}[h!]
\centering
\footnotesize

\begin{threeparttable}
\caption{ Reflectional symmetry partition for mo-swimmer-v5 observation space.}
\label{tab:swimmer}
\begin{tabular}{l c c c}
\toprule
Index & Observation Component & Type & Symmetry \\
\midrule \midrule

0 & angle of the front tip & angle & Asymmetric \\

1 & angle of the first rotor & angle & Symmetric \\
2 & angle of the second rotor & angle & Symmetric \\

3 & velocity of the tip along the x-axis & velocity & Asymmetric \\
4 & velocity of the tip along the y-axis & velocity & Symmetric \\

5 & angular velocity of the front tip & angular velocity & Asymmetric \\
6 & angular velocity of first rotor & angular velocity & Symmetric \\
7 & angular velocity of second rotor & angular velocity & Symmetric \\
\bottomrule
\end{tabular}
\end{threeparttable}
\end{table}

\section{Additional Details of Experimental Settings}\label{app:B}

\textbf{Evaluation Measures.} For the approximated Pareto front, we consider three well-known metrics that investigate the extent of the approximated front. 

First, we consider hypervolume (HV) \citep{fonseca2006improved}, which measures the volume of the objective space dominated by the approximated Pareto front relative to a reference point. A downside of many evaluation measures is that they require domain knowledge about the true underlying Pareto front, whereas HV only considers a reference point without any a priori knowledge, making it ideal to assess the volume of the front. The reference point is typically set to the nadir point or slightly worse, and following \citet{felten_toolkit_2023}, we set it to $-100$ for all objectives and environments. The HV is defined as follows:
\begin{equation*}
HV(CS, \bm{r}) = \lambda\left(\bigcup_{\bm{cs} \in CS} {\bm{x} \in \mathbb{R}^L : \bm{cs} \preceq \bm{x} \preceq \bm{r}}\right),
\end{equation*}
where $CS = {\bm{cs}_1, \bm{cs}_2, \ldots, \bm{cs}_n}$ is the coverage set, or the Pareto front approximation, $\bm{r} \in \mathbb{R}^L$ is the reference point, $\bm{cs} \preceq \bm{x}$ means ${cs}_i \leq x_i$ for all objectives $i = 1, \ldots, L$, and $\lambda(\cdot)$ denotes the Lebesgue measure. Yet, hypervolume values are difficult to interpret, as they do not have a direct link to any notion of value or utility \citep{hayes2022practical}.

As such, we also consider the Expected Utility Metric (EUM) \citep{zintgraf2015quality}, which computes the expected maximum utility across different preference weight vectors, and is defined as follows:
\begin{equation*}
EUM(CS, \mathcal{W}) = \frac{1}{|\mathcal{W}|} \sum_{\bm{\omega} \in \mathcal{W}} \max_{\bm{cs} \in CS} U(\bm{\omega}, \bm{cs}),
\end{equation*}
where $\mathcal{W} = \{\bm{\omega}_1, \bm{\omega}_2, \ldots, \bm{\omega}_k\}$ is a set of weight vectors, $|\mathcal{W}|$ is the cardinality of the weight set, $U(\bm{\omega}, \bm{cs})$ is the utility function, which is set to $U(\bm{\omega}, \bm{s}) = \bm{\omega} \cdot \bm{cs} = \sum_{i=1}^L \omega_i \cdot {cs}_i$.

To specifically assess performance with respect to distributional preferences, we also consider one metric designed to evaluate the optimality of the entire return distribution associated with the learned policies \citep{NEURIPS2023_32285dd1}. 

To be precise, we consider the Variance Objective (VO), which evaluates how well the policy set can balance the trade-off between maximising expected returns and minimising their variance. A set of $M$ random preference vectors is generated, where each vector specifies a different weighting between the expected return and its standard deviation for each objective. The satisfaction score $u(p_i, \pi_j)$ for a policy $\pi_j$ under preference $p_i$ is a weighted sum of the expected return $\mathbb{E}[Z(\pi_j)]$ and the negative standard deviation $-\sqrt{\text{Var}[Z(\pi_j)]}$. The final metric is the mean score over these preferences, rewarding policies that achieve high expected returns with low variance:
\begin{equation*}
    \text{VO}(\Pi, \{p_i\}_{i=1}^M) = \frac{1}{M} \sum_{i=1}^{M} \max_{\pi_j \in \Pi} u(p_i, \pi_j).
\end{equation*}

\textbf{Hyperparameters.}
Due to time, computational limitations, and the excessive number of hyperparameters, we do not perform an extensive hyperparameter tuning process. Below are the used hyperparameters. All hyperparameters that are not mentioned below are set to their default value. 

The probability of releasing sparse rewards $p_{\text{rel}}$ is always set to a one-hot vector, where sparsity is imposed on the reward dimension related to moving forward. Since the main goal is to move forward, imposing sparsity on this channel should make it a more difficult task for the reward shaping model. Furthermore, we deal with extreme heterogeneous sparsity, where most channels exhibit regular rewards, but one channel only releases a reward at the end of an episode, making it more difficult for the model to link certain states and actions to the observed cumulative reward. 

The hyperparameters in Table \ref{tab:hyperparams2} for ReSymNet are identical for each environment. The advantage of using the same hyperparameters for each environment is that if one configuration performs well everywhere, it could indicate that the proposed method is inherently stable, especially given the noted diversity between the considered environments. However, this does come at a cost of potentially suboptimal performance per environment. 

\begin{table}[h!]
\centering
\footnotesize

\begin{threeparttable}
\caption{ Hyperparameters for ReSymNet.}
\label{tab:hyperparams2}
\begin{tabular}{l c c}
\toprule
& PRISM \\
\midrule \midrule
Initial collection $N$ &  1000\\
Expert collection $E$ &  1000\\
Number of refinements $IR$  & 2\\
Timesteps per cycle $M$ &  100,000 \\
Epochs &  1000\\
Learning rate & 0.005 \\
Learning rate scheduler &Exponential \\
Learning rate decay  & 0.99 \\
Ensemble size $|\mathcal{E}|$& 3\\
Hidden dimension & 256\\
Dropout  &  0.3\\
Initialisation & Kaiman \citep{he2015delving}\\
Validation split & 0.2 \\
Patience & 20\\
Batch size & 32\\
\bottomrule
\end{tabular}
\end{threeparttable}
\end{table}
The hyperparameter controlling the symmetry loss differs per environment, since some environments require strict equivariance, whereas others require a more flexible approach. Table \ref{tab:symreg hyper} shows the used values.

\begin{table}[h!]
\centering
\footnotesize

\begin{threeparttable}
\caption{SymReg hyperparameter.}
\label{tab:symreg hyper}
\begin{tabular}{l c c c c}
\toprule
& Mo-hopper-v5 & Mo-walker2d-v5 & Mo-halfcheetah-v5 & Mo-swimmer-v5 \\
\midrule \midrule
$\lambda$ & 0.01  & 1 & 0.01 & 0.005 \\
\bottomrule
\end{tabular}
\end{threeparttable}
\end{table}

\section{Pareto Fronts}\label{app:pareto-fronts}

Figure \ref{fig:four_figures2} shows the approximated Pareto fronts. The results demonstrate that shaped rewards yield superior performance, covering a wider and more optimal range of the objective space compared to dense and sparse rewards.

\begin{figure}[h!]
    \centering
    \begin{subfigure}[b]{0.245\textwidth}
        \centering
        \includegraphics[width=\textwidth]{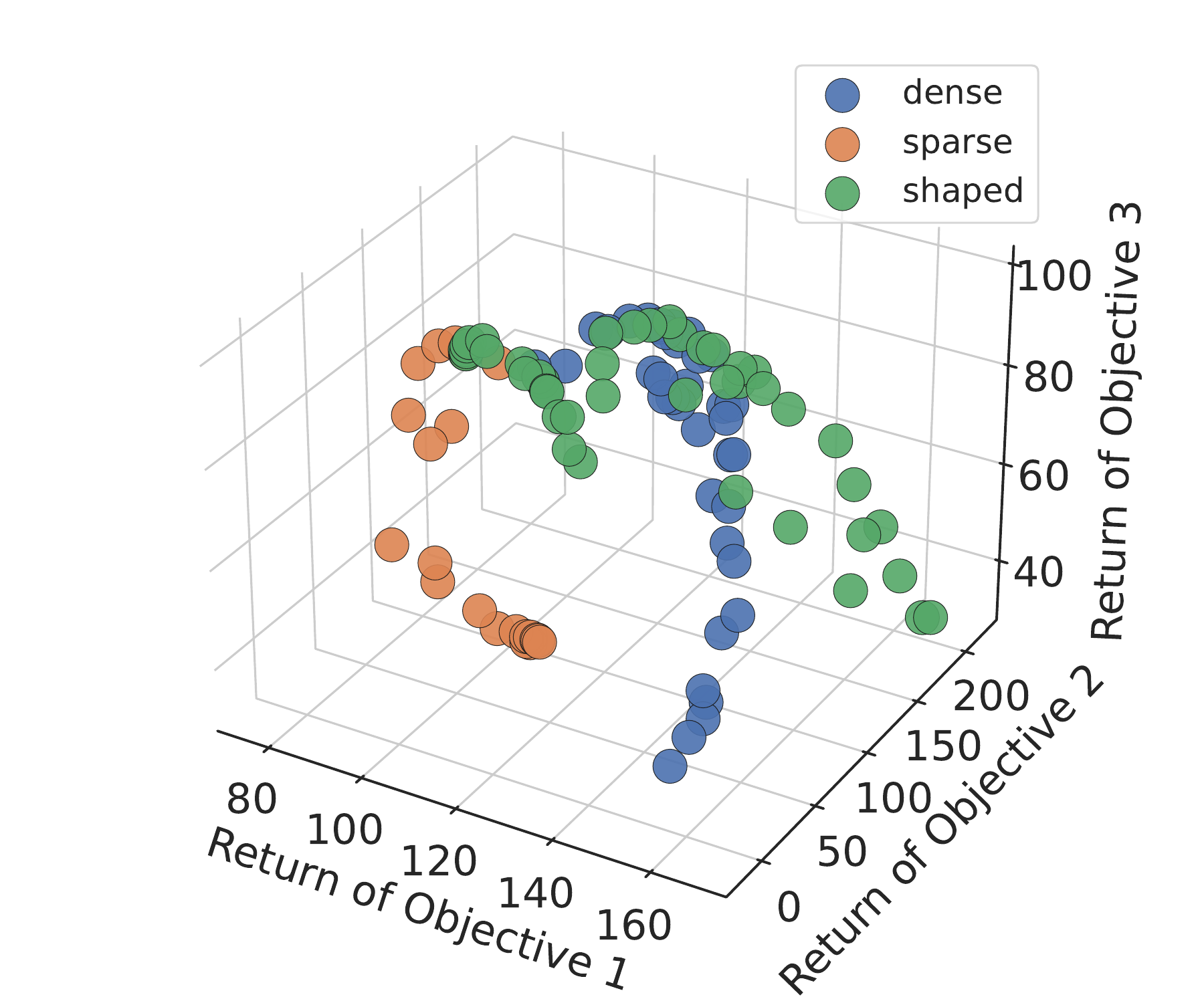}
        \caption{Mo-hopper-v5}
        \label{fig:a2}
    \end{subfigure}
    \hfill 
    \begin{subfigure}[b]{0.245\textwidth}
        \centering
        \includegraphics[width=\textwidth]{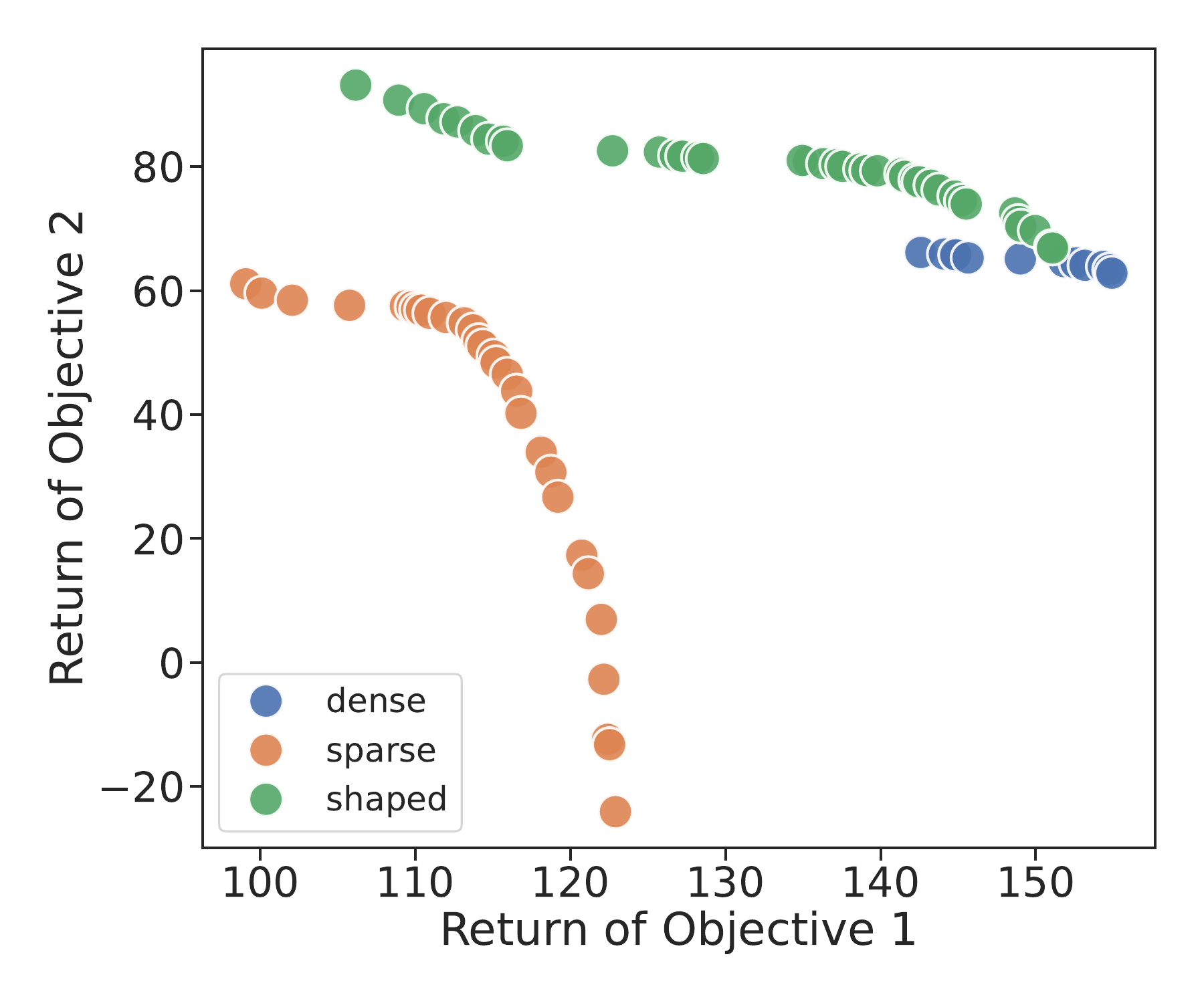}
        \caption{Mo-walker2d-v5}
        \label{fig:b2}
    \end{subfigure}
    \hfill 
    \begin{subfigure}[b]{0.245\textwidth}
        \centering
        \includegraphics[width=\textwidth]{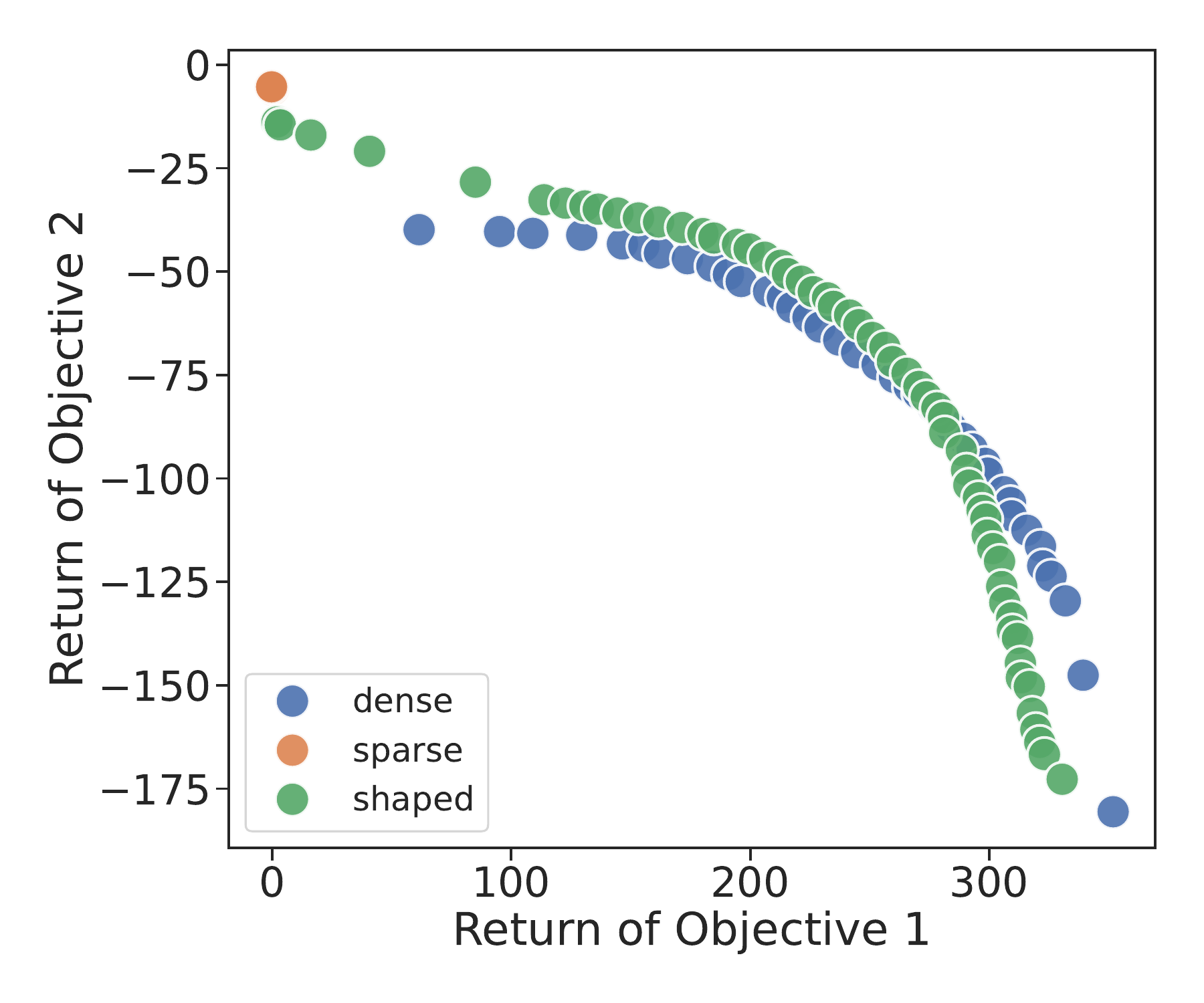}
        \caption{Mo-halfcheetah-v5}
        \label{fig:c2}
    \end{subfigure}
    \hfill 
    \begin{subfigure}[b]{0.245\textwidth}
        \centering
        \includegraphics[width=\textwidth]{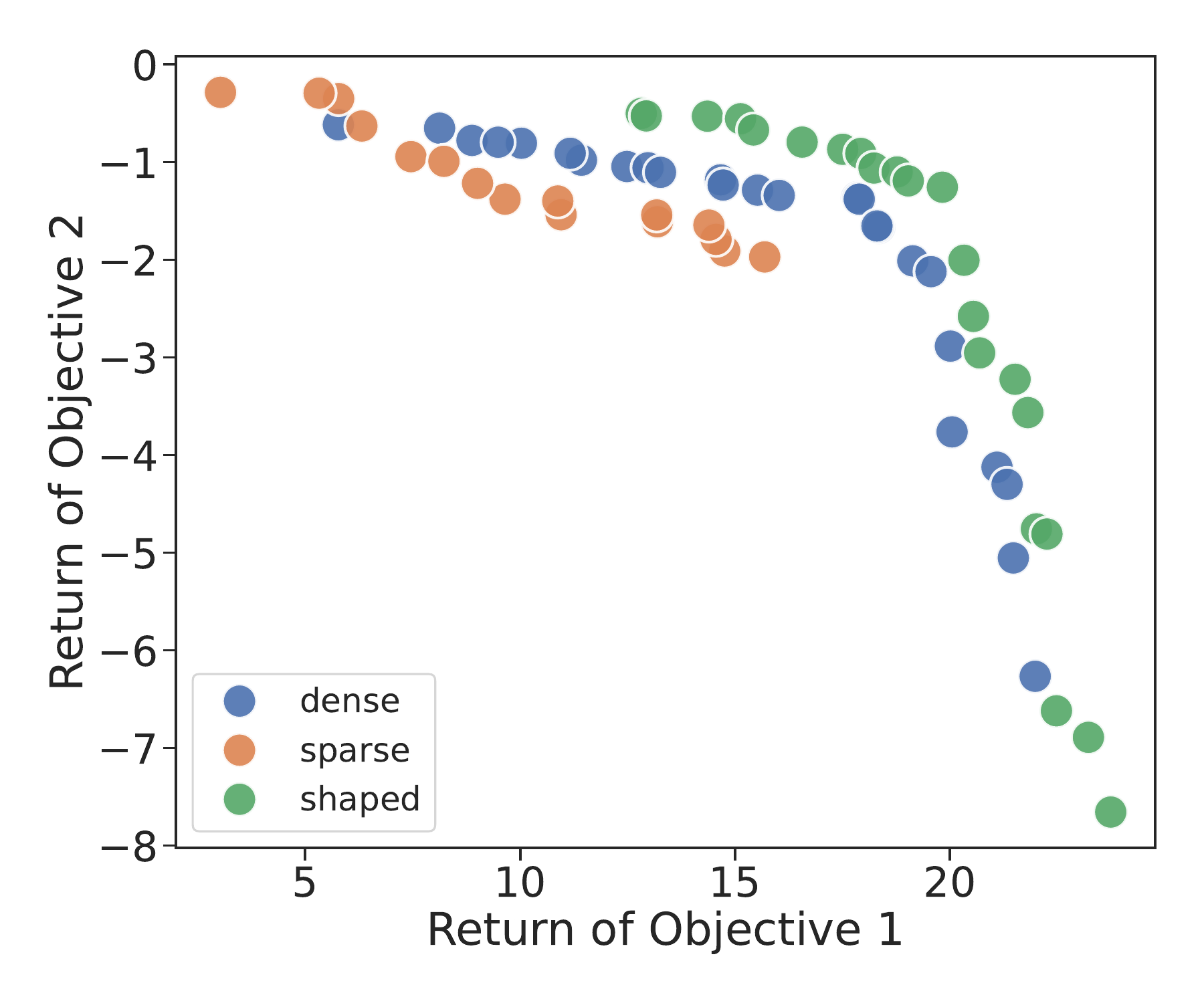}
        \caption{Mo-swimmer-v5}
        \label{fig:d2}
    \end{subfigure}
    \caption{The approximated Pareto front for dense rewards (blue dots), sparse rewards (orange dots), and shaped rewards (green dots). Sparsity is imposed on the first reward objective.}
    \label{fig:four_figures2}
\end{figure}

\section{Ablation Study}\label{app:ablation}
Tables \ref{tab:ablation} and \ref{tab:ablation2} report the obtained values for the ablation study. Results are again averaged over ten trials, similar to the main experiments. 
\begin{table}[h!]
    \centering
    \footnotesize
    \begin{adjustbox}{width=\textwidth, totalheight=\textheight,keepaspectratio}
    \begin{threeparttable}
    
    \caption{PRISM ablation study results. We report the average hypervolume (HV), Expected Utility Metric (EUM), and Variance Objective (VO) over 10 trials, with the standard error shown in grey. w/o is the abbreviation of without. The largest values are in bold font. }
    \label{tab:ablation}

    \begin{tabular}{l l l l l l l l}
    \toprule
  Environment  & Metric& PRISM & w/o residual & w/o dense rewards & w/o ensemble & w/o refinement & w/o loss \\
    \midrule \midrule
      \multirow{3}{*}{Mo-hopper-v5} 
      & HV ($\times 10^7$) & \textbf{1.58}\,$\pm$\,\textcolor{gray}{0.05} & 1.29\,$\pm$\,\textcolor{gray}{0.09} & 1.38\,$\pm$\,\textcolor{gray}{0.11} & 1.38\,$\pm$\,\textcolor{gray}{0.08} & 1.55\,$\pm$\,\textcolor{gray}{0.04} & 1.42\,$\pm$\,\textcolor{gray}{0.07} \\
      & EUM & \textbf{147.43}\,$\pm$\,\textcolor{gray}{2.61} & 128.40\,$\pm$\,\textcolor{gray}{6.06} & 134.67\,$\pm$\,\textcolor{gray}{6.89} & 135.28\,$\pm$\,\textcolor{gray}{4.91} & 145.89\,$\pm$\,\textcolor{gray}{2.73} & 137.85\,$\pm$\,\textcolor{gray}{4.22} \\
      & VO & \textbf{66.66}\,$\pm$\,\textcolor{gray}{1.40} & 58.61\,$\pm$\,\textcolor{gray}{2.71} & 61.21\,$\pm$\,\textcolor{gray}{3.03} & 61.51\,$\pm$\,\textcolor{gray}{2.19} & 66.54\,$\pm$\,\textcolor{gray}{1.34} & 62.71\,$\pm$\,\textcolor{gray}{1.83} \\
      \midrule
      \multirow{3}{*}{Mo-walker2d-v5} 
      & HV ($\times 10^4$) & \textbf{4.77}\,$\pm$\,\textcolor{gray}{0.07} & 4.65\,$\pm$\,\textcolor{gray}{0.11} & 4.66\,$\pm$\,\textcolor{gray}{0.06} & 4.60\,$\pm$\,\textcolor{gray}{0.08} & 4.60\,$\pm$\,\textcolor{gray}{0.09} & 4.58\,$\pm$\,\textcolor{gray}{0.13} \\
      & EUM & \textbf{120.43}\,$\pm$\,\textcolor{gray}{1.64} & 114.33\,$\pm$\,\textcolor{gray}{2.48} & 116.83\,$\pm$\,\textcolor{gray}{1.65} & 113.79\,$\pm$\,\textcolor{gray}{2.02} & 114.98\,$\pm$\,\textcolor{gray}{2.84} & 112.77\,$\pm$\,\textcolor{gray}{3.01} \\
      & VO & \textbf{59.35}\,$\pm$\,\textcolor{gray}{0.80} & 56.46\,$\pm$\,\textcolor{gray}{1.21} & 57.67\,$\pm$\,\textcolor{gray}{0.73} & 56.19\,$\pm$\,\textcolor{gray}{0.97} & 57.03\,$\pm$\,\textcolor{gray}{1.42} & 55.59\,$\pm$\,\textcolor{gray}{1.44} \\
      \midrule
      \multirow{3}{*}{Mo-halfcheetah-v5} 
      & HV ($\times 10^4$) & \textbf{2.25}\,$\pm$\,\textcolor{gray}{0.18} & 1.95\,$\pm$\,\textcolor{gray}{0.20} & 2.08\,$\pm$\,\textcolor{gray}{0.21} & 1.91\,$\pm$\,\textcolor{gray}{0.19} & 2.23\,$\pm$\,\textcolor{gray}{0.18} & 1.90\,$\pm$\,\textcolor{gray}{0.19} \\
      & EUM & 89.94\,$\pm$\,\textcolor{gray}{15.33} & 73.06\,$\pm$\,\textcolor{gray}{16.57} & 82.24\,$\pm$\,\textcolor{gray}{16.97} & 81.60\,$\pm$\,\textcolor{gray}{17.65} & \textbf{92.68}\,$\pm$\,\textcolor{gray}{14.79} & 71.12\,$\pm$\,\textcolor{gray}{16.91} \\
      & VO & 40.72\,$\pm$\,\textcolor{gray}{7.02} & 32.99\,$\pm$\,\textcolor{gray}{7.65} & 37.31\,$\pm$\,\textcolor{gray}{7.99} & 36.76\,$\pm$\,\textcolor{gray}{8.06} & \textbf{42.28}\,$\pm$\,\textcolor{gray}{6.85} & 32.12\,$\pm$\,\textcolor{gray}{7.75} \\
      \midrule
      \multirow{3}{*}{Mo-swimmer-v5} 
      & HV ($\times 10^4$) & \textbf{1.21}\,$\pm$\,\textcolor{gray}{0.00} & \textbf{1.21}\,$\pm$\,\textcolor{gray}{0.00} & 1.20\,$\pm$\,\textcolor{gray}{0.00} & 1.20\,$\pm$\,\textcolor{gray}{0.00} & \textbf{1.21}\,$\pm$\,\textcolor{gray}{0.00} & 1.20\,$\pm$\,\textcolor{gray}{0.00} \\
      & EUM & 9.44\,$\pm$\,\textcolor{gray}{0.14} & 9.39\,$\pm$\,\textcolor{gray}{0.15} & 9.07\,$\pm$\,\textcolor{gray}{0.11} & 9.25\,$\pm$\,\textcolor{gray}{0.13} & \textbf{9.46}\,$\pm$\,\textcolor{gray}{0.13} & 9.35\,$\pm$\,\textcolor{gray}{0.14} \\
      & VO & \textbf{4.24}\,$\pm$\,\textcolor{gray}{0.07} & 4.20\,$\pm$\,\textcolor{gray}{0.08} & 4.09\,$\pm$\,\textcolor{gray}{0.05} & 4.15\,$\pm$\,\textcolor{gray}{0.08} & \textbf{4.24}\,$\pm$\,\textcolor{gray}{0.07} & \textbf{4.24}\,$\pm$\,\textcolor{gray}{0.07} \\
    \bottomrule
    \end{tabular}
    \end{threeparttable}
    \end{adjustbox}
\end{table}

\begin{table}[h!]
    \centering
    \footnotesize
    \begin{threeparttable}
    
    \caption{ReSymNet ablation study results. We report the average hypervolume (HV), Expected Utility Metric (EUM), and Variance Objective (VO) over 10 trials, with the standard error shown in grey. w/o is the abbreviation of without.}
    \label{tab:ablation2}

    \begin{tabular}{l l l l}
    \toprule
  Environment  & Metric& uniform & random \\
    \midrule \midrule
      \multirow{3}{*}{Mo-hopper-v5} 
      & HV ($\times 10^7$) & 1.38\,$\pm$\,\textcolor{gray}{0.08} &  0.49\,$\pm$\,\textcolor{gray}{0.06} 
 \\
      & EUM & 135.19\,$\pm$\,\textcolor{gray}{5.30} & 65.22\,$\pm$\,\textcolor{gray}{6.63} \\
      & VO & 63.90\,$\pm$\,\textcolor{gray}{2.34} & 
29.62\,$\pm$\,\textcolor{gray}{3.68} \\
      \midrule
      \multirow{3}{*}{Mo-walker2d-v5} 
      & HV ($\times 10^4$) & 4.67\,$\pm$\,\textcolor{gray}{0.07} &  1.18\,$\pm$\,\textcolor{gray}{0.10} 
 \\
      & EUM & 116.72\,$\pm$\,\textcolor{gray}{2.11} & 16.52\,$\pm$\,\textcolor{gray}{4.98} \\
      & VO & 56.22\,$\pm$\,\textcolor{gray}{1.01}  & 
3.77\,$\pm$\,\textcolor{gray}{2.46} \\
      \midrule
      \multirow{3}{*}{Mo-halfcheetah-v5} 
      & HV ($\times 10^4$) & 0.98\,$\pm$\,\textcolor{gray}{0.00} & 0.78\,$\pm$\,\textcolor{gray}{0.05} 
 \\
      & EUM & -1.34\,$\pm$\,\textcolor{gray}{0.39} & -10.52\,$\pm$\,\textcolor{gray}{2.67}    \\
      & VO & -0.85\,$\pm$\,\textcolor{gray}{0.20} & 
-6.51\,$\pm$\,\textcolor{gray}{1.48} \\
      \midrule
      \multirow{3}{*}{Mo-swimmer-v5} 
      & HV ($\times 10^4$) & 1.09\,$\pm$\,\textcolor{gray}{0.01} & 1.10\,$\pm$\,\textcolor{gray}{0.02}  
\\
      & EUM & 4.37\,$\pm$\,\textcolor{gray}{0.69} & 3.75\,$\pm$\,\textcolor{gray}{0.87} \\
      & VO & 1.56\,$\pm$\,\textcolor{gray}{0.33} & 
1.06\,$\pm$\,\textcolor{gray}{0.40}\\
    \bottomrule
    \end{tabular}
    \end{threeparttable}
\end{table}

\newpage

\section{Generalisability}\label{app:extra_spar}
\subsection{Sparsity on Other Objectives}
We further investigate the robustness of PRISM by inverting the sparsity setting: we maintain the forward velocity reward as dense but make the control cost objective sparse. Table \ref{tab:all_energy} shows that, without hyperparameter tuning, PRISM handles this problem much better than the baselines.

\begin{table}[h!]
    \centering
    \footnotesize
    \begin{threeparttable}
    
    \caption{Experimental results on the control cost objective. We report the average hypervolume (HV), Expected Utility Metric (EUM), and Variance Objective (VO) over 10 trials, with the standard error shown in grey. The largest (best) values are in bold font.}
    \label{tab:all_energy}

    \begin{tabular}{l l l l l}
    \toprule
   Environment & Metric & Oracle & Baseline & PRISM \\
    \midrule \midrule
\multirow{3}{*}{Mo-hopper-v5} 
    & HV ($\times 10^7$) 
        & 1.30\,$\pm$\,\textcolor{gray}{0.13} 
        & 1.19\,$\pm$\,\textcolor{gray}{0.10} 
        & \textbf{1.51}\,$\pm$\,\textcolor{gray}{0.11} \\
    & EUM 
        & 129.04\,$\pm$\,\textcolor{gray}{7.96} 
        & 124.82\,$\pm$\,\textcolor{gray}{7.21} 
        & \textbf{142.89}\,$\pm$\,\textcolor{gray}{7.38} \\
    & VO 
        & 59.07\,$\pm$\,\textcolor{gray}{3.45} 
        & 56.21\,$\pm$\,\textcolor{gray}{3.20} 
        & \textbf{67.58}\,$\pm$\,\textcolor{gray}{3.31} \\

      \midrule
\multirow{3}{*}{Mo-walker2d-v5} 
    & HV ($\times 10^4$) 
        & 4.21\,$\pm$\,\textcolor{gray}{0.11} 
        & 3.16\,$\pm$\,\textcolor{gray}{0.13} 
        & \textbf{4.59}\,$\pm$\,\textcolor{gray}{0.14} \\
    & EUM 
        & 107.58\,$\pm$\,\textcolor{gray}{2.86} 
        & 85.95\,$\pm$\,\textcolor{gray}{3.27} 
        & \textbf{114.62}\,$\pm$\,\textcolor{gray}{2.80} \\
    & VO 
        &53.22\,$\pm$\,\textcolor{gray}{1.39} 
        & 41.29\,$\pm$\,\textcolor{gray}{1.49} 
        & \textbf{54.84}\,$\pm$\,\textcolor{gray}{1.25} \\
      \midrule
\multirow{3}{*}{Mo-halfcheetah-v5} 
    & HV ($\times 10^4$) 
        & 1.70\,$\pm$\,\textcolor{gray}{0.20} 
        & 0.00\,$\pm$\,\textcolor{gray}{0.00} 
        & \textbf{1.72}\,$\pm$\,\textcolor{gray}{0.19} \\
    & EUM 
        & \textbf{81.29}\,$\pm$\,\textcolor{gray}{21.85} 
        & -101.49\,$\pm$\,\textcolor{gray}{3.23} 
        & 76.50\,$\pm$\,\textcolor{gray}{20.85} \\
    & VO 
        & \textbf{36.84}\,$\pm$\,\textcolor{gray}{10.06} 
        & -56.26\,$\pm$\,\textcolor{gray}{1.63} 
        & 31.27\,$\pm$\,\textcolor{gray}{8.68} \\

      \midrule
     \multirow{3}{*}{Mo-swimmer-v5} 
    & HV ($\times 10^4$) 
        & \textbf{1.21}\,$\pm$\,\textcolor{gray}{0.00} 
        & 1.05\,$\pm$\,\textcolor{gray}{0.02} 
        & \textbf{1.21}\,$\pm$\,\textcolor{gray}{0.01} \\
    & EUM 
        & \textbf{9.41}\,$\pm$\,\textcolor{gray}{0.12} 
        & 1.50\,$\pm$\,\textcolor{gray}{1.00} 
        & 9.32\,$\pm$\,\textcolor{gray}{0.19} \\
    & VO 
        & \textbf{4.22}\,$\pm$\,\textcolor{gray}{0.08} 
        & -0.61\,$\pm$\,\textcolor{gray}{0.68} 
        & 3.95\,$\pm$\,\textcolor{gray}{0.08} \\

    \bottomrule
    \end{tabular}
    \begin{tablenotes}[flushleft]
\item[\hspace{-\labelsep}] 
\end{tablenotes}
    \end{threeparttable}
\end{table}

For mo-hopper-v5, PRISM improves HV by 16\% over the oracle ($1.51 \times 10^7$ compared to $1.30 \times 10^7$) and 27\% over the baseline. Similar gains are observed for mo-walker2d-v5, where PRISM achieves a 9\% HV improvement over the oracle and 45\% over the baseline. Notably, in mo-halfcheetah-v5, the baseline suffers a collapse (HV of $0.00$), whereas PRISM recovers the performance to exceed the oracle ($1.72 \times 10^4$ against $1.70 \times 10^4$). These improvements imply that PRISM effectively reconstructs the dense penalty signal, preventing the agent from exploiting the delay to maximise velocity at the cost of extreme energy inefficiency.

Improvements in EUM follow the same trend, with mo-walker2d-v5 showing an increase of roughly 33\% compared to the baseline ($114.62$ vs $85.95$). On distributional metrics, PRISM delivers more consistent performance than the baseline. In mo-swimmer-v5, the baseline’s VO drops to $-0.61$, indicating high instability, whereas PRISM achieves $3.95$, comparable to the oracle ($4.22$). These gains are crucial because they indicate that PRISM produces Pareto fronts that are not only high-performing but also balanced and robust, effectively mitigating the high-variance behaviour from the baseline.

\subsection{Sensitivity to Sparsity}

Figure \ref{fig:spar2} demonstrates that PRISM maintains robust performance across varying levels of reward sparsity. While performance is generally consistent, we observe minor fluctuations at intermediate values (e.g., $p_{rel}=0.2$ in mo-hopper-v5 and mo-walker2d-v5). Two key factors explain this behaviour: (1) PRISM was hyperparameter-tuned specifically for the extreme sparsity setting ($p_{rel}=0$), which is the most challenging MORL scenario. We utilised a fixed set of hyperparameters across all experiments to demonstrate method stability rather than optimising for each sparsity level, and (2) increasing $p_{rel}$ increases the number of available reward signals (data points) per episode. Since ReSymNet was calibrated for the data-scarce sparse setting, the increase of supervision targets at higher $p_{rel}$ levels changes optimisation dynamics, leading to temporary instability. Despite these factors, PRISM consistently recovers high performance, proving its capability to handle heterogeneous reward structures without requiring specific tuning for denser environments.
\begin{figure}[h!]
    \centering
    \begin{subfigure}[b]{0.24\textwidth}
        \centering
        \includegraphics[width=\textwidth]{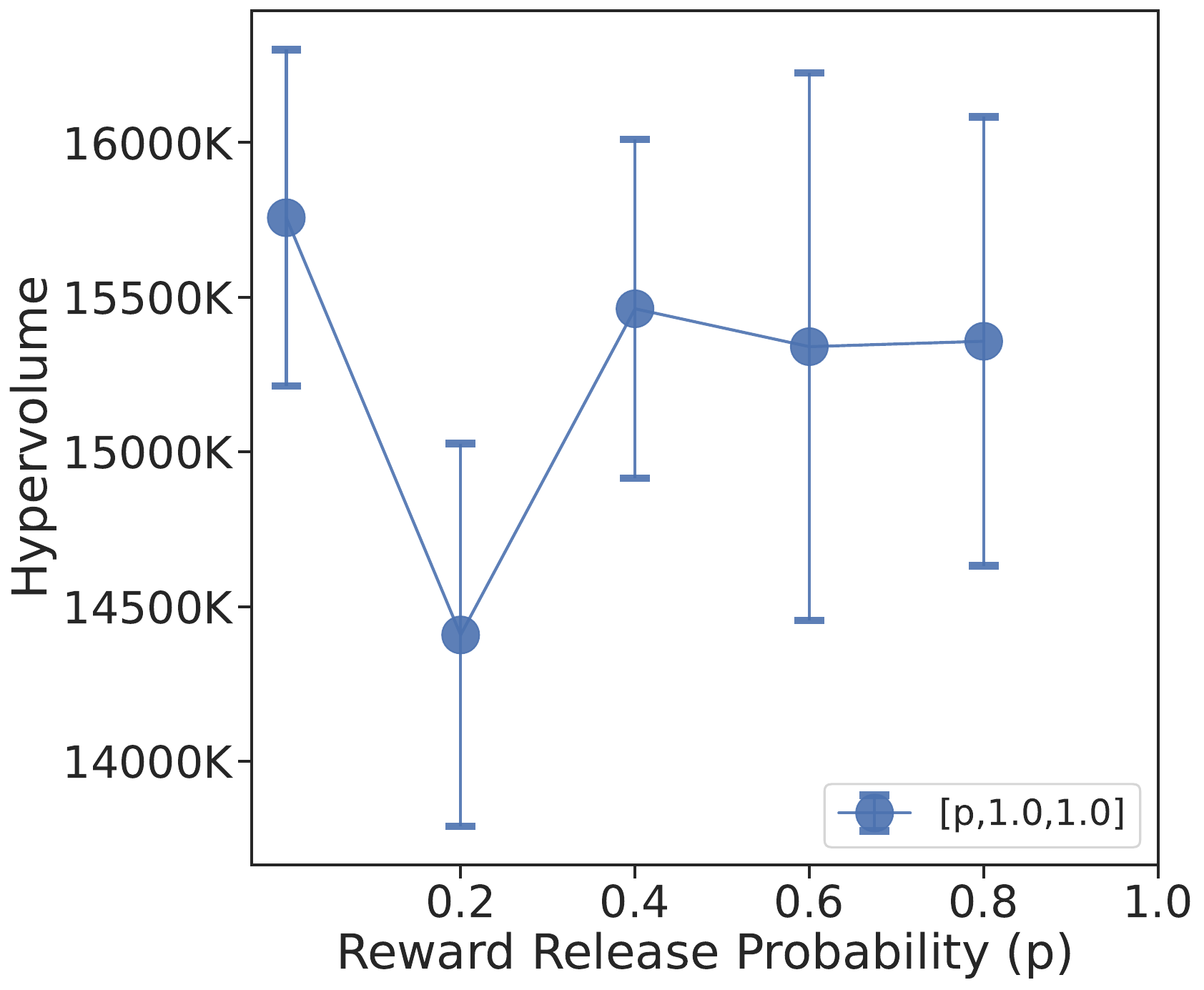}
        \caption{Mo-hopper-v5}
        \label{fig:spar-hop2}
    \end{subfigure}
    \hfill 
    \begin{subfigure}[b]{0.24\textwidth}
        \centering
        \includegraphics[width=\textwidth]{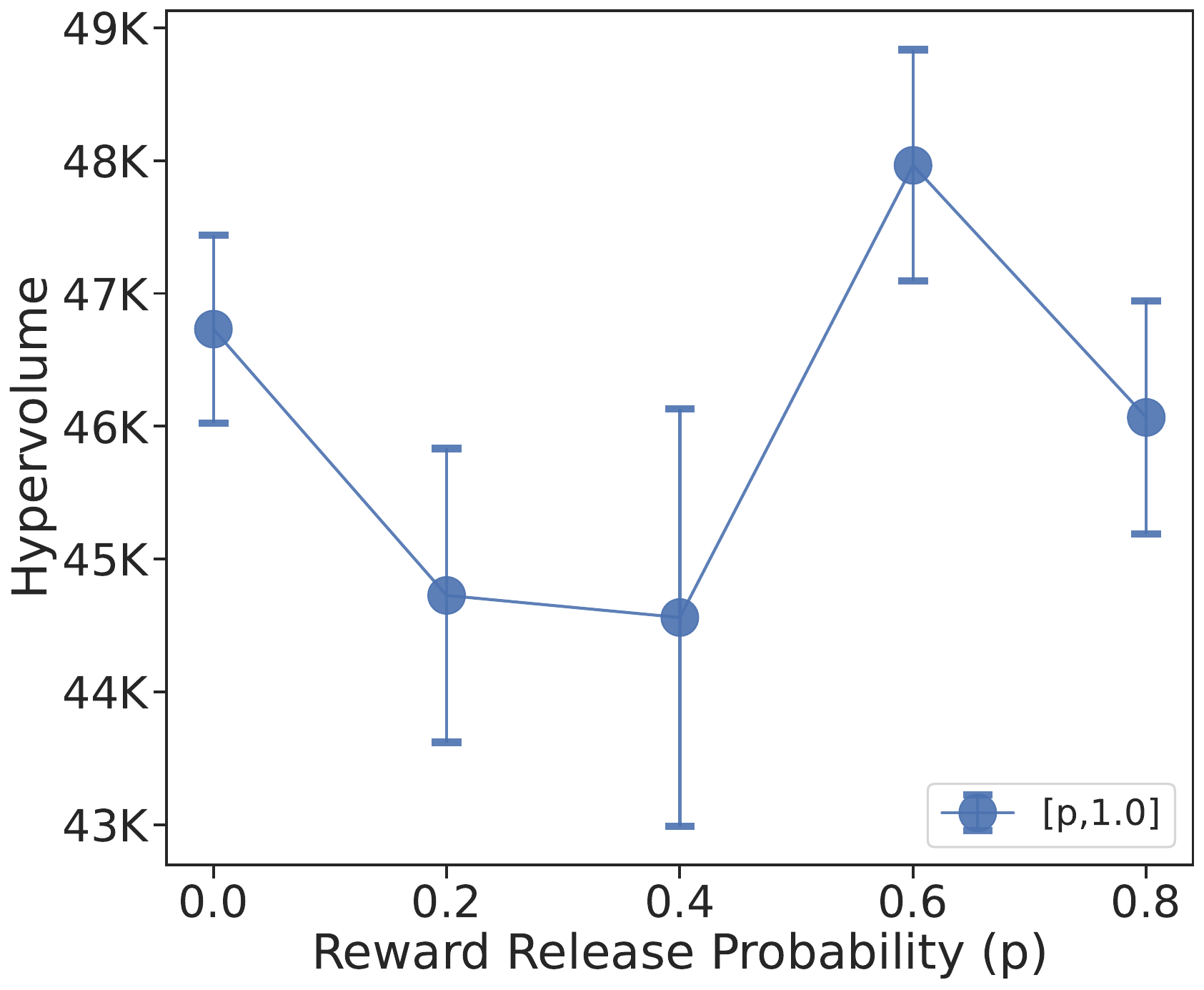}
        \caption{Mo-walker2d-v5}
        \label{fig:spar-walker2}
    \end{subfigure}
    \hfill 
    \begin{subfigure}[b]{0.24\textwidth}
        \centering
        \includegraphics[width=\textwidth]{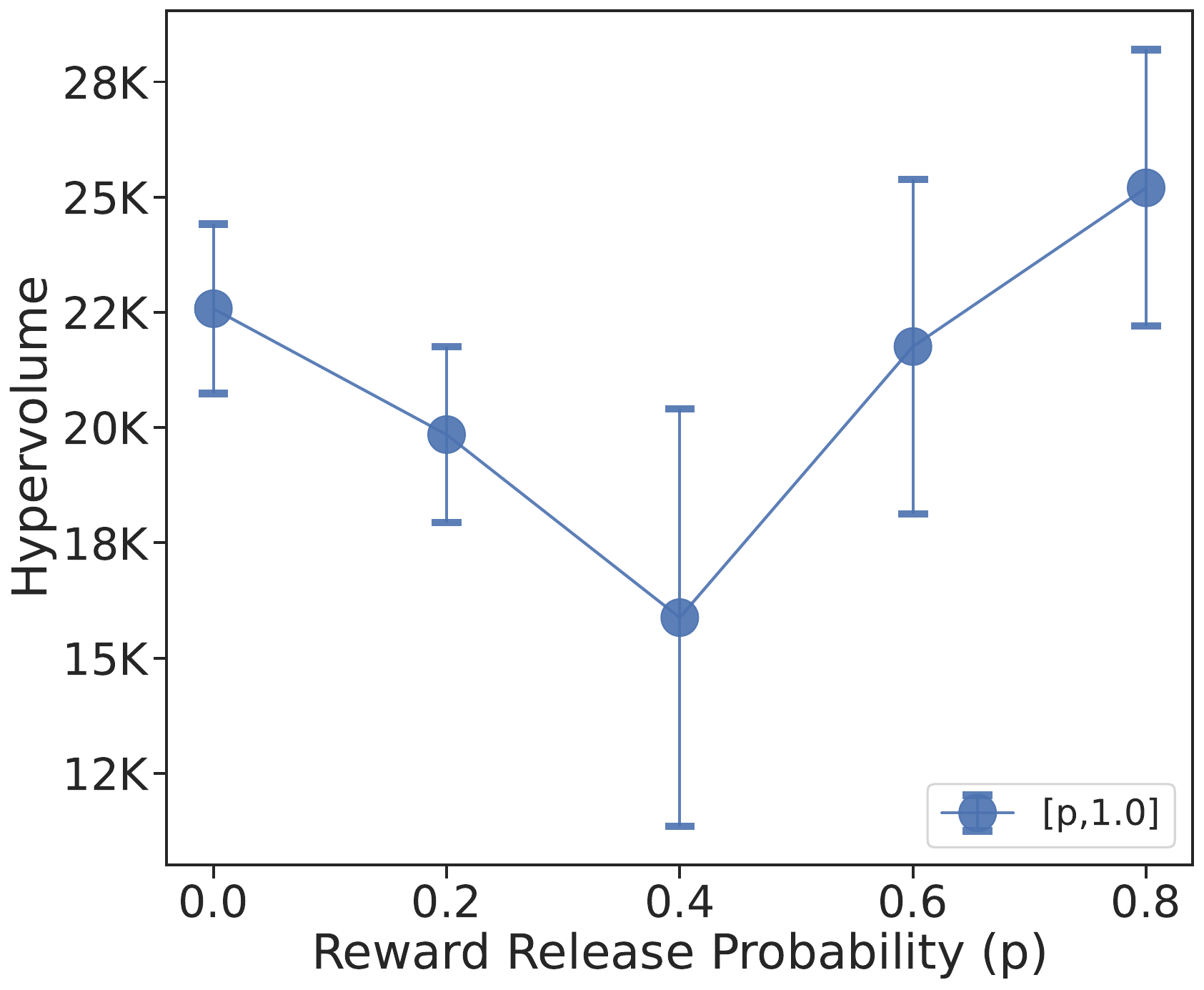}
        \caption{Mo-halfcheetah-v5}
        \label{fig:spar-chee2}
    \end{subfigure}
    \hfill 
    \begin{subfigure}[b]{0.24\textwidth}
        \centering
        \includegraphics[width=\textwidth]{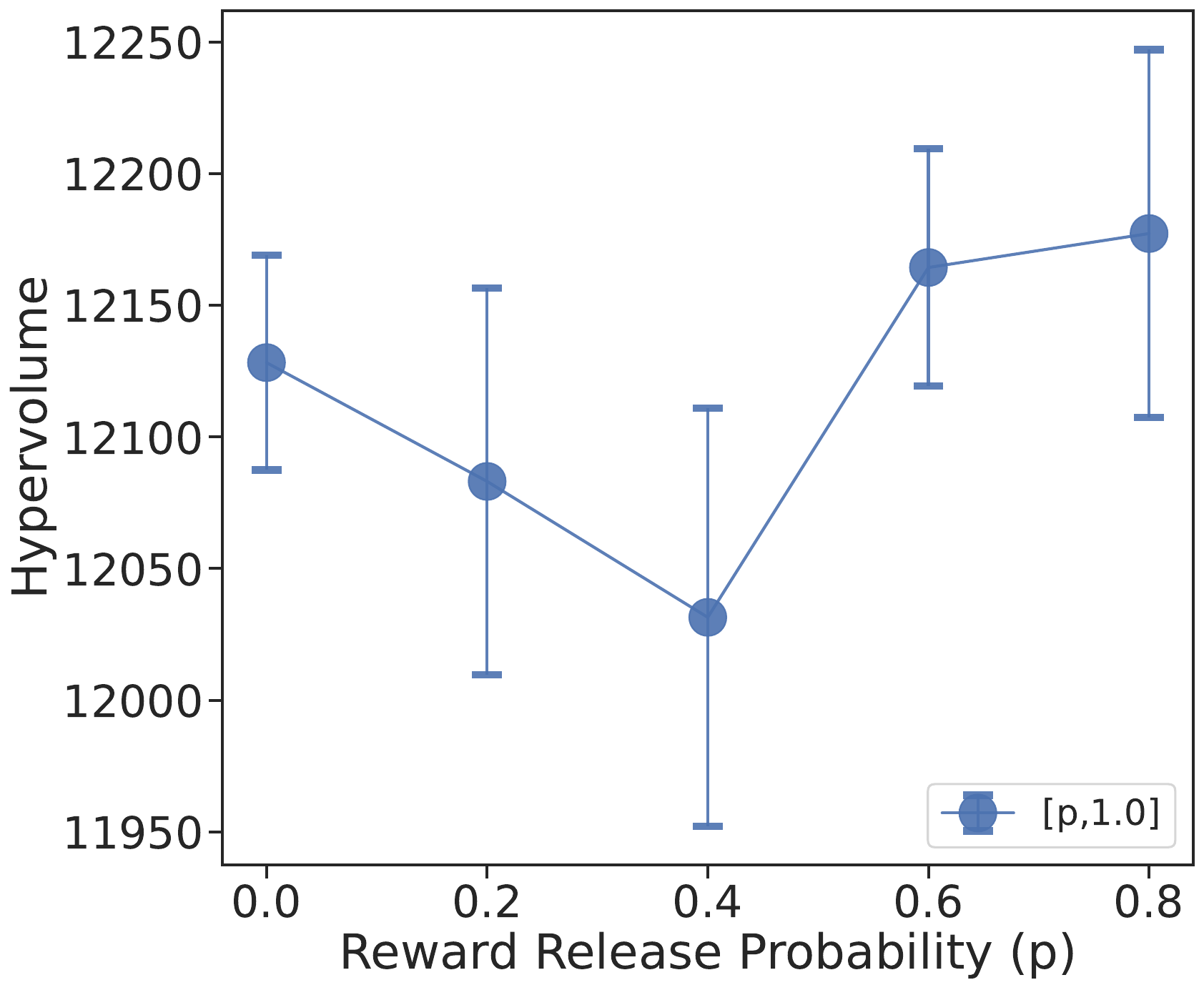}
        \caption{Mo-swimmer-v5}
        \label{fig:spar-swim2}
    \end{subfigure}
    \caption{The obtained hypervolume for various levels of sparsity for PRISM.}
    \label{fig:spar2}
\end{figure}

\subsection{Sensitivity to MORL Algorithms}
To demonstrate that PRISM is a model-agnostic framework not limited to specific architectures, we evaluated its performance using GPI-PD (Generalised Policy Improvement with Linear Dynamics) \citep{DBLP:conf/atal/AlegreBRN023} as an alternative backbone to CAPQL. Table \ref{tab:gpd} confirms that PRISM remains highly effective, consistently outperforming the sparse baseline and obtaining near-oracle performance.

\begin{table}[h!]
    \centering
    \footnotesize
    \begin{threeparttable}
    
    \caption{Experimental results of GPI-PD. We report the average hypervolume (HV), Expected Utility Metric (EUM), and Variance Objective (VO) over 10 trials, with the standard error shown in grey. The largest (best) values are in bold font.}
    \label{tab:gpd}

    \begin{tabular}{l l l l l}
    \toprule
   Environment & Metric & Oracle & Baseline & PRISM \\
    \midrule \midrule
\multirow{3}{*}{Mo-hopper-v5} 
    & HV ($\times 10^7$) 
        & \textbf{1.65}\,$\pm$\,\textcolor{gray}{0.10} 
        & 0.67\,$\pm$\,\textcolor{gray}{0.04} 
        & \textbf{1.65}\,$\pm$\,\textcolor{gray}{0.07} \\
    & EUM 
        & \textbf{151.45}\,$\pm$\,\textcolor{gray}{5.87} 
        & 85.87\,$\pm$\,\textcolor{gray}{3.17} 
        & 148.19\,$\pm$\,\textcolor{gray}{4.26} \\
    & VO 
        & \textbf{72.26}\,$\pm$\,\textcolor{gray}{2.90} 
        & 41.21\,$\pm$\,\textcolor{gray}{1.44} 
        & 70.24\,$\pm$\,\textcolor{gray}{2.51} \\

      \midrule
\multirow{3}{*}{Mo-walker2d-v5} 
    & HV ($\times 10^4$) 
        & \textbf{5.93}\,$\pm$\,\textcolor{gray}{0.10} 
        & 3.20\,$\pm$\,\textcolor{gray}{0.23} 
        & 5.61\,$\pm$\,\textcolor{gray}{0.10} \\
    & EUM 
        & \textbf{141.88}\,$\pm$\,\textcolor{gray}{2.38} 
        & 76.41\,$\pm$\,\textcolor{gray}{6.47} 
        & 132.67\,$\pm$\,\textcolor{gray}{2.26} \\
    & VO 
        & \textbf{67.63}\,$\pm$\,\textcolor{gray}{1.17} 
        & 35.64\,$\pm$\,\textcolor{gray}{3.91} 
        & 63.19\,$\pm$\,\textcolor{gray}{1.75} \\

      \midrule
\multirow{3}{*}{Mo-halfcheetah-v5} 
    & HV ($\times 10^4$) 
        & 1.80\,$\pm$\,\textcolor{gray}{0.22} 
        & 1.00\,$\pm$\,\textcolor{gray}{0.02} 
        & \textbf{2.24}\,$\pm$\,\textcolor{gray}{0.16} \\
    & EUM 
        & \textbf{164.75}\,$\pm$\,\textcolor{gray}{14.21} 
        & -1.31\,$\pm$\,\textcolor{gray}{0.54} 
        & 99.89\,$\pm$\,\textcolor{gray}{8.06} \\
    & VO 
        & \textbf{73.90}\,$\pm$\,\textcolor{gray}{7.05} 
        & -1.14\,$\pm$\,\textcolor{gray}{0.31} 
        & 40.74\,$\pm$\,\textcolor{gray}{5.17} \\

      \midrule
   \multirow{3}{*}{Mo-swimmer-v5} 
    & HV ($\times 10^4$) 
        & \textbf{1.23}\,$\pm$\,\textcolor{gray}{0.01} 
        & 1.12\,$\pm$\,\textcolor{gray}{0.01} 
        & 1.22\,$\pm$\,\textcolor{gray}{0.00} \\
    & EUM 
        & \textbf{9.68}\,$\pm$\,\textcolor{gray}{0.17} 
        & 5.17\,$\pm$\,\textcolor{gray}{0.58} 
        & 9.56\,$\pm$\,\textcolor{gray}{0.13} \\
    & VO 
        & 4.23\,$\pm$\,\textcolor{gray}{0.14} 
        & 2.18\,$\pm$\,\textcolor{gray}{0.39} 
        & \textbf{4.37}\,$\pm$\,\textcolor{gray}{0.18} \\

    \bottomrule
    \end{tabular}
    \begin{tablenotes}[flushleft]
\item[\hspace{-\labelsep}] 
\end{tablenotes}
    \end{threeparttable}
\end{table}

In mo-hopper-v5, PRISM achieves an HV of $1.65 \times 10^7$, matching the oracle exactly and far exceeding the baseline ($0.67 \times 10^7$). This trend of near-perfect recovery is consistent across mo-walker2d-v5 and mo-swimmer-v5. This indicates that the shaped rewards generated by ReSymNet are robust enough to guide different policy optimisation mechanisms effectively. In mo-halfcheetah-v5, PRISM achieves a significantly higher HV ($2.24$) compared to the oracle ($1.80$).

Notably, these results were obtained with minimal hyperparameter tuning due to computational constraints. While this lack of fine-tuning explains the slight gap in EUM/VO metrics for mo-halfcheetah-v5 compared to the oracle, the method's ability to achieve such strong results with a completely different backbone highlights PRISM's inherent stability and generalisability.

\section{Declaration on Large Language Models}

Large Language Models (LLMs) were used for (1) polishing the wording of the manuscript for clarity and readability, (2) brainstorming about algorithm names and their abbreviations, and (3) searching for algorithms for consideration in the preliminary stage.

\end{document}